\documentclass{article}
\usepackage[table,dvipsnames]{xcolor}
\usepackage{fullpage}
\usepackage{amsmath,amsthm}
\usepackage[mathscr]{eucal}
\usepackage{amssymb}
\usepackage{mathtools}
\usepackage[vlined,ruled]{algorithm2e}
\usepackage{graphicx}
\usepackage{subcaption}
\usepackage{float}
\usepackage{color} %
\usepackage{mdframed}
\usepackage{enumitem}
\usepackage{framed}
\usepackage{imakeidx}
\usepackage[
  style=alphabetic,
  citestyle=alphabetic,
  natbib=true,
  backend=bibtex,
  maxcitenames=3,  mincitenames=1,
  maxbibnames=6,  minbibnames=1,
  firstinits=true,
  backref=true, 
  hyperref=true,
  doi=false,  isbn=false,  url=false,
  arxiv=abs
]{biblatex}

\usepackage{newpxmath}
\usepackage{newpxtext}

\makeindex[title=Alphabetical Index,intoc]

\DeclareMathOperator*{\argmax}{argmax}
\DeclareMathOperator*{\argmin}{argmin}
\usepackage{color}

\usepackage{microtype}
\DisableLigatures{encoding = *, family = * }
\SetKwInput{KwInput}{Input}
\SetKwInput{KwOutput}{Output}
\SetKw{KwTo}{to}
\SetKw{KwDownto}{downto}

\newcommand{\figdir}{./figures}

\usepackage[bookmarks, colorlinks=true, citecolor=Violet,linkcolor=Mahogany,urlcolor=blue]{hyperref}

\newtheorem{thm}{Theorem}[section]

\newtheorem{defn}{Definition}[section]

\newcommand{\transpose}[1]{{#1}^{\trans}}
\newcommand{\real}{\mathbb{R}}
\newcommand{\trans}{\top}
\newcommand{\energy}{\ensuremath{\mathcal E}}

\newcommand{\keywordIndex}[1]{#1\index{#1}}
\newcommand{\keyword}[1]{\keywordIndex{#1}}
\newcommand{\keywordDef}[1]{{\bf #1}\index{#1|textbf}}

\newcommand{\entropyrv}{\ensuremath{R}}
\newcommand{\entropyrvv}{\ensuremath{r}}

\newcommand{\submodlongonly}[1]{}
\newcommand{\submodshortonly}[1]{#1}
\newcommand{\submodlongshortalt}[2]{#2}

\newcommand{\submodprob}[1]{\ensuremath{\mathsf{P}\left[#1\right]}}
\newcommand{\submodfun}{\ensuremath{f}}
\newcommand{\submodconcavefun}{\ensuremath{\phi}}
\newcommand{\submodhessianmatrix}{\ensuremath{M}}
\newcommand{\submodrealvectorx}{\ensuremath{x}}
\newcommand{\submodrealvectory}{\ensuremath{y}}
\newcommand{\submodrelaxrealvectordim}{\ensuremath{n}}
\newcommand{\submodrealvectordim}{\ensuremath{d}}
\newcommand{\submodjoin}{\ensuremath{\vee}}
\newcommand{\submodmeet}{\ensuremath{\wedge}}

\newcommand{\submodfgain}[3]{{\ensuremath #1(#2|#3)}}

\newcommand{\submodconvexfun}{\ensuremath{\psi}}
\newcommand{\submodaltfun}{\ensuremath{g}}
\newcommand{\submodaltaltfun}{\ensuremath{h}}
\newcommand{\submodgroundset}{\ensuremath{V}}
\newcommand{\submodgroundsetsize}{\ensuremath{n}}
\newcommand{\submodel}{\ensuremath{v}}
\newcommand{\submodell}{\ensuremath{w}}
\newcommand{\submodelll}{\ensuremath{u}}
\newcommand{\submodellll}{\ensuremath{x}}
\newcommand{\submodelI}{\ensuremath{i}}
\newcommand{\submodelJ}{\ensuremath{j}}
\newcommand{\submodcnstr}{\ensuremath{\mathcal C}}
\newcommand{\submodcardcnstr}{\ensuremath{k}}
\newcommand{\submodknapnstr}{\ensuremath{b}}
\newcommand{\submodmodfun}{\ensuremath{m}}
\newcommand{\submodmodaltfun}{\ensuremath{x}}
\newcommand{\submodmodaltaltfun}{\ensuremath{y}}
\newcommand{\submodmodfunconst}{\ensuremath{c}}

\newcommand{\submodfuncweight}{\ensuremath{\omega}}
\newcommand{\submodmodrealvalue}{\ensuremath{z}}

\newcommand{\submodaltgroundset}{\ensuremath{E}}

\newcommand{\submodcurvature}{\ensuremath{\kappa}}

\newcommand{\submodsetX}{\ensuremath{X}}
\newcommand{\submodsetY}{\ensuremath{Y}}
\newcommand{\submodsetZ}{\ensuremath{Z}}
\newcommand{\submodaltsetX}{\ensuremath{A}}
\newcommand{\submodaltelx}{\ensuremath{a}}
\newcommand{\submodaltsetY}{\ensuremath{B}}

\newcommand{\submodaltsetZ}{\ensuremath{C}}
\newcommand{\submodaltelz}{\ensuremath{c}}
\newcommand{\submodaltsetZZ}{\ensuremath{D}}
\newcommand{\submodsetchain}{\ensuremath{S}}
\newcommand{\submodlexweight}{\ensuremath{\lambda}}
\newcommand{\submodfeatureset}{\ensuremath{U}}
\newcommand{\submodfeatureel}{\ensuremath{u}}
\newcommand{\submodscscscskthreshold}{\ensuremath{\alpha}}

\newcommand{\Set}{{\ensuremath S}\xspace} %
\newcommand{\SetI}{{\ensuremath R}\xspace} %
\newcommand{\SetS}{{\ensuremath S}\xspace} %
\newcommand{\SetT}{{\ensuremath T}\xspace} %
\newcommand{\SetC}{{\ensuremath C}\xspace} %

\newcommand{\submodvv}{{\ensuremath v}\xspace} %
\newcommand{\submodvvx}{{\ensuremath x}\xspace} %
\newcommand{\submodvvy}{{\ensuremath y}\xspace} %

\newcommand{\setplussingle}[2]{\ensuremath{{#1} \cup \{#2\}}}

\newcommand{\setminussingle}[2]{\ensuremath{{#1} \setminus \{#2\}}}

\newcommand{\eqnlabel}[1]{\label{eqn:#1}}

\newcommand{\submodELa}{\ensuremath{a}}
\newcommand{\submodELb}{\ensuremath{b}}
\newcommand{\submodELc}{\ensuremath{c}}
\newcommand{\submodELd}{\ensuremath{d}}
\newcommand{\submodELe}{\ensuremath{e}}
\newcommand{\submodELf}{\ensuremath{f}}

\newcommand{\submodbernoullirv}{\ensuremath{B}}

\newcommand{\submodgraph}{\ensuremath{\mathcal G}}
\newcommand{\submodedgesgroundset}{\submodaltgroundset}
\newcommand{\submodedgessubset}{\submodaltsetX}
\newcommand{\submodedgeel}{\ensuremath{e}}
\newcommand{\submodedgeela}{\ensuremath{a}}
\newcommand{\submodedgeelb}{\ensuremath{b}}
\newcommand{\submodedgeweight}{\ensuremath{w}}
\newcommand{\submodgraphcuthyperparameter}{\ensuremath{\lambda}}

\newcommand{\lovasz}{Lov\'asz}

\newcommand{\submodmatroidindsetI}{\ensuremath{I}}
\newcommand{\submodmatroidsetofsets}{\ensuremath{\mathcal I}}
\newcommand{\submodpartmatroidnumblocks}{\ensuremath{m}}
\newcommand{\submodpartmatroidlimit}{\ensuremath{\ell}}

\newcommand{\submodDPPmatrix}{\ensuremath{\mathbf M}}

\newcommand{\submodcharv}{\ensuremath{\mathbf 1}}

\newcommand{\submodbasepoly}{\ensuremath{B}}
\newcommand{\submodccl}[1]{{\ensuremath \check #1}}
\newcommand{\submodlex}[1]{{\ensuremath \breve #1}}
\newcommand{\submodorder}{{\ensuremath \sigma}}
\newcommand{\submodvecx}{{\ensuremath x}}
\newcommand{\submodvecy}{{\ensuremath y}}

\newcommand{\smcoffee}{{\ensuremath c}}
\newcommand{\smlemon}{{\ensuremath l}}
\newcommand{\smmilk}{{\ensuremath m}}
\newcommand{\smtea}{{\ensuremath t}}

\newcommand{\submodneighborenergyconsant}{{\ensuremath \beta}}
\newcommand{\submodexternalenergyconsant}{{\ensuremath \gamma}}

\renewcommand{\submodlongonly}[1]{#1}
\renewcommand{\submodshortonly}[1]{}
\renewcommand{\submodlongshortalt}[2]{#1}

\usepackage{authblk}
\title{Submodularity\\ In Machine Learning and Artificial Intelligence}
\author[1,2]{Jeffrey A. Bilmes}
\affil[1]{Department of Electrical and Computer Engineering, University of Washington, Seattle, 98195}
\affil[2]{Department of Computer Science and Engineering, University of Washington, Seattle, 98195}

\date{\today}

\begin{document}

\maketitle

\begin{abstract}
In this manuscript, we offer a gentle review of submodularity and supermodularity and their properties. We offer a plethora of submodular definitions; a full description of a number of example submodular functions and their generalizations; example discrete constraints; a discussion of basic algorithms for maximization, minimization, and other operations; a brief overview of continuous submodular extensions; and some historical applications.  We then turn to how submodularity is useful in machine learning and artificial intelligence. This includes summarization, and we offer a complete account of the differences between and commonalities amongst sketching, coresets, extractive and abstractive summarization in NLP, data distillation and condensation, and data subset selection and feature selection. We discuss a variety of ways to produce a submodular function useful for machine learning, including heuristic hand-crafting, learning or approximately learning a submodular function or aspects thereof, and some advantages of the use of a submodular function as a coreset producer. We discuss submodular combinatorial information functions, and how submodularity is useful for clustering, data partitioning, parallel machine learning, active and semi-supervised learning, probabilistic modeling, and structured norms and loss functions.\footnote{This manuscript is a greatly extended version of a section on submodularity written for ``Probabilistic Machine Learning: Advanced Topics'', by Kevin Murphy, MIT Press, 2023.}
\end{abstract}

\tableofcontents

\submodshortonly{

\section{Submodular optimization}
\label{sec:submodular}
\label{sec:submodularGreedy}

\coauthor{This section was written by Jeff Bilmes.}

}

\submodlongonly{
 \section{Introduction}

Many, if not most, problems in machine learning involve some form of
optimization. The problem of ``learning'' itself can be seen simply as
a problem of optimizing a specified objective over a set of parameters
where the objective is also parameterized by data.  The typical
example is $\min_{\theta} J_\theta({\mathcal D})$ where $\theta$ is a
vector of continuous valued parameters and ${\mathcal D}$ is a sampled
(training) dataset. \submodlongonly{As is widely known, convex optimization
is an important mathematical strategy used in machine learning theory
and by many machine learning applications. Many seemingly disparate
optimization problems such as nonlinear classification
\cite{Schoelkopf02}, clustering \cite{lashkari07convex}, and
dimensionality reduction \cite{weinberger05nonlinear} can be cast as
convex programs. When minimizing a convex loss function, we can be
assured that efficiently finding an optimal solution, even for large
problems, is feasible. Convex optimization is a structural property of
the objective in the space of inherently continuous optimization
problems.}

There are many problems in machine learning, however, that are
inherently discrete, where optimization must occur not over a
continuous but rather a discrete parameter space. In this case,
the objective provides a value over a countable or finite (although
often exponentially large) number of possibilities.  For example,
given a set of features that potentially could be used as the input to
a classifier, how should one choose a useful subset? If there are $m$
features, there are $2^m$ possible subsets, far too many to consider
them all in an exhaustive search. \submodlongonly{This is the standard feature
selection problem in machine learning for which there are many
possible solutions.} Rather than select features, we might wish to
select a subset of the training data. The reasons for doing so abound:
the training data might be too large, or it might be redundant, and
hence it will be unnecessarily costly to use all of it. One option is
to select a random subset, but this can miss corner cases. Perhaps
there is a better way to select a subset that is good in some way,
i.e., that is guaranteed to faithfully represent all of the
data. Another example is human labeling and annotating a
training dataset, a process that is often a costly, time-consuming,
tedious, and error-prone endeavor. Rather than labeling all of the
data, one can perhaps choose a good subset on which having the labels
would be just as good as having the labels on the whole. \submodlongonly{Batch active
learning is an example of this.}

As yet another example: how should one choose the structure, e.g., the
width and depth, of a deep neural network?  Again, this is an
important problem and can be seen as a generalization of feature
selection since we are considering not only all possible subsets of
the input feature set, but all sets of possible hidden units in each
layer of a DNN, and all possible number of layers, a countably
infinite space to consider, in theory.  Despite the difficulty
presented by the combinatorial exposition of possibilities,
researchers have endeavored to provide useful solutions, examples
being all of the work on recent AutoML and NAS methods.  The lottery ticket
hypothesis, stating that any neural network can be approximated
sufficiently well by selecting the right subnetwork (i.e., a subset of
hidden units and weights) out of a much larger random network, is
enticing, but how could one master the finite but still combinatorial
explosion of possible subnetworks?

To offer still another example of discrete optimization in ML,
consider the case when one wishes to compute MAP inference in a
probabilistic model: $\argmax_{y} p(y|x)$ where $y$ is a multivariate
integer-valued variable. Performing such inference naively is clearly,
in the size of $y$, an exponential cost prospect.  How can one do this
efficiently?  This problem is essentially the same as the semantic
image segmentation problem described in other chapters of the book,
where $x$ is a set of pixels and $y$ is a set of integer labels one
for each pixel. One way to tame the complexity of this problem is to
assume that certain factorization properties are true of $p(y|x)$
which allows the use of the distributed law, and this is one of the
purposes of graphical models (i.e., to describe what factorization
properties are active). This allows exact or (more often) approximate
solutions to be feasibly and practically computed. However, what if
the factorization properties assumed of $p(y|x)$ that enable tractable
exact or approximate computation are inaccurate or undesirable? Does
this leave us with no recourse to proceed?

Lastly, while it is known that continuous representations of discrete
phenomena (such as words, sentences, or even knowledge representation
and logic clauses in the domain of NLP) are useful and there has
been much recent research on this topic, at some point we must
eventually make a final crisp decision. For example, in a speech
recognition system, we need a final single word string hypothesis
corresponding to what was spoken. In an autonomous vehicle, we need at
some point to decide whether to stop, go, turn left, right, or to
continue forward. In financial investing, we need to decide to buy,
sell, or do nothing.  These are all discrete decisions that any
machine learning system must eventually deal with.

Regardless of the approach taken to solve these problems, we see how
they all require a form of discrete optimization for their solution.
Many other machine learning problems are also inherently and
fundamentally discrete.  It is thus useful to study machine learning
structures that themselves are inherently discrete. It is also
important for these structures to be easily applicable to many
different problems, and also that are amenable to efficient
combinatorial optimization strategies, even when the solution space
may be exponential in the problem size. Submodular functions fit this
bill well.
}

\submodshortonly{This section provides a brief overview
  of submodularity in machine learning.\footnote{A greatly extended version of the
    material in this section may be found
    at~\cite{bilmes-submod-and-ml-2022}.}}
Submodularity has an extremely simple definition. However, the
``simplest things are often the most complicated to understand
fully''~\cite{samuelson1974complementarity}, and while submodularity
has been studied extensively over the years, it continues
to yield new and surprising insights and properties, some of which are
extremely relevant to data science, machine learning, and artificial
intelligence.
A submodular function operates on subsets of some finite
\textit{ground set}, $\submodgroundset$. Finding a guaranteed good
subset of $\submodgroundset$ would ordinarily require an amount of
computation exponential in the size of $\submodgroundset$.  Submodular
functions, however, have certain properties that make optimization
either tractable or approximable where otherwise neither would be
possible.  The properties are quite natural, however, so submodular
functions are both flexible and widely applicable to real
problems. Submodularity involves an intuitive and natural diminishing
returns property, stating that adding an element to a smaller set
helps more than adding it to a larger set. Like convexity,
submodularity allows one to efficiently find provably optimal or
near-optimal solutions. In contrast to convexity, however, where
little regarding maximization is guaranteed, submodular functions can
be both minimized and (approximately) maximized.  Submodular
maximization and minimization, however, require very different
algorithmic solutions and have quite different applications.  It is
sometimes said that submodular functions are a discrete form of
convexity. This is not quite true, as submodular functions are like
both convex and concave functions, but also have properties that are
similar simultaneously to both convex and concave functions at the
same time, but then some properties of submodularity are neither like
convexity nor like concavity. Convexity and concavity, for example, can be
conveyed even as univariate functions.  This is impossible for
submodularity, as submodular functions are defined based only on the
response of the function to changes amongst different variables in a
multidimensional discrete space\submodlongonly{ (this is discussed a bit
more in the context of DR-submodularity below)}.

\submodlongonly{
There is a growing interest in submodularity in machine learning, as
witnessed by the increasing number of such articles seen in recent
years at machine learning conferences, and some of these will be
surveyed in Section~\ref{sec:appl-subm-mach}. Before we do, we start
the next section by describing submodularity in as simple a way as
possible.
}

\submodlongshortalt{
  \section{Intuition, Examples, and Background}
}{
  \subsection{Intuition, Examples, and Background}
}

Let us define a \emph{set function}
$\submodfun: 2^\submodgroundset \to \real$ as one that assigns a value
to every subset of $\submodgroundset$. The notation
$2^\submodgroundset$ is the power set of $\submodgroundset$, and has
size $2^{|\submodgroundset|}$ which means that $\submodfun$ lives in
space $\real^{2^\submodgroundsetsize}$ --- i.e., since there are
$2^\submodgroundsetsize$ possible subsets of $\submodgroundset$,
$\submodfun$ can return $2^\submodgroundsetsize$ distinct values.  We
use the notation $\submodsetX + \submodel$ as shorthand for
$\submodsetX \cup \{ \submodel \}$.  Also, the value of an element in
a given context is so widely used a concept, we have a special
notation for it --- the incremental value \emph{gain} of $\submodel$
in the context if $\submodsetX$ is defined as
$\submodfun(\submodel|\submodsetX) = \submodfun(\submodsetX +
\submodel) - \submodfun(\submodsetX)$. Thus, while
$\submodfun(\submodel)$ is the value of element $\submodel$,
$\submodfun(\submodel|\submodsetX)$ is the value of element
$\submodel$ if you already have $\submodsetX$.  We also define the
gain of set $\submodsetX$ in the context of $\submodsetY$ as
$\submodfun(\submodsetX|\submodsetY) = \submodfun(\submodsetX \cup
\submodsetY) - \submodfun(\submodsetY)$.

\submodlongonly{
\subsubsection{The Value of Friendship}

To introduce submodularity, we begin with the value of friendship.
Let's say that $\submodfun(\submodsetX)$ measures the value of a set
$\submodsetX$ of friends and $\submodel$ is a friend that you do not yet have.
Then $\submodfun(\submodsetX + \submodel)$ is the value once you have
gained friend $\submodel$ and
$\submodfun(\submodel|\submodsetX) = \submodfun(\submodsetX+\submodel)
- \submodfun(\submodsetX)$ is the incremental value of gaining friend
$\submodel$ if you already have the set of friends
$\submodsetX$. Submodularity says that the incremental value of any
friend $\submodel$ decreases the more friends you acquire. So if you
increase your friend set from $\submodsetX$ to
$\submodsetY \supset \submodsetX$, then the incremental value of
friend $\submodel$ cannot increase. If you have a friend Taylor, then
Taylor would be much more valuable to you as your first friend than
they would be if you already had Emerson, Tatum, Bellamy, Leon, and
Tatum as friends.  In general, submodular functions are those that
satisfy the property of diminishing returns. This means that the
incremental ``value'' or ``gain'' of $\submodel$ decreases as the
context in which $\submodel$ is considered grows from $\submodsetX$ to
$\submodsetY$.
}

\subsubsection{Coffee, Lemon, Milk, Tea}
\label{sec:coffee-lemon-milk}

\submodlongonly{As our next example, we}\submodshortonly{As a simple
  example,} will explore the manner in which the value of everyday
items may interact and combine, namely \keyword{coffee, lemon, milk,
  and tea}.  Consider the value relationships amongst the four items
coffee ($\smcoffee$), lemon ($\smlemon$), milk ($\smmilk$), and tea
($\smtea$) as shown in Figure~\ref{fig:coffee_tea_lemon_milk}. Suppose
you just woke up, and there is a function
$\submodfun : 2^\submodgroundset \to \real$ that provides the average
valuation for any subset of the items in $\submodgroundset$ where
$\submodgroundset = \{\smcoffee, \smlemon,\smmilk,\smtea\}$. You can
think of this function as giving the average price a typical person
would be willing to pay for any subset of items. Since nothing should
cost nothing, we would expect that $\submodfun(\emptyset) = 0$.
Clearly, one needs either coffee or tea in the morning, so
$\submodfun(\smcoffee) > 0$ and $\submodfun(\smtea) > 0$, and coffee
is usually more expensive than tea, so that
$\submodfun(\smcoffee) > \submodfun(\smtea)$ pound for pound. Also
more items cost more, so that, for example,
$0 < \submodfun(\smcoffee) < \submodfun(\smcoffee,\smmilk) <
\submodfun(\smcoffee,\smmilk,\smtea) <
\submodfun(\smcoffee,\smlemon,\smmilk,\smtea)$. Thus, the function
$\submodfun$ is strictly \emph{monotone}, or
$\submodfun(\submodsetX) < \submodfun(\submodsetY)$ whenever
$\submodsetX \subset \submodsetY$.

\begin{figure}
\centering
\includegraphics[page=2,width=0.95\textwidth]{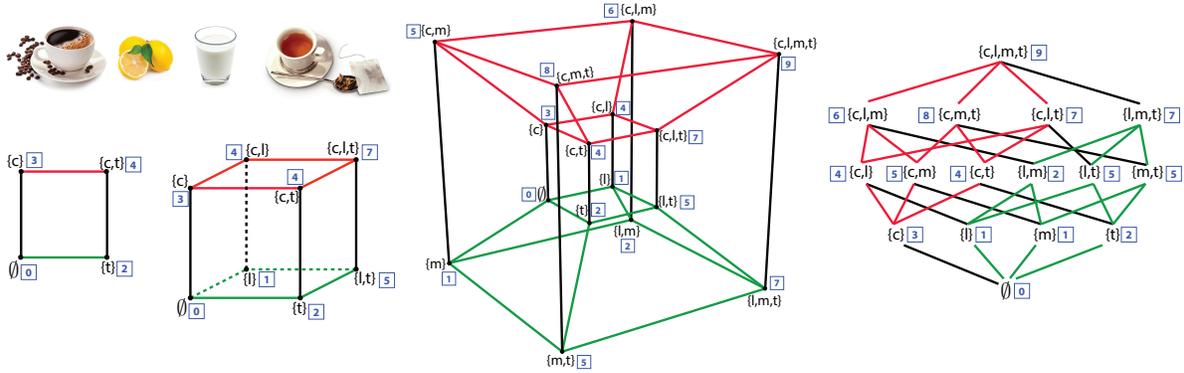}
\caption{
  The value relationships between coffee (\smcoffee), lemon (\smlemon),
  milk (\smmilk), and tea (\smtea).
  On the left, we first see a simple square showing the relationships
  between coffee and tea and see that they are substitutive (or submodular).
  In this, and all of the shapes, the vertex label set is indicated in curly braces
  and the value at that vertex is a blue integer in a box.
  We next see a three-dimensional cube that adds lemon to the coffee and tea set.
  We see that tea and lemon are complementary (supermodular), but
  coffee and lemon are additive (modular, or independent). We
  next see a four-dimensional hypercube (tesseract) showing all of the
  value relationships described in the text. The four-dimensional hypercube
  is also shown as a lattice (on the right) showing the same relationships
  as well as two (red and green, also shown in the tesseract) of the eight three-dimensional cubes contained within.
}
\label{fig:coffee_tea_lemon_milk}
\end{figure}

The next thing we note is that coffee and tea may substitute for each
other -- they both have the same effect, waking you up.  \submodlongonly{Once you have
tea you do not need as much coffee, and vice versa.}  They are
mutually redundant, and they decrease each other's value since once you have
had a cup of coffee, a cup of tea is less necessary and less
desirable. Thus,
$\submodfun(\smcoffee,\smtea) < \submodfun(\smcoffee) +
\submodfun(\smtea)$, which is known as a \emph{subadditive}
relationship, the whole is less than the sum of the parts. On the
other hand, some items complement each other. For example, milk and
coffee are better combined together than when both are considered in
isolation, or
$\submodfun(\smmilk,\smcoffee) > \submodfun(\smmilk) +
\submodfun(\smcoffee)$, a \emph{superadditive} relationship, the whole
is more than the sum of the parts. \submodlongonly{ A similar superadditive relationship
exists between milk and tea
($\submodfun(\smmilk,\smtea) > \submodfun(\smmilk) +
\submodfun(\smtea)$) as well as lemon and tea
($\submodfun(\smlemon,\smtea) > \submodfun(\smlemon) +
\submodfun(\smtea)$).} A few of the items do not affect each other's
price. For example, lemon and milk cost the same together as apart,
so
$\submodfun(\smlemon,\smmilk) = \submodfun(\smlemon) +
\submodfun(\smmilk)$, an \emph{additive} or \emph{modular}
relationship --- such a relationship is perhaps midway between a
subadditive and a superadditive relationship and can be seen as a
form of independence.  \submodlongonly{ A similar thing is true for lemon and coffee
$\submodfun(\smlemon,\smcoffee) = \submodfun(\smlemon) +
\submodfun(\smcoffee)$.}

Things become more interesting when we consider three or more items
together. For example, once you have tea, lemon becomes less valuable
when you acquire milk since there might be those that prefer milk to
lemon in their tea.  Similarly, milk becomes less valuable once you
have acquired lemon since there are those who prefer lemon in their
tea to milk. So, once you have tea, lemon and milk are substitutive,
you would never use both as the lemon would only curdle the milk.
These are \emph{submodular} relationships,
$\submodfun(\smlemon|\smmilk,\smtea) < \submodfun(\smlemon|\smtea)$
and $\submodfun(\smmilk|\smlemon,\smtea) < \submodfun(\smmilk|\smtea)$
each of which implies that
$\submodfun(\smlemon,\smtea) + \submodfun(\smmilk,\smtea) >
\submodfun(\smlemon,\smmilk,\smtea) + \submodfun(\smtea)$.  The value
of lemon (respectively milk) with tea decreases in the larger context
of having milk (respectively lemon) with tea, typical of submodular
relationships.

\submodlongonly{
A
submodular relationship also exists over coffee, milk, and tea.  Once
we have milk, then the presence of tea decreases the value of coffee
again since coffee and tea are substitutes perhaps even more so with
milk. This means that
$\submodfun(\smcoffee|\smtea,\smmilk) < \submodfun(\smcoffee|\smmilk)$
and
$\submodfun(\smtea|\smcoffee,\smmilk) < \submodfun(\smtea|\smmilk)$,
implying
$\submodfun(\smmilk,\smtea) + \submodfun(\smmilk,\smcoffee) >
\submodfun(\smmilk,\smcoffee,\smtea) + \submodfun(\smmilk)$.
}

\submodlongonly{ Another
submodular relationship exists over coffee, lemon, and tea. Once you
have lemon, then the presence of tea again reduces the value of
coffee. Lemon alone does not affect the price of coffee, but tea and
lemon perhaps renders coffee undesirable. The relationships are
$\submodfun(\smcoffee|\smtea,\smlemon) <
\submodfun(\smcoffee|\smlemon)$ and
$\submodfun(\smtea|\smcoffee,\smlemon) < \submodfun(\smtea|\smlemon)$,
implying
$\submodfun(\smlemon,\smtea) + \submodfun(\smcoffee,\smlemon) >
\submodfun(\smcoffee,\smlemon,\smtea) + \submodfun(\smlemon)$.
}

Not all of the items are in a submodular relationship, as sometimes
the presence of an item can increase the value of another item.
For example, once you have milk, then tea becomes still more valuable
when you also acquire lemon, since tea with the choice of either lemon
or milk is more valuable than tea with the option only of milk.
Similarly, once you have milk, lemon becomes more valuable when you
acquire tea, since lemon with milk alone is not nearly as valuable as
lemon with tea, even if milk is at hand. This means that
$\submodfun(\smtea|\smlemon,\smmilk) > \submodfun(\smtea|\smmilk)$ and
$\submodfun(\smlemon|\smtea,\smmilk) > \submodfun(\smlemon|\smmilk)$
implying
$\submodfun(\smlemon,\smmilk) + \submodfun(\smmilk,\smtea) <
\submodfun(\smlemon,\smmilk,\smtea) + \submodfun(\smmilk)$.  These are
known as \emph{supermodular} relationships, where the value increases as the
context increases.

\submodlongonly{
Another supermodular relationship exists over
lemon, milk, and tea. The value of milk increases if you have both tea
and lemon compared to if you have only lemon, and the value of tea
increases if you have both milk and lemon over if you just have lemon,
or
$\submodfun(\smmilk|\smlemon,\smtea) > \submodfun(\smmilk|\smlemon)$
and
$\submodfun(\smlemon|\smmilk,\smtea) > \submodfun(\smtea|\smlemon)$
implying
$\submodfun(\smlemon,\smtea) + \submodfun(\smlemon,\smmilk) <
\submodfun(\smlemon,\smmilk,\smtea) + \submodfun(\smlemon)$.  Other
supermodular relationships exist over lemon, milk, and coffee
($\submodfun(\smcoffee,\smlemon) + \submodfun(\smlemon,\smmilk) <
\submodfun(\smcoffee,\smlemon,\smmilk) + \submodfun(\smlemon)$), over
coffee, milk, and tea
($\submodfun(\smcoffee,\smtea) + \submodfun(\smcoffee,\smmilk) <
\submodfun(\smcoffee,\smmilk,\smtea) + \submodfun(\smcoffee)$ and
$\submodfun(\smcoffee,\smtea) + \submodfun(\smmilk,\smtea) <
\submodfun(\smcoffee,\smmilk,\smtea) + \submodfun(\smcoffee)$), and
also over coffee, lemon, and tea
($\submodfun(\smcoffee,\smtea) + \submodfun(\smcoffee,\smlemon) <
\submodfun(\smcoffee,\smlemon,\smtea) + \submodfun(\smcoffee)$).
}

\submodlongonly {

Some of the triples have a relationship that lies in between
submodularity and supermodularity. For example, if you start with
milk, then the value of lemon should be no different than if you had
started with both milk and coffee, since lemon is of no more use with
coffee and milk as it is with just milk. This means that
$\submodfun(\smlemon|\smcoffee,\smmilk) =
\submodfun(\smlemon|\smmilk)$. Similarity, lemon does not improve
coffee's value over having just milk, or
$\submodfun(\smcoffee|\smlemon,\smmilk) =
\submodfun(\smcoffee|\smmilk)$. These both imply
$\submodfun(\smcoffee,\smmilk) + \submodfun(\smlemon,\smmilk) =
\submodfun(\smcoffee,\smlemon,\smmilk) + \submodfun(\smmilk)$ which is
known as a \emph{modular} (or \emph{additive}) relationship. These are
actually conditionally modular relationships since we are measuring
the value relationships between lemon and coffee in the presence of
milk.

In the current case, as mentioned above, lemon and coffee are
already additive when considered alone. In other cases, items that are
additive alone, e.g., milk and lemon, might have interactive values in
a context, e.g., milk and lemon when considered in the context of tea
where the relationship becomes substitutive or submodular.
This is a particular type of v-structure that is representable
with Bayesian networks, depicted as $\smmilk \rightarrow \smtea \leftarrow \smlemon$,
where $\smmilk$ and $\smlemon$ are unconditionally
independent but conditionally dependent. Here, however,
the dependence need not be statistical and instead
can be combinatorial.
Another conditionally
modular relationship is coffee and lemon in the context of tea, where
we would expect that
$\submodfun(\smcoffee,\smtea) + \submodfun(\smlemon,\smtea) =
\submodfun(\smcoffee,\smlemon,\smtea) + \submodfun(\smtea)$.

}

We have asked for a set of relationships amongst various subsets of
the four items
$\submodgroundset = \{\smcoffee, \smlemon,\smmilk,\smtea\}$, Is there
a function that offers a value to each
$\submodsetX \subseteq \submodgroundset$ that satisfies all of the
above relationships?  Figure~\ref{fig:coffee_tea_lemon_milk} in fact
shows such a function.  On the left, we see a two-dimensional square
whose vertices indicate the values over subsets of
$\{ \smcoffee, \smtea \}$ and we can quickly verify that the sum of
the blue boxes on north-west (corresponding to
$\submodfun(\{\smcoffee\})$) and south-east corners (corresponding to
$\submodfun(\{\smtea\})$) is greater than the sum of the north-east
and south-west corners, expressing the required submodular
relationship.  Next on the right is a three-dimensional cube that adds
the relationship with lemon. Now we have six squares, and we see that
the values at each of the vertices all satisfy the above requirements
--- we verify this by considering the valuations at the four corners
of every one of the six faces of the cube.  Since
$|\submodgroundset|=4$, we need a four-dimensional hypercube to show
all values, and this may be shown in two ways. It is first shown as a
tesseract, a well-known three-dimensional projection of a
four-dimensional hypercube. In the figure, all vertices are labeled
both with subsets of $\submodgroundset$ as well as the function value
$\submodfun(\submodsetX)$ as the blue number in a
box.\submodlongonly{ The inner three-dimensional cube in the tesseract
  is the same as the cube second from left, and the outer
  three-dimensional cube in the tesseract includes the interactions
  with milk. Overall, there are eight three-dimensional cubes within
  the tesseract, each with six faces. } The figure on the right shows a
\emph{lattice} version of the four-dimensional hypercube, where
corresponding three-dimensional cubes are shown in green and
red. \submodlongonly{It can be verified that the values at the
  vertices, even though they consist only of integers from 0 through
  9, satisfy all of the aforementioned subadditive, superadditive,
  submodular, supermodular, and modular relationships amongst subsets
  of items.}

We thus see that a set function is defined for all subsets of a ground
set, and that they correspond to valuations at all vertices of the
hypercube. For the particular function over valuations of subsets of
coffee, lemon, milk, and tea, we have seen submodular, supermodular,
and modular relationships all in one function. Therefore, the overall
function $\submodfun$ defined in
Figure~\ref{fig:coffee_tea_lemon_milk} is neither submodular,
supermodular, nor modular. For combinatorial auctions, there is often
a desire to have a diversity of such manners of
relationships~\cite{lehmann2006combinatorial} --- representation of
these relationships can be handled by a difference of submodular
functions~\cite{narasimhan2005-subsup,iyer2012-min-diff-sub-funcs} or
a sum of a submodular and supermodular
function~\cite{bai-bp-functions-2018} (further described below).
\submodlongonly{ Nevertheless, it is useful to demonstrate submodular
  and supermodular relationships amongst particular subsets of
  everyday items.}  In machine learning, however, most of the time we
are interested in functions that are submodular (or modular, or
supermodular) everywhere\submodlongonly{, over all subsets.  The next example will
describe a famous example of such a pure submodular function}.

\submodlongonly{
\subsubsection{Entropy}
\label{sec:entropy}

Our next example of submodularity uses a concept that is likely
familiar to many machine learning scientists, namely entropy, which
measures the uncertainty or the information in a set of random
variables.  Suppose we are given a set of $\submodgroundsetsize$
discrete random variables $\entropyrv_1$, $\entropyrv_2$, \dots,
$\entropyrv_n$.  The set $\submodgroundset = \{1,2,\dots,n\}$ is an
index set, and any subset $\submodaltsetX \subseteq \submodgroundset$
of integers can be used to index a subset of random variables. Thus,
if
$\submodaltsetX = \{i_1, i_2, \dots, i_k\} \subseteq
\submodgroundset$, then
$\entropyrv_A = \{ \entropyrv_{i_1}, \entropyrv_{i_2}, \dots,
\entropyrv_{i_k} \}$.  Given a probability distribution $p$ over
$\entropyrv_\submodgroundset$, the entropy of a set $\submodaltsetX$
of random variables is defined as
$H(\entropyrv_\submodaltsetX) = - \sum_{x_\submodaltsetX}
p(\entropyrvv_\submodaltsetX) \log p(\entropyrvv_\submodaltsetX)$.  We
note that the number of terms in the summation is exponential in the
set size $|\submodaltsetX|$. There are a number of properties that are
known to be true of the entropy function, and they were established by
Shannon back in 1948. The first is that the entropy is
non-negative, i.e., $H(\entropyrv_\submodaltsetX) \geq 0$ for any
$\submodaltsetX$. The second
is monotonicity, i.e., for sets
$\submodaltsetX \subseteq \submodaltsetY$, we have that
$H(\entropyrv_\submodaltsetX) \leq H(\entropyrv_\submodaltsetY)$. This
means that more random variables cannot possess less information. The
third, and perhaps most important property of entropy is equivalent to
submodularity. The conditional mutual information between two sets
$\submodaltsetX$ and $\submodaltsetY$ of random variables given a
third set $\submodaltsetZ$, all three of which are mutually disjoint,
is defined as
$I(\entropyrv_{\submodaltsetX} ; \entropyrv_{\submodaltsetY} |
\entropyrv_\submodaltsetZ) =
\sum_{\entropyrvv_\submodaltsetX,\entropyrvv_b,\entropyrvv_\submodaltsetZ}
p(\entropyrvv_\submodaltsetX,\entropyrvv_\submodaltsetY,\entropyrvv_\submodaltsetZ)
\log \frac{
  p(\entropyrvv_\submodaltsetX,\entropyrvv_\submodaltsetY,\entropyrvv_\submodaltsetZ)p(\entropyrvv_\submodaltsetZ)
} { p(\entropyrvv_\submodaltsetX,\entropyrvv_\submodaltsetZ)
  p(\entropyrvv_\submodaltsetY,\entropyrvv_\submodaltsetZ) }$.  The
conditional mutual information can be expressed in terms of the
entropy,
$I(\entropyrv_\submodaltsetX; \entropyrv_\submodaltsetY |
\entropyrv_\submodaltsetZ) =
H(\entropyrv_\submodaltsetX,\entropyrv_\submodaltsetZ) +
H(\entropyrv_\submodaltsetY,\entropyrv_\submodaltsetZ) -
H(\entropyrv_\submodaltsetX,\entropyrv_\submodaltsetY,\entropyrv_\submodaltsetZ)
- H(\entropyrv_\submodaltsetZ)$. A well-known property is that the
conditional mutual information is non-negative, i.e.,
$I(\entropyrv_\submodaltsetX; \entropyrv_\submodaltsetY |
\entropyrv_\submodaltsetZ) \geq 0$.\submodlongonly{

We can view this inequality in a slightly different way.  First, for
any set $\submodsetX \subseteq \submodgroundset$, we define entropy as
a set function $\submodfun(\submodsetX) = H(\entropyrv_\submodsetX)$,
which gives the entropy of the set of random variables indexed by
$\submodsetX$. Then if we set
$\submodsetX = \submodaltsetX \cup \submodaltsetZ$,
$\submodsetY = \submodaltsetY \cup \submodaltsetZ$, then the above
inequality about the non-negativity of mutual information becomes
$\submodfun(\submodsetX) + \submodfun(\submodsetY) -
\submodfun(\submodsetX \cup \submodsetY) - \submodfun(\submodsetX \cap
\submodsetY) \geq 0$, which is exactly the condition for the
submodularity of the set function (see below). } Thus, the Shannon
entropy function is submodular when seen as a function of subsets of
random variables. Statement equivalent to the submodularity of the
Shannon entropy function was implied back
in Shannon's paper~\cite{shannon1948}
and made explicit in~\cite{mcgill1954multivariate}.
These inequalities therefore became known, in the
information theory community, as the
Shannon inequalities~\cite{Yeung91,zhang1997non,zhang1998characterization}.
}

\subsection{Submodular Basic Definitions}

For a function to be submodular, it must satisfy the submodular
relationship for all subsets. We arrive at the following definition.
\begin{defn}[\keywordDef{Submodular} Function]
  A given set function $\submodfun : 2^\submodgroundset \to \real$ is submodular
  if for all $\submodsetX,\submodsetY \subseteq \submodgroundset$, we have the following inequality:
  \begin{equation}
    \submodfun(\submodsetX) + \submodfun(\submodsetY) \geq \submodfun(\submodsetX \cup \submodsetY) + \submodfun(\submodsetX \cap \submodsetY)
    \label{eqn:main_submodular_inequality}
  \end{equation}
  \label{defn:main_submodular_def}
  \vspace{-1.5\baselineskip}
\end{defn}
There are also many other equivalent definitions of submodularity\submodshortonly{~\cite{bilmes-submod-and-ml-2022}} some
of which are more intuitive and easier to understand.  For example, submodular
functions are those set functions that satisfy the property of
\keywordIndex{diminishing returns}. If we think of a function $\submodfun(\submodsetX)$ as measuring the
value of a set $\submodsetX$ that is a subset of a larger set of data items
$\submodsetX \subseteq \submodgroundset$, then the submodular property means that the
incremental ``value'' of adding a data item $\submodel$ to set $\submodsetX$ decreases
as the size of $\submodsetX$ grows. This gives us a second classic definition of
submodularity.
\begin{defn}[Submodular Function via Diminishing Returns]
  A given set function $\submodfun : 2^\submodgroundset \to \real$ is
  submodular if for all
  $\submodsetX,\submodsetY \subseteq \submodgroundset$, where
  $\submodsetX \subseteq \submodsetY$ and for all
  $\submodel \notin \submodsetY$, we have the following inequality:
  \begin{equation}
    \submodfun(\submodsetX + \submodel) - \submodfun(\submodsetX) \geq \submodfun(\submodsetY + \submodel) - \submodfun(\submodsetY)%
    \label{eqn:diminishing_returns_submodular_inequality}%
  \end{equation}\label{defn:diminishing_returns_submodular_def}%
  \vspace{-1.5\baselineskip}
\end{defn}%
The property that the incremental value of lemon with tea
is less than the incremental value of lemon once milk is already in
the tea is equivalent to Equation~\ref{eqn:main_submodular_inequality}
if we set $\submodsetX = \{ \smmilk,\smtea\}$ and
$\submodsetY = \{ \smlemon, \smtea\}$ (i.e.
$\submodfun(\smmilk,\smtea) + \submodfun(\smlemon,\smtea) >
\submodfun(\smlemon,\smmilk,\smtea) + \submodfun(\smtea)$).  It is
naturally also equivalent to
Equation~\ref{eqn:diminishing_returns_submodular_inequality} if we set
$\submodsetX = \{ \smtea \}$, $\submodsetY = \{ \smmilk,\smtea \}$,
and with $\submodel = \smlemon$ (i.e.,
$\submodfun(\smlemon|\smmilk,\smtea) < \submodfun(\smlemon|\smtea)$).
\submodlongonly{
The function listed in Figure~\ref{fig:coffee_tea_lemon_milk} is not
submodular for all subsets as mentioned above.  The entropy function,
however, seen as a function of a set of random variables, is
submodular over all subsets, and this is widely known as the ``further
conditioning reduces entropy'' property of the entropy function.
}

\submodlongonly{

It is not hard to see how Equation~\eqref{defn:main_submodular_def}
and Equation~\eqref{defn:diminishing_returns_submodular_def} are
mathematically identical, as is done in the following:
\begin{thm}
Given function $f: 2^V \to \mathbb R$, then
\begin{align}
f(A) + f(B) \geq f(A \cup B) + f(A \cap B)
\text{ for all } A,B \subseteq V \tag{Eq.~\eqref{defn:main_submodular_def}}
\end{align}
if and only if
\begin{align}
f(v|X)\geq f(v|Y) \text{ for all } X \subseteq Y \subseteq V \text{ and }
v \notin Y \tag{Eq.~\eqref{defn:diminishing_returns_submodular_def}}
\end{align}
\end{thm}
\begin{proof}

Eq.~\eqref{defn:main_submodular_def}$\Rightarrow$Eq.~\eqref{defn:diminishing_returns_submodular_def}:
Set $A \gets X \cup \{ v\}$, $B \gets Y$. Then
$A \cup B = Y \cup \{ v\}$ and $A \cap B = X$
and $f(A) - f(A \cap B) \geq f(A \cup B) - f(B)$ implies Eq.~\eqref{defn:diminishing_returns_submodular_def}.

Eq.~\eqref{defn:diminishing_returns_submodular_def}$\Rightarrow$Eq.~\eqref{defn:main_submodular_def}:
Order $A \setminus B = \{v_1, v_2, \dots, v_r\}$ arbitrarily. For
$i \in [r]$,
\begin{align}
  f(v_i | (A \cap B) \cup \{ v_1, v_2, \dots, v_{i-1}\} )
  \geq f(v_i | B \cup \{ v_1, v_2, \dots, v_{i-1} \}). 
\notag
\end{align}
Applying telescoping summation to both sides, we get:
\begin{align*}
\sum_{i=1}^r 
f(v_i | (A \cap B) \cup \{ v_1, v_2, \dots, v_{i-1}\} )
&\geq 
\sum_{i=1}^r 
f(v_i | B \cup \{ v_1, v_2, \dots, v_{i-1} \}) \\
\Rightarrow \qquad f(A) - f(A \cap B) &\geq f(A \cup B) - f(B)  
\end{align*}
\end{proof}

There are many different ways to define submodularity.
Theorem~\ref{thm:many_defs_of_submodularity} gives a list of
equivalencies --- if any one of the equations are true as stated, then
the function is submodular.

\begin{thm}[The Many Definitions of Submodularity]
  \label{thm:many_defs_of_submodularity} Given a function
  $\submodfun : 2^{\submodgroundset} \to \real$, and the gain function defined as
  $\submodfgain{\submodfun}{\SetI}{\Set} = \submodfun(\SetI \cup \Set) - \submodfun(\Set)$, for
  any sets $\Set,\SetI \subseteq \submodgroundset$,
  the following equations are equivalent. In other words, any one of
  them can be used as the definition of a submodular function.
\begin{flalign}
\submodfun(\submodaltsetX) + \submodfun(\submodaltsetY) 
& \geq \submodfun(\submodaltsetX \cup \submodaltsetY) + \submodfun(\submodaltsetX \cap \submodaltsetY),\;\; \forall \submodaltsetX,\submodaltsetY \subseteq \submodgroundset
&& \llap{(Classic)}\tag{\ref{eqn:main_submodular_inequality}} \\ %
\submodfgain{\submodfun}{\submodvv}{\SetS} & \geq \submodfgain{\submodfun}{\submodvv}{\SetT}, \; 
\forall \SetS \subseteq \SetT \subseteq \submodgroundset, \;
\forall \submodvv \in \submodgroundset \setminus \SetT && \llap{(Dim.\ Returns)}\tag{\ref{eqn:diminishing_returns_submodular_inequality}}\\ %
\submodfgain{\submodfun}{\submodvvx}{\SetS} &\geq \submodfgain{\submodfun}{\submodvvx}{\setplussingle{\SetS}{\submodvvy}}, \;
\forall \SetS \subseteq \submodgroundset,\;
\forall \submodvvx \in \submodgroundset \setminus (\setplussingle{\SetS}{\submodvvy}) && \llap{(Four Points)} \label{eqn:submodularity_inequality_on_hypercube_face}\\ %
\submodfgain{\submodfun}{\SetC}{\SetS} &\geq \submodfgain{\submodfun}{\SetC}{\SetT},\; 
\forall \SetS \subseteq \SetT \subseteq \submodgroundset, \;
\forall \SetC \subseteq \submodgroundset \setminus \SetT && \llap{(Group\ Dim.\ Returns)}\eqnlabel{submod_group_dim_returns} \\
\submodfgain{\submodfun}{\submodaltsetX \cup \submodaltsetY}{\submodaltsetX \cap \submodaltsetY}
& \leq \submodfgain{\submodfun}{\submodaltsetX}{\submodaltsetX \cap \submodaltsetY} + \submodfgain{\submodfun}{\submodaltsetY}{\submodaltsetX \cap \submodaltsetY},\;
\forall \submodaltsetX, \submodaltsetY \subseteq \submodgroundset && \llap{(Conditional\ Subadditivity)}\eqnlabel{submod_cond_subadditive} \\
\submodfun(\SetT) &\leq \submodfun(\SetS) 
+ \sum_{\submodvv \in \SetT \setminus \SetS } \submodfgain{\submodfun}{\submodvv}{\SetS} 
&& \llap{(Union ULB)}\eqnlabel{submod_union_ulb}  \\
& \qquad 
- \sum_{\submodvv \in \SetS \setminus \SetT} \submodfgain{\submodfun}{\submodvv}{ \setminussingle{(\SetS \cup \SetT)}{\submodvv}}, \; 
\forall \SetS,\SetT \subseteq \submodgroundset &&  \notag \\
\submodfun(\SetT) & \leq \submodfun(\SetS) + \sum_{\submodvv \in \SetT \setminus \SetS} \submodfgain{\submodfun}{\submodvv}{\SetS}, \; 
\forall \SetS \subseteq \SetT \subseteq \submodgroundset && \llap{(Union UB)}\eqnlabel{submod_union_ub} \\
\submodfun(\SetT) & \leq \submodfun(\SetS) 
+ \sum_{\submodvv \in \SetT \setminus \SetS} \submodfgain{\submodfun}{\submodvv}{\SetS \cap \SetT}
&& \llap{(Intersection ULB)}\eqnlabel{submod_intersection_ulb} \\
& \qquad 
- \sum_{\submodvv \in \SetS \setminus \SetT} \submodfgain{\submodfun}{\submodvv}{\setminussingle{\SetS}{\submodvv}}, \; 
\forall \SetS,\SetT \subseteq \submodgroundset &&  \notag \\
\submodfun(\SetT) & \leq \submodfun(\SetS) - \sum_{\submodvv \in \SetS \setminus \SetT} 
\submodfgain{\submodfun}{\submodvv}{\setminussingle{\SetS}{\submodvv}}, \; 
\forall \SetT \subseteq \SetS \subseteq \submodgroundset && \llap{(Intersection UB)}\eqnlabel{submod_intersection_ub} \\
\submodfun(\SetT) &\leq 
\submodfun(\SetS) 
+ \sum_{\submodvv \in \SetT \setminus \SetS} \submodfgain{\submodfun}{\submodvv}{\SetS \cap \SetT}
&&  \llap{(Union/Intersection)}\eqnlabel{submod_union_intersection} \\
&\qquad 
- \sum_{\submodvv  \in \SetS \setminus \SetT} \submodfgain{\submodfun}{\submodvv}{ \setminussingle{(\SetS \cup \SetT)}{\submodvv}},  \; 
\forall \SetS,\SetT \subseteq \submodgroundset \notag
\end{flalign}
\end{thm}

}

\submodlongshortalt{While entropy as seen above is submodular, there are many other
functions as well, some of which we examine below.}{There are many functions that are submodular, one famous one being Shannon entropy seen as a function of subsets of random variables.} We first point out
that there are non-negative (i.e.,
$\submodfun(\submodaltsetX) \geq 0, \forall \submodaltsetX$), monotone
non-decreasing (i.e.,
$\submodfun(\submodaltsetX) \leq \submodfun(\submodaltsetY)$ whenever
$\submodaltsetX \subseteq \submodaltsetY$) submodular functions that
are not
entropic~\cite{Yeung91,zhang1997non,zhang1998characterization}, so
submodularity is not just a trivial restatement of the class of
entropy functions.  When a function is monotone non-decreasing,
submodular, and \emph{normalized} so that $\submodfun(\emptyset) = 0$,
it is often referred to as a \keywordDef{polymatroid function}. Thus, while
the entropy function is a polymatroid function, it does not encompass
all polymatroid functions even though all polymatroid functions
satisfy the properties Claude Shannon mentioned as being
natural for an ``information'' function (see Section~\ref{sec:comb-inform-funct}).
\submodlongonly{Most polymatroid functions can be seen as a form
of dispersion, diversity, or information function, in that they tend
to give high valuation to sets $\submodaltsetX$, of a given size, that have a lot
of diversity (a high amount of information about $\submodgroundset$), and tend to
give low value to sets, of the same size, that are more homogeneous (a
low amount of information about $\submodgroundset$). The entropy function is a good
example of this, where diversity over a set of random variables is
determined by the degree to which the set of random variables are
statistically independent and the degree to which they have high
marginal entropy.
}

\submodlongonly{
A supermodular function has the following (reversed inequality
relative to submodularity) definition:
\begin{defn}[Increasing Returns]
  A given set function $\submodfun : 2^\submodgroundset \to \real$ is \emph{supermodular}
  if for all $\submodsetX,\submodsetY \subseteq \submodgroundset$, where $\submodsetX \subseteq \submodsetY$ and
  for all $\submodel \notin \submodsetY$,
  we have the following inequality:
  \begin{equation}
    \submodfun(\submodsetX + \submodel) - \submodfun(\submodsetX) \leq \submodfun(\submodsetY + \submodel) - \submodfun(\submodsetY)
    \label{eqn:increasing_returns_submodular_inequality}
  \end{equation}
\end{defn}

We can alternatively define supermodularity analogous to
Definition~\ref{defn:main_submodular_def}, where we say the function
$\submodfun$ is supermodular if
$\submodfun(\submodsetX) + \submodfun(\submodsetY) \leq
\submodfun(\submodsetX \cup \submodsetY) + \submodfun(\submodsetX \cap
\submodsetY)$ for all arguments.
}

\submodlongonly{Alternatively, a}\submodshortonly{A} function
$\submodfun$ is supermodular if and only if $-\submodfun$ is
submodular. If a function is both submodular and supermodular, it is
known as a \emph{modular} function. It is always the case that modular
functions may take the form of a vector-scalar pair
$(\submodmodfun,\submodmodfunconst)$ where
$\submodmodfun : 2^\submodgroundset \to \real$ and where
$\submodmodfunconst \in \real$ is a constant, and where for any
$\submodaltsetX \subseteq \submodgroundset$, we have that
$\submodmodfun(\submodaltsetX) = \submodmodfunconst + \sum_{\submodel
  \in \submodaltsetX} \submodmodfun_\submodel$. If the modular
function is normalized, so that $\submodmodfun(\emptyset) = 0$, then
$\submodmodfunconst=0$ and the modular function can be seen simply as
a vector $\submodmodfun \in \real^\submodgroundset$. Hence, we
sometimes say that the modular function
$\submodmodaltfun \in \real^\submodgroundset$ offers a value for set
$\submodaltsetX$ as the partial sum
$\submodmodaltfun( \submodaltsetX ) = \sum_{\submodel \in
  \submodaltsetX} \submodmodaltfun(\submodel)$. Many combinatorial
problems use modular functions as objectives. For example, the graph
cut problem uses a modular function defined over the edges, judges a cut
in a graph as the modular function applied to the edges that comprise
the cut.

As can be seen from the above, and by considering
Figure~\ref{fig:coffee_tea_lemon_milk}, a submodular function, and in
fact any set function, $\submodfun : 2^\submodgroundset \to \real$ can
be seen as a function defined only on the vertices of the
$\submodgroundsetsize$-dimensional unit hypercube
$[0,1]^\submodgroundsetsize$. Given any set
$\submodsetX \subseteq \submodgroundset$, we define
$\submodcharv_\submodsetX \in \{0,1\}^\submodgroundset$ to be the
characteristic vector of set $\submodsetX$ defined as
$\submodcharv_\submodsetX(\submodel) = 1$ if
$\submodel \in \submodsetX$ and
$\submodcharv_\submodsetX(\submodel) = 0$ otherwise. This gives us a
way to map from any set $\submodsetX \subseteq \submodgroundset$ to a
binary vector $\submodcharv_\submodsetX$. We also see that
$\submodcharv_\submodsetX$ is itself a modular function since
$\submodcharv_\submodsetX \in \{0,1\}^\submodgroundset \subset
\real^\submodgroundset$.

\submodlongonly{
A submodular function is not defined using relationships only between
pairs of objects like in a typical weighted graph where every edge
connects two vertices.  Rather, a submodular function provides a
distinct value for every unique subset
$\submodsetX \subseteq \submodgroundset$, and the value
$\submodfun(\submodsetX)$ need not be a function of the values of
pairs of elements $\submodfun(\submodelI,\submodelJ)$ for
$\submodelI,\submodelJ \in \submodsetX$. Most generally, a submodular
function requires the specification of $2^\submodgroundsetsize$
distinct parameters.  To compare a submodular function with a
graphical model, for example, a submodular function (most generally)
corresponds to a single clique over a set of $\submodgroundsetsize$
variables. Such a graphical model, which normally encodes
factorization properties via graph separation properties, would have
no factorization properties since nothing is separated in a clique.
It is for this reason that submodular functions can be seen to be a
form of ``anti-graphical'' model. In other words,
submodular/supermodular functions and graphical models describe
distinct properties --- a graphical model depicts who directly
interacts with whom while submodular (or supermodular) function states
the manner that a group of items may interact amongst themselves. Both
graphical models and submodular functions are useful, and they are in
fact complementary, since we can use a graphical model to decompose
how a submodular function is formed. As a simple example, suppose
$\submodgroundset = \{ \submodel, \submodell, \submodelll,
\submodellll \}$ and we form the function
$\submodfun(\submodsetX) = \sqrt{|\submodsetX \cap \{ \submodel,
  \submodell \}|} + \sqrt{|\submodsetX \cap \{ \submodell, \submodelll
  \}|} + \sqrt{|\submodsetX \cap \{ \submodelll, \submodellll
  \}|}$. Then, not only is function submodular (see below for why), it
abides by the rules of a chain graphical model
$\submodel \relbar \submodell \relbar \submodelll \relbar
\submodellll$. That is, a graphical model describes which variables
may or may not directly interact with each other depending on the
presence or absence of edges in a graph, a submodular (or
supermodular) function specifies \emph{how} variables interact.
}

Submodular functions share
a number of
properties in common with both convex and concave functions
\cite{lovasz1983submodular}, including wide applicability, generality,
multiple representations, and closure
under 
a number of 
common operators (including mixtures, truncation, complementation, and
certain convolutions).
There is one important submodular closure property
that we state here --- that is that if we take non-negative weighted
(or conical) combinations of submodular functions, we preserve
submodularity. In other words, if we have a set of $k$ submodular
functions, $\submodfun_i: 2^\submodgroundset \to \real$, $i \in [k]$,
and we form
$\submodfun(\submodsetX) = \sum_{i=1}^k \submodfuncweight_i
\submodfun_i(\submodsetX)$ where $\submodfuncweight_i \geq 0$ for all
$i$, then Definition~\ref{defn:main_submodular_def} immediately
implies that $\submodfun$ is also submodular.  When we consider
Definition~\ref{defn:main_submodular_def}, we see that submodular
functions live in a cone in $2^\submodgroundsetsize$-dimensional space
defined by the intersection of an exponential number of half-spaces
each one of which is defined by one of the inequalities of the form
$\submodfun(\submodsetX) + \submodfun(\submodsetY) \geq
\submodfun(\submodsetX \cup \submodsetY) + \submodfun(\submodsetX \cap
\submodsetY)$. Each submodular function is therefore a point in that
cone. It is therefore not surprising that taking conical combinations
of such points stays within this cone.

\subsection{Example Submodular Functions}
\label{sec:example-subm-funct}

As mentioned above, there are many functions that are submodular
besides entropy. Perhaps the simplest such function is
$\submodfun(\submodaltsetX) = \sqrt{|\submodaltsetX|}$ which is the
composition of the square-root function (which is concave) with the
cardinality $|\submodaltsetX|$ of the set $\submodaltsetX$. The gain
function is
$\submodfun(\submodaltsetX+\submodel) - \submodfun(\submodaltsetX) =
\sqrt{k+1} - \sqrt{k}$ if $|\submodaltsetX|=k$, which we know to be a
decreasing in $k$, thus establishing the submodularity of
$\submodfun$. In fact, if $\submodconcavefun : \real \to \real$ is any
concave function, then
$\submodfun(\submodaltsetX) = \submodconcavefun(|\submodaltsetX|)$
will be submodular for the same reason.\footnote{While we will not be
  extensively discussing supermodular functions in this section,
  $\submodfun(\submodaltsetX) = \submodconcavefun(|\submodaltsetX|)$
  is supermodular for any convex function $\submodconcavefun$.}
Generalizing this slightly further, a function defined as
$\submodfun(\submodaltsetX) = \submodconcavefun(\sum_{\submodaltelx
  \in \submodaltsetX} \submodmodfun(\submodaltelx))$ is also
submodular, whenever $\submodmodfun(\submodaltelx) \geq 0$ for all
$\submodaltelx \in \submodgroundset$. This yields a composition of a
concave function with a modular function
$\submodfun(\submodaltsetX) =
\submodconcavefun(\submodmodfun(\submodaltsetX))$ since
$\sum_{\submodaltelx \in \submodaltsetX} \submodmodfun(\submodaltelx)
= \submodmodfun(\submodaltsetX)$.  We may take sums of such functions
as well as add a final modular function without losing submodularity, leading to
$\submodfun(\submodaltsetX) = \sum_{\submodfeatureel \in
  \submodfeatureset} \submodconcavefun_\submodfeatureel
(\sum_{\submodaltelx \in \submodaltsetX}
\submodmodfun_{\submodfeatureel}(\submodaltelx))
+
\sum_{\submodaltelx \in \submodaltsetX}
 \submodmodfun_{\pm}(\submodaltelx)$ where
$\submodconcavefun_\submodfeatureel$ can be a distinct concave
function for each $\submodfeatureel$,
$\submodmodfun_{\submodfeatureel,\submodaltelx}$ is a non-negative
real value for all $\submodfeatureel$ and $\submodaltelx$,
and $\submodmodfun_{\pm}(\submodaltelx)$ is an arbitrary real number.
Therefore, $\submodfun(\submodaltsetX) = \sum_{\submodfeatureel \in
  \submodfeatureset} \submodconcavefun_\submodfeatureel
(\submodmodfun_{\submodfeatureel}(\submodaltsetX))
+ \submodmodfun_{\pm}(\submodaltsetX)$ where
$\submodmodfun_{\submodfeatureel}$ is a $\submodfeatureel$-specific
non-negative modular function
and $\submodmodfun_{\pm}$ is an arbitrary modular function.
Such functions are sometimes known as
\keywordDef{feature-based} submodular functions~\cite{bilmes-dsf-arxiv-2017} because $\submodfeatureset$
can be a set of non-negative features (in the machine-learning
``bag-of-words'' sense) and this function measures a form of
dispersion over $\submodaltsetX$ as determined by the set of features
$\submodfeatureset$.

A function such as
$\submodfun(\submodaltsetX) = \sum_{\submodfeatureel \in
  \submodfeatureset} \submodconcavefun_\submodfeatureel
(\submodmodfun_{\submodfeatureel}(\submodaltsetX))$ tends to award
high diversity to a set $\submodaltsetX$ that has a high valuation by
a distinct set of the features $\submodfeatureset$. The reason is
that, due to the concave nature of
$\submodconcavefun_\submodfeatureel$, any addition to the argument
$\submodmodfun_{\submodfeatureel}(\submodaltsetX)$ by adding, say,
$\submodel$ to $\submodaltsetX$ would diminish as $\submodaltsetX$
gets larger. In order to produce a set larger than $\submodaltsetX$
that has a much larger valuation, one must use a feature
$\submodfeatureel' \neq \submodfeatureel$ that has not yet diminished
as much.

\emph{Facility location} is another well-known submodular function ---
perhaps an appropriate nickname would be the ``$k$-means of submodular functions,''
due to its applicability, utility, ease-of-use (it needs only an
affinity matrix), and similarity to $k$-medoids problems.  The
facility location function is defined using an affinity matrix as
follows:
$\submodfun(\submodaltsetX) = \sum_{v \in \submodgroundset}
\max_{\submodaltelx \in \submodaltsetX}
\text{sim}(\submodaltelx,\submodel)$ where
$\text{sim}(\submodaltelx,\submodel)$ is a non-negative measure of the
affinity (or similarity) between element $\submodaltelx$ and
$\submodel$.  Here, every element $\submodel \in \submodgroundset$
must have a representative within the set $\submodaltsetX$ and the
representative for each $\submodel \in \submodgroundset$ is chosen to
be the element $\submodaltelx \in \submodaltsetX$ most similar to
$\submodel$.  This function is also a form of dispersion or diversity
function because, in order to maximize it, every element
$\submodel \in \submodgroundset$ must have some element similar to it
in $\submodaltsetX$.  The overall score is then the sum of the
similarity between each element $\submodel \in \submodgroundset$ and
$\submodel$'s representative. This function is monotone (since as
$\submodaltsetX$ includes more elements to become
$\submodaltsetY \supseteq \submodaltsetX$, it is possible only to find
an element in $\submodaltsetY$ more similar to a given $\submodel$
than an element in $\submodaltsetX$).

While the facility location looks quite
different from a feature-based function, it is possible to precisely
represent any facility location function with a feature-based
function. \submodlongonly{One way to see this is to consider
$\max_{\submodaltelx \in \submodaltsetX}
\submodmodaltfun_\submodaltelx = \lim_{\gamma \to \infty}
\frac{1}{\gamma}\log(\sum_{\submodaltelx \in \submodaltsetX}
\exp(\gamma \submodmodaltfun_\submodaltelx ))$, for non-negative
values $\submodmodaltfun_\submodaltelx$.  Thus, the $\max$ function
acts as a kind of sharp concave function.  Alternatively, we can
capture the facility location directly by a feature-based function,
term by term, by sorting the individual modular functions.}  Consider
just
$\max_{\submodaltelx \in \submodaltsetX}
\submodmodaltfun_\submodaltelx$ and, without loss of generality,
assume that
$0 \leq \submodmodaltfun_1 \leq \submodmodaltfun_2 \leq \dots \leq
\submodmodaltfun_n$. Then
$\max_{\submodaltelx \in \submodaltsetX}
\submodmodaltfun_\submodaltelx = \sum_{i=1}^\submodgroundsetsize
\submodmodaltaltfun_i \min(|\submodaltsetX \cap \{ i, i+1, \dots,
\submodgroundsetsize \}|,1)$ where
$\submodmodaltaltfun_i = \submodmodaltfun_i - \submodmodaltfun_{i-1}$
and we set $\submodmodaltfun_0 = 0$. We note that this is a sum of
weighted concave composed with modular functions since
$\min(\alpha,1)$ is concave in $\alpha$, and
$|A \cap \{ i, i+1, \dots, \submodgroundsetsize \}|$ is a modular
function in $\submodaltsetX$.  Thus, the facility location function, a sum of these, is
merely a feature-based function. \submodlongonly{In fact, since \emph{budgeted
  additive functions}~\cite{devanur2013approximation} are the same as \emph{truncated modular
  functions} (i.e., those that take the form
$\submodfun(\submodsetX) = \min(\alpha, \sum_{\submodel \in \submodsetX}
\submodmodfun(\submodel))$), we see that a facility location function
can be represented by class of functions smaller than feature-based
functions.}

Feature-based functions, in fact, are quite expressive, and can be
used to represent many different submodular functions including set
cover and graph-based functions. For example, we can define a
\emph{set cover function}, given a set of sets
$\{ \submodfeatureset_\submodel \}_{\submodel \in \submodgroundset}$,
via
$\submodfun(\submodsetX) = \left| \bigcup_{\submodel \in \submodsetX}
  \submodfeatureset_\submodel \right|$.  If
$\submodfun(\submodsetX) = |\submodfeatureset|$ where
$\submodfeatureset = \bigcup_{\submodel \in \submodgroundset}
\submodfeatureset_\submodel$ then $\submodsetX$ indexes a set that
fully covers $\submodfeatureset$.  This can also be represented as
$\submodfun(\submodsetX) = \sum_{\submodfeatureel \in
  \submodfeatureset} \min( 1,
\submodmodfun_\submodfeatureel(\submodsetX) )$ where
$\submodmodfun_\submodfeatureel(\submodsetX)$ is a modular function
where $\submodmodfun_\submodfeatureel(\submodel) = 1$ if and only if
$\submodfeatureel \in \submodfeatureset_\submodel$ and otherwise
$\submodmodfun_\submodfeatureel(\submodel) = 0$. We see that this is a
feature-based submodular function since $\min(1,x)$ is concave in $x$,
and $\submodfeatureset$ is a set of features.

This construct can be
used to produce the vertex cover function if we set
$\submodfeatureset = \submodgroundset$ to be the set of vertices in a
graph, and set $\submodmodfun_\submodfeatureel(\submodel) = 1$ if and
only if vertices $\submodfeatureel$ and $\submodel$ are adjacent in
the graph and otherwise set
$\submodmodfun_\submodfeatureel(\submodel) = 1$.
Similarly, the edge cover function can be expressed
by setting $\submodgroundset$ to be
the set of edges in a graph, $\submodfeatureset$
to be the set of vertices in the graph,
and $\submodmodfun_\submodfeatureel(\submodel) = 1$ if and
only edge $\submodel$ is incident to
vertex $\submodfeatureel$.

A generalization of the
set cover function is the \emph{probabilistic coverage} function.  Let
$\submodprob{\submodbernoullirv_{\submodfeatureel,\submodel} = 1}$ be the
probability of the presence of feature (or concept) $\submodfeatureel$
within element $\submodel$.  Here, we treat
$\submodbernoullirv_{\submodfeatureel,\submodel}$ as a Bernoulli
random variable for each element $\submodel$ and feature
$\submodfeatureel$ so that
$\submodprob{\submodbernoullirv_{\submodfeatureel,\submodel} = 1} = 1 -
\submodprob{\submodbernoullirv_{\submodfeatureel,\submodel} = 0}$.  Then we
can define the probabilistic coverage function 
as 
$\submodfun(\submodsetX) 
= \sum_{\submodfeatureel \in \submodfeatureset} \submodfun_\submodfeatureel(\submodsetX)$
where,
for feature
$\submodfeatureel$, we have
$\submodfun_\submodfeatureel(\submodsetX) = 1 - \prod_{\submodel \in
  \submodsetX} ( 1 -
\submodprob{\submodbernoullirv_{\submodfeatureel,\submodel}=1} )$ which
indicates the degree to which feature $\submodfeatureel$ is
``covered'' by $\submodsetX$.  If we set
$\submodprob{\submodbernoullirv_{\submodfeatureel,\submodel} = 1} = 1$ if
and only if $\submodfeatureel \in \submodfeatureset_\submodel$ and
otherwise
$\submodprob{\submodbernoullirv_{\submodfeatureel,\submodel}=1} = 0$, then
$\submodfun_\submodfeatureel(\submodsetX) = \min( 1,
\submodmodfun_\submodfeatureel(\submodsetX) )$ and the set cover
function can be represented as
$\sum_{\submodfeatureel \in \submodfeatureset}
\submodfun_\submodfeatureel(\submodsetX)$.  We can generalize this in
two ways. First, to make it softer and more probabilistic we allow
$\submodprob{\submodbernoullirv_{\submodfeatureel,\submodel} = 1}$ to be any
number between zero and one. We also allow each feature to have a
non-negative weight. This yields the general form of the probabilistic
coverage function, which is defined by taking a weighted combination
over all features:
$\submodfun_\submodfeatureel(\submodsetX) = \sum_{\submodfeatureel \in
  \submodfeatureset} \submodfuncweight_\submodfeatureel
\submodfun_\submodfeatureel(\submodsetX)$ where
$\submodfuncweight_\submodfeatureel \geq 0$ is a weight for feature
$\submodfeatureel$.  Observe that
$1 - \prod_{\submodel \in \submodsetX} ( 1 -
\submodprob{\submodbernoullirv_{\submodfeatureel,\submodel}=1} ) = 1 -
\exp(- \submodmodfun_\submodfeatureel(\submodsetX) ) =
  \submodconcavefun(\submodmodfun_\submodfeatureel(\submodsetX))$
where $\submodmodfun_\submodfeatureel$ is a modular function with
evaluation
$\submodmodfun_\submodfeatureel(\submodsetX) = \sum_{\submodel \in
  \submodsetX} \log
\bigl(1/(1-\submodprob{\submodbernoullirv_{\submodfeatureel,\submodel}=1} )\bigr)$
and for $\submodmodrealvalue \in \real$,
$\submodconcavefun(\submodmodrealvalue) =
1-\exp(-\submodmodrealvalue)$ is a concave function. Thus, the
probabilistic coverage function (and its set cover specialization) is
also a feature-based function. \submodlongonly{A variant of this function utilizes the
probability $\submodprob{\submodfeatureel|\submodel}$ rather than
$\submodprob{\submodbernoullirv_{\submodfeatureel,\submodel}=1}$, and while
this form is not able to perfectly represent the set cover function,
it is useful for many applications, for example when
$\submodprob{\submodfeatureel|\submodel}$ indicates the probability of topic
$\submodfeatureel$ being present in document $\submodel$.}

Another common submodular function is the graph cut function. Here, we
measure the value of a subset of $\submodgroundset$ by the edges that
cross between a set of nodes and all but that set of nodes.  We are
given an undirected non-negative weighted graph
$\submodgraph=(\submodgroundset,\submodedgesgroundset,\submodedgeweight)$
where $\submodgroundset$ is the set of nodes,
$\submodedgesgroundset \subseteq \submodgroundset \times
\submodgroundset$ is the set of edges, and
$\submodedgeweight \in \real^\submodedgesgroundset_+$ are non-negative
edge weights corresponding to symmetric matrix
(so $\submodedgeweight_{\submodelI,\submodelJ}
= \submodedgeweight_{\submodelJ,\submodelI}$).  For any $\submodedgeel \in \submodedgesgroundset$, we
have $\submodedgeel = \{\submodelI,\submodelJ\}$ for some
$\submodelI,\submodelJ \in \submodgroundset$ with
$\submodelI \neq \submodelJ$, the graph cut function
$\submodfun: 2^\submodgroundset \to \real$ is defined as
$\submodfun(\submodsetX) = \sum_{\submodelI \in \submodsetX,
  \submodelJ \in \bar \submodsetX}
\submodedgeweight_{\submodelI,\submodelJ}$ where
$\submodedgeweight_{\submodelI,\submodelJ} \geq 0$ is the weight of
edge $\submodedgeel = \{\submodelI,\submodelJ\}$
($\submodedgeweight_{\submodelI,\submodelJ} = 0$ if the edge does not
exist), and where
$\bar \submodsetX = \submodgroundset \setminus \submodsetX$ is the
complement of set $\submodsetX$. Notice that we can
write the graph cut function as follows:
\begin{align}
\hspace{-2em}
\submodfun(\submodsetX)
&=
\sum_{\submodelI \in \submodsetX,
  \submodelJ \in \bar \submodsetX}
\submodedgeweight_{\submodelI,\submodelJ}
= \sum_{\submodelI,\submodelJ \in \submodgroundset}
\submodedgeweight_{\submodelI,\submodelJ} \mathbf 1\{
\submodelI \in \submodsetX,
  \submodelJ \in \bar \submodsetX \} \\
\submodlongonly{
&= \frac{1}{2}
\sum_{\submodelI,\submodelJ \in \submodgroundset}
\submodedgeweight_{\submodelI,\submodelJ}
\left[ \mathbf 1\{
\submodelI \in \submodsetX,
  \submodelJ \in \bar \submodsetX \}
+
\mathbf 1\{
\submodelJ \in \submodsetX,
  \submodelI \in \bar \submodsetX \}
\right] \label{eq:due_to_symmetry}
\\
&= 
\frac{1}{2}
\sum_{\submodelI,\submodelJ \in \submodgroundset} \submodedgeweight_{\submodelI,\submodelJ}
\Bigl(
  \min(|\submodsetX \cap \{\submodelI,\submodelJ\}|,1)
+   \min(|(\submodgroundset \setminus \submodsetX) \cap \{\submodelI,\submodelJ\}|,1)
- 1
     \Bigr) \\
  }
  &=
\frac{1}{2}
\sum_{\submodelI,\submodelJ \in \submodgroundset} \submodedgeweight_{\submodelI,\submodelJ}
\min(|\submodsetX \cap \{\submodelI,\submodelJ\}|,1)
+
\frac{1}{2}
\sum_{\submodelI,\submodelJ \in \submodgroundset} \submodedgeweight_{\submodelI,\submodelJ}
    \min(|(\submodgroundset \setminus \submodsetX) \cap \{\submodelI,\submodelJ\}|,1)
    -
\frac{1}{2}
    \sum_{\submodelI,\submodelJ \in \submodgroundset} \submodedgeweight_{\submodelI,\submodelJ} \\
  &= \tilde \submodfun(\submodsetX) + \tilde \submodfun(\submodgroundset \setminus \submodsetX)
    - \tilde \submodfun(\submodgroundset)
\end{align}
where
$\tilde \submodfun(\submodsetX) = \frac{1}{2}
\sum_{\submodelI,\submodelJ \in \submodgroundset}
\submodedgeweight_{\submodelI,\submodelJ} \min(|\submodsetX \cap
\{\submodelI,\submodelJ\}|,1)$\submodlongonly{
(which itself is a form of monotone submodular graph function
that, for a set of nodes $\submodsetX$,
returns the sum of edge weights for all
edges incident to at least one node in $\submodsetX$)
and where
Equation~\eqref{eq:due_to_symmetry} follows due to symmetry}.
Therefore, since $\min(\alpha,1)$ is concave, and since
$\submodmodfun_{i,j}(\submodsetX) = |\submodsetX \cap
\{\submodelI,\submodelJ\}|$ is modular,
$\tilde \submodfun(\submodsetX)$ is submodular for all
$\submodelI,\submodelJ$.  Also, since $\tilde \submodfun(\submodsetX)$
is submodular, so is
$\tilde \submodfun(\submodgroundset \setminus \submodsetX)$ (in
$\submodsetX$).  Therefore, the graph cut function can be expressed as
a sum of non-normalized feature-based functions. Note that here the
second modular function is not normalized and is non-increasing, and
also we subtract the constant $\tilde \submodfun(\submodgroundset)$ to
achieve equality.

\submodlongonly{
We can express the graph cut as a pure feature-based function as
described above, as well as allow directed graph cuts, as follows.
The indicator
$\mathbf 1\{ \submodelI \in \submodsetX, \submodelJ \in \bar
\submodsetX \}$ used above is an inherently directed asymmetric
quantity whose directness quality vanished due to the symmetry of
$\submodedgeweight$.  Suppose now it is allowed for
$\submodedgeweight_{\submodelI,\submodelJ} \neq
\submodedgeweight_{\submodelJ,\submodelI}$ and we interpret
$\submodedgeweight_{\submodelI,\submodelJ}$ as the weight of a
directed edge from node $\submodelI$ to node $\submodelJ$.  An
undirected edge is still expressible whenever
$\submodedgeweight_{\submodelI,\submodelJ} =
\submodedgeweight_{\submodelJ,\submodelI}$.  The general graph cut can then be
expressed as follows:
\begin{align}
\submodfun(\submodsetX)
&=
\sum_{\submodelI \in \submodsetX,
  \submodelJ \in \bar \submodsetX}
\submodedgeweight_{\submodelI,\submodelJ}
= \sum_{\submodelI,\submodelJ \in \submodgroundset}
\submodedgeweight_{\submodelI,\submodelJ} \mathbf 1\{
\submodelI \in \submodsetX,
                 \submodelJ \in \bar \submodsetX \} \\
  &= \sum_{\submodelI,\submodelJ \in \submodgroundset}
    \submodedgeweight_{\submodelI,\submodelJ}
    \Bigl[
    \min( 2 \times \mathbf 1\{ \submodelI \in \submodsetX \} + \mathbf 1\{
\submodelJ \in \submodsetX \},2) - | \submodsetX \cap \{ \submodelI,\submodelJ \} |
    \Bigr] \\
  &=
\sum_{\submodelI,\submodelJ \in \submodgroundset}
    \submodedgeweight_{\submodelI,\submodelJ}
    \min( \submodmodfun_{\submodelI,\submodelJ}(\submodsetX),2)
    -
\sum_{\submodelI,\submodelJ \in \submodgroundset}
    \submodedgeweight_{\submodelI,\submodelJ}
    | \submodsetX \cap \{ \submodelI,\submodelJ \} |
\end{align}
where $\submodmodfun_{\submodelI,\submodelJ}(\submodsetX) = \min( 2 \times \mathbf 1\{ \submodelI \in \submodsetX \} + \mathbf 1\{
\submodelJ \in \submodsetX \},2)$
is a non-negative monotone non-decreasing modular function for all $\submodelI,\submodelJ$
and the r.h.s.\ term is negative modular.
}

\submodlongonly{
We can slightly generalize the graph cut function to a parameterized
form in the following. Let $\submodgraphcuthyperparameter = 1$ for
now.  Then the graph cut can be written as:
\begin{align}
\submodfun(\submodsetX)
&=
\sum_{\submodelI \in \submodsetX,
  \submodelJ \in \bar \submodsetX}
                          \submodedgeweight_{\submodelI,\submodelJ} 
  = \sum_{\submodelI \in \submodsetX}
    \left[
    \sum_{\submodelJ \in \submodsetX} (1 - \submodgraphcuthyperparameter) \submodedgeweight_{\submodelI,\submodelJ}
    +  \sum_{\submodelJ \in \bar \submodsetX}
    \submodedgeweight_{\submodelI,\submodelJ}
                          \right] \\
  &=
\sum_{\submodelI \in \submodsetX,
    \submodelJ \in \submodgroundset} \submodedgeweight_{\submodelI,\submodelJ}
    - \submodgraphcuthyperparameter
    \sum_{\submodelI \in \submodsetX,
    \submodelI \in \submodsetX} \submodedgeweight_{\submodelI,\submodelJ}
    =
    \sum_{\submodelI \in \submodsetX}
    \submodedgeweight_{\submodelI}
    - \submodgraphcuthyperparameter
    \sum_{\submodelI \in \submodsetX,
    \submodelI \in \submodsetX} \submodedgeweight_{\submodelI,\submodelJ}
\end{align}
where $\submodedgeweight_{\submodelI}$ is the weight of node
$\submodelI$ (the sum of the weight of edges incident to vertex
$\submodelI$) and we note that
$\sum_{\submodelI \in \submodsetX, \submodelJ \in \submodsetX}
\submodedgeweight_{\submodelI,\submodelJ}$ is a classic supermodular
function in $\submodsetX$ (see
Section~\ref{sec:example-superm-funct}). Since this takes the form of
a modular minus a supermodular function, any value
$\submodgraphcuthyperparameter \geq 0$ will preserve submodularity, so
we consider this a form of generalized graph cut. This generalized form holds for
both undirected and directed graph cuts.  We leave as an exercise the
question of if this can be represented as a feature-based function or
not.
}

\submodlongonly{
One way of viewing the above, when wishing to maximize this function,
is to consider $\submodedgeweight_{\submodelI,\submodelJ}$ as a
similarity value between node $\submodelI$ and $\submodelJ$ in a
graph.  Then
$\sum_{\submodelI \in \submodsetX, \submodelJ \in \submodgroundset}
\submodedgeweight_{\submodelI,\submodelJ}$ measures overall how
similar the set $\submodsetX$ is to all of the nodes while
$\sum_{\submodelI \in \submodsetX, \submodelI \in \submodsetX}
\submodedgeweight_{\submodelI,\submodelJ}$ can be viewed as a
supermodular redundancy penalty, since it is the sum of pairwise
similarities between every node in $\submodsetX$.  We wish for sets
that are similar to all the nodes but that are not redundant.  Using a
negative supermodular function as a redundancy penalty is the approach
taken in~\cite{lin2010-submod-sum-nlp}.  If there is an item
$\submodelJ \in \submodgroundset$ that is very similar to many other
items, meaning $\submodedgeweight_{\submodelI,\submodelJ}$ is very
large for all $\submodelI$, even when all of the elements are quite
different from each other (which implies a strong transitivity
relationship would not exist in this case), then maximizing
$\submodfun(\submodsetX)$ might choose elements for $\submodsetX$ only
because they are similar to this one element $\submodelJ$ not because
they are similar to many other elements (which ideally is what we
would want). The supermodular redundancy penalty does not help in this
case since the items being mostly dissimilar to each other would still
be granted a small penalty.  To limit this from happening, we can
modify the first term so that it saturates, and produce a further generalized
graph cut function of the form:
\begin{align}
\submodfun(\submodsetX)
  &=
    \sum_{
  \submodelJ \in \submodgroundset
  }
  \min( \sum_{\submodelI \in \submodsetX}  \submodedgeweight_{\submodelI,\submodelJ},
  \alpha \sum_{\submodelI \in \submodgroundset}  \submodedgeweight_{\submodelI,\submodelJ})
  -
  \submodgraphcuthyperparameter
  \sum_{\submodelI \in \submodsetX,
  \submodelI \in \submodsetX} \submodedgeweight_{\submodelI,\submodelJ} \\
  &=
    \sum_{
  \submodelJ \in \submodgroundset
  }
  \min( C_{\submodelJ}(\submodsetX), \alpha C_{\submodelJ}(\submodgroundset))
  -
  \submodgraphcuthyperparameter
  \sum_{\submodelI \in \submodsetX,
  \submodelI \in \submodsetX} \submodedgeweight_{\submodelI,\submodelJ} 
\end{align}
where
$C_{\submodelJ}(\submodsetX) = \sum_{\submodelI \in \submodsetX}
\submodedgeweight_{\submodelI,\submodelJ}$ is a modular function for
all $\submodelJ$ and $0 \leq \alpha \leq 1$.  The minimum function
limits the influence that node $\submodelJ$ has on selecting elements
for $\submodsetX$ to no more than an $\alpha$ fraction of the total
possible affinity $\submodelJ$ has with the remaining nodes. Since
we've turned the modular function into a feature-based function,
submodularity is preserved. In fact, it is sometimes useful to use this
function even when $\lambda = 0$, but one should be reminded
again that this recovers the standard graph cut function
when $\lambda = 1$ and $\alpha = 1$.
}

Another way to view the graph cut function is to consider the
non-negative weights as a modular function defined over the edges.
That is, we view $\submodedgeweight \in \real^\submodedgesgroundset_+$
as a modular function
$\submodedgeweight : 2^\submodedgesgroundset \to \real_+$ where for
every $\submodedgessubset \subseteq \submodedgesgroundset$,
$\submodedgeweight(\submodedgessubset) = \sum_{\submodedgeel \in
  \submodedgessubset} \submodedgeweight(\submodedgeel)$ is the weight
of the edges $\submodedgessubset$ where
$\submodedgeweight(\submodedgeel)$ is the weight of edge
$\submodedgeel$. Then the graph cut function becomes
$\submodfun(\submodsetX) =
\submodedgeweight(\{(\submodedgeela,\submodedgeelb) \in
\submodedgesgroundset : \submodedgeela \in \submodsetX, \submodedgeelb
\in \submodsetX \setminus \submodsetX\})$. We view
$\{(\submodedgeela,\submodedgeelb) \in \submodedgesgroundset :
\submodedgeela \in \submodsetX, \submodedgeelb \in \submodsetX
\setminus \submodsetX\}$ as a set-to-set mapping function, that maps
subsets of nodes to subsets of edges, and the edge weight modular
function $\submodedgeweight$ measures the weight of the resulting
edges. This immediately suggests that other functions can measure the
weight of the resulting edges as well, including non-modular
functions. One example is to use a polymatroid function itself leading
$\submodaltaltfun(\submodsetX) = \submodaltfun(
\{(\submodedgeela,\submodedgeelb) \in \submodedgesgroundset :
\submodedgeela \in \submodsetX, \submodedgeelb \in \submodsetX
\setminus \submodsetX\})$ where
$\submodaltfun : 2^\submodedgesgroundset \to \real_+$ is a submodular
function defined on subsets of edges.  The function $\submodaltaltfun$
is known as the \keywordDef{cooperative cut} function, and it is neither
submodular nor supermodular in general but there are many useful and
practical algorithms that can be used to optimize
it~\cite{jegelka-coop-cut-journal-2016} thanks to its internal yet
exposed and thus available to exploit submodular structure.

While feature-based functions are flexible and powerful, there is a
strictly broader class of submodular functions, unable to be expressed
by feature-based functions, that are related to deep neural
networks. Here, we create a recursively nested composition of concave
functions with sums of compositions of concave functions. An example
is
$\submodfun(\submodaltsetX) = \submodconcavefun(
\sum_{\submodfeatureel \in \submodfeatureset}
\submodfuncweight_\submodfeatureel \submodconcavefun_\submodfeatureel
(\sum_{\submodaltelx \in \submodaltsetX}
\submodmodfun_{\submodfeatureel,\submodaltelx}))$, where
$\submodconcavefun$ is an outer concave function composed with a
feature-based function, with
$\submodmodfun_{\submodfeatureel,\submodaltelx} \geq 0$ and
$\submodfuncweight_\submodfeatureel \geq 0$.  This is known as a
two-layer \keywordDef{deep submodular function} (DSF). A three-layer DSF has the
form
$\submodfun(\submodaltsetX) = \submodconcavefun( \sum_{\submodaltelz
  \in \submodaltsetZ} \submodfuncweight_\submodaltelz
\submodconcavefun_\submodaltelz(\sum_{\submodfeatureel \in
  \submodfeatureset}
\submodfuncweight_{\submodfeatureel,\submodaltelz}
\submodconcavefun_\submodfeatureel (\sum_{\submodaltelx \in
  \submodaltsetX} \submodmodfun_{\submodfeatureel,\submodaltelx})))$.
DSFs strictly expand the class of submodular functions beyond feature-based functions, meaning that there are feature-based functions that
cannot~\citep{bilmes-dsf-arxiv-2017} represent deep submodular functions, even simple ones.
\submodlongonly{
For example, the function $f:2^V \to \real$
where $\submodgroundset = \{\submodELa,\submodELb,\submodELc,\submodELd,\submodELe,\submodELf\}$
cannot be represented by a
feature-based function.
\begin{align}
f(A) = \min\Bigl(
\min(|A \cap \{\submodELa,\submodELb,\submodELc,\submodELd\}|,3)
+
\min(|A \cap \{\submodELc,\submodELd,\submodELe,\submodELf\}|,3),
5
\Bigr)
\label{eq:non_matroid_dsf}
\end{align}
In fact, the space of representable submodular
functions strictly grows with each new layer added to a DSF.  DSFs are
analogous to deep neural networks, but where the weights may never be
negative, and the non-linear activation functions are monotone
non-decreasing concave functions, for example $\min(\alpha,1)$ or
$\sqrt{\alpha}$.  A full discussion is given in~\citep{bilmes-dsf-arxiv-2017}.
}

\submodlongonly{
\subsubsection{Example Supermodular Functions}
\label{sec:example-superm-funct}

There are also a number of simple functions that are supermodular. Firstly,
the feature-based functions mentioned above. That is,
a composition of a convex function with a modular function
$\submodaltfun(\submodaltsetX) =
\submodconvexfun(\submodmodfun(\submodaltsetX))$ since
$\sum_{\submodaltelx \in \submodaltsetX} \submodmodfun(\submodaltelx)
= \submodmodfun(\submodaltsetX)$.  We may take sums of such functions
as well without losing supermodularity, leading to
$\submodaltfun(\submodaltsetX) = \sum_{\submodfeatureel \in
  \submodfeatureset} \submodconvexfun_\submodfeatureel
(\sum_{\submodaltelx \in \submodaltsetX}
\submodmodfun_{\submodfeatureel}(\submodaltelx))$ where
$\submodconvexfun_\submodfeatureel$ can be a distinct convex
function for each $\submodfeatureel$ and
$\submodmodfun_{\submodfeatureel,\submodaltelx}$ is a non-negative
real value for all $\submodfeatureel$ and $\submodaltelx$.  This is
identical to sums of functions each of which is a convex composed
with a modular function
$\submodaltfun(\submodaltsetX) = \sum_{\submodfeatureel \in
  \submodfeatureset} \submodconvexfun_\submodfeatureel
(\submodmodfun_{\submodfeatureel}(\submodaltsetX))$ where
$\submodmodfun_{\submodfeatureel}$ is a $\submodfeatureel$-specific
non-negative modular function.

Another classic submodular function has been known as the
sum-sum-dispersion function. It is defined on a square matrix with
non-negative entries, specifically
$\submodaltfun(\submodaltsetX) = \sum_{\submodaltelx,\submodaltelx \in
  \submodaltsetX} \text{sim}(\submodaltelx,\submodaltelx')$ where
$\text{sim}(\submodaltelx,\submodaltelx')$ is any non-negative value
that expresses the relationship between $\submodaltelx$ and
$\submodaltelx'$.  If $\text{sim}(\submodaltelx,\submodaltelx')$ is an
affinity or similarity score, then $\submodaltfun(\submodaltsetX)$ being
large indicates the homogeneity of the set $\submodaltsetX$ since
everything with $\submodaltsetX$ is mutually pairwise similar.
}

\subsection{Submodular Optimization}
\label{sec:subm-optim}

\submodlongonly{As mentioned earlier, machine learning involves (often continuous)
mathematical optimization.} Submodular functions, while discrete, would
not be very useful if it was not possible to optimize over them
efficiently. There are many natural problems in machine
learning that can be cast as submodular optimization and that can be
addressed relatively efficiently.

When one wishes to encourage diversity, information, spread, high
complexity, independence, coverage, or dispersion, one usually will
maximize a submodular function, in the form of
$\max_{\submodaltsetX \in \submodcnstr} \submodfun(\submodaltsetX)$
where $\submodcnstr \subseteq 2^\submodgroundset$ is a constraint set,
a set of subsets we are willing to accept as feasible solutions (more
on this below).  \submodlongonly{For example, $\submodfun$ might correspond to the
value of a set of sensor locations in an environment, and we wish to
find the best set $\submodsetX \subseteq \submodgroundset$ of sensors
locations given a fixed upper limit on the number of sensors
$\submodcardcnstr$. Alternatively, if $\submodgroundset$ is a set of
documents, one might wish to choose a subset
$\submodsetX \subset \submodgroundset$ of documents that act as a good
representative of a complete set of documents --- this is a
\emph{document summarization} task. As a third example, if
$\submodgroundset$ is a large training dataset over which it is costly
to train a deep neural network, one might wish to find a
representative subset $\submodsetX$ of size $\submodcardcnstr$ to
train in as surrogate for $\submodgroundset$, one that works much
better than a uniformly-at-random subset of size
$\submodcardcnstr$. This is a coreset task.  Submodularity is a good
general model for such problems because submodular functions act as a
form of information function (Section~\ref{sec:comb-inform-funct}).
Maximizing a submodular function subject to a constraint selects a
subset that is non-redundant and thus not wasteful.  
\submodlongonly{
If we can find a
small set that retains most or all of the valuation of the whole
(which maximization strives to achieve), we have not lost information
as measured by the submodular function, and if the subset is small, it
becomes an efficient representation of the information contained in
the whole.
}}

Why is submodularity, in general, a good model for diversity?
Submodular functions are such that once you have some elements, any
other elements
not in your possession
but that are similar to,
explained by, or represented by the elements in your possession
become
less valuable. Thus, in order to maximize the function, one must
choose other elements that are dissimilar to, or not well represented
by, the ones you already have. That is, the elements similar to the
ones you own are diminished in value relative to their original
values, while the elements dissimilar to the ones you have do not have
diminished value relative to their original values. Thus, maximizing a
submodular function successfully involves choosing elements that are
jointly dissimilar amongst each other, which is a definition of
diversity. Diversity in general is a critically important aspect in
machine learning and artificial intelligence. For example, bias in
data science and machine learning can often be seen as some lack of
diversity somewhere. Submodular functions have the potential to
encourage (and even ensure) diversity, enhance balance, and reduce
bias in artificial intelligence.

Note that in order for a submodular function to appropriately model
diversity, it is important for it to be instantiated
appropriately. Figure~\ref{fig:fl_vs_random} shows an example in two
dimensions. The plot compares the ten points chosen according to a
facility location instantiated with a Gaussian kernel, along with the
random samples of size ten. We see that the facility location selected
points are more diverse and tend to cover the space much better than
any of the randomly selected points, each of which miss large regions
of the space and/or show cases where points near each other are
jointly selected.

\begin{figure}
\centering
\includegraphics[page=1,width=0.95\textwidth]{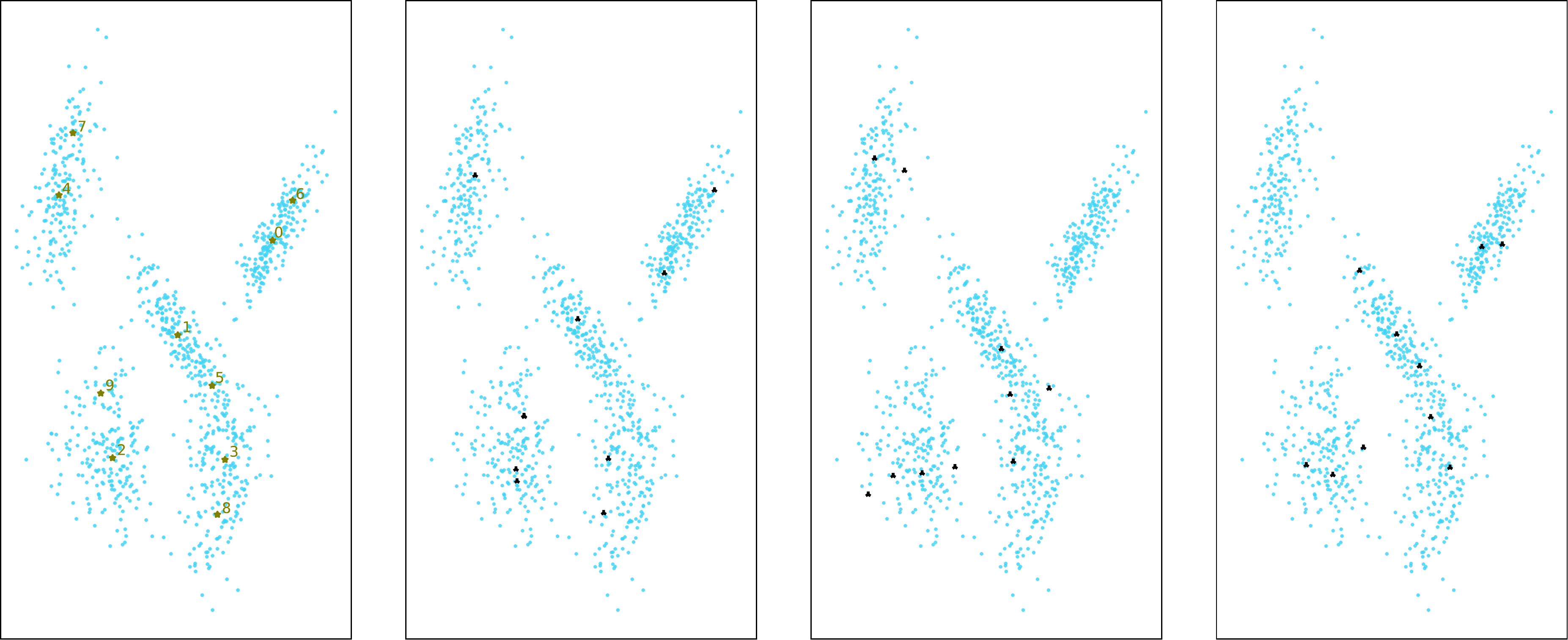}
\caption{
  Far Left: cardinality constrained (to ten) submodular maximization of a facility location function over 1000 points
  in two dimensions. Similarities are based on a Gaussian kernel 
  $\text{sim}(\submodaltelx,\submodel) = \exp(-d(\submodaltelx,\submodel))$ where $d(\cdot,\cdot)$ is
  a distance. Selected points are green stars, and the greedy order is also shown next to each selected
  point. Right three plots: different uniformly-at-random subsets of size ten.
}
\label{fig:fl_vs_random}
\end{figure}

When one wishes for homogeneity, conformity, low complexity,
coherence, or cooperation, one will usually minimize a submodular
function, in the form of
$\min_{\submodaltsetX \in \submodcnstr}
\submodfun(\submodaltsetX)$. For example, if $\submodgroundset$ is a
set of pixels in an image, one might wish to choose a subset of pixels
corresponding to a particular object over which the properties (i.e.,
color, luminance, texture) are relatively homogeneous. Finding a set
$\submodsetX$ of size $\submodcardcnstr$, even if $\submodcardcnstr$
is large, need not have a large valuation $\submodfun(\submodsetX)$,
in fact it could even have the least valuation.  Thus, semantic image
segmentation could work even if the object being segmented and
isolated consists of the majority of image pixels. 

\submodlongonly{
As another example, clustering is a problem that often asks for
subsets of data that are homogeneous in some way --- one may wish to
cluster a dataset $\submodgroundset$ by removing, one by one, subsets
$\submodsetX$ from $\submodgroundset$ that are homogeneous, and this
would mean finding a subset with minimum valuation
$\submodfun(\submodsetX)$. As a third example, suppose we are
interested in the most probable assignment to a random vector (i.e.,
MAP probabilistic inference) in the probability model
$p(\submodsetX) = \frac{1}{Z} \exp( - \submodfun ( \submodsetX ))$.
Here the function $\submodfun$ is the energy function, $Z$ is the
``partition'' function, and $p(\submodsetX)$ is said to be a
\emph{log-supermodular} probabilistic model. Log-supermodular models
give high probability to configurations of the random variables that
are homogeneous or ``regular'' as measured by $\submodfun$ (since $\submodfun$ gives
such configurations a low value). Finding the MAP solution is the same
as finding the minimum energy configuration, and if $\submodfun$ is
submodular, this is submodular minimization. Probabilistic modeling
with submodular/supermodular functions is further discussed in
Section~\ref{sec:prob-model}.
}

\submodlongonly{
Arbitrary set functions are impossible to tractably optimize while
having any form of mathematical quality assurance. In fact, it is
easy, given a set $\submodsetX^*$ with
$|\submodsetX^*| = \submodcardcnstr$, to define a function
$\submodfun$ such that $\submodfun(\submodsetX) = \alpha$ for all
$\submodsetX \neq \submodsetX^*$ and
$\submodfun(\submodsetX^*) = \beta$. With the right values of
$\alpha < \beta$, we see that, when trying to maximize this function,
anything other than a full exponential (in $\submodgroundsetsize$)
search will give an unboundedly large error in the worst case,
regardless of the outcome of the $\text{P} = \text{NP}$
question. Fortunately, this function is not submodular since given
sets $\submodsetZ_1,\submodsetZ_2$ with
$\submodsetZ_1 \neq \submodsetZ_2$ and
$\submodsetZ_1 \cup \submodsetZ_2 = \submodsetX^*$, we have
$\submodfun(\submodsetZ_1) + \submodfun(\submodsetZ_2) = \alpha +
\alpha < \beta + \alpha = \submodfun(\submodsetZ_1 \cup \submodsetZ_2)
+ \submodfun(\submodsetZ_1 \cap \submodsetZ_2)$. Neither is the
function supermodular since given some $\submodsetZ$ with
$|\submodsetZ \setminus \submodsetX^*| > 0$ and
$|\submodsetX^* \setminus \submodsetZ| > 0$, we have
$\submodfun(\submodsetX^*) + \submodfun(\submodsetZ) = \beta + \alpha
> 2\alpha = \submodfun(\submodsetX^* \cup \submodsetZ) +
\submodfun(\submodsetX^* \cap \submodsetZ)$. Therefore, we might
expect there is some hope when we require the function to be
submodular.
}

\submodlongonly{
Submodular optimization problems are more interesting when we impose
constraints. For example, when unconstrainedly maximizing a
polymatroid function, a trivial solution is $\submodgroundset$ so the
problem becomes interesting only when there is a constraint. 
The simplest and most widely used constraint is the cardinality constraint,
giving the following optimization problem:
$\max_{\submodsetX \subseteq \submodgroundset : |\submodsetX| \leq
  \submodcardcnstr } \submodfun(\submodsetX)$.  When $\submodfun$
represents the information in the set $\submodsetX$, this optimization
procedure asks for the size-limited set $\submodsetX$ of size
$\submodcardcnstr$ that has the most information about
$\submodgroundset$. One way to see this is to note that
$\submodfun(\submodgroundset) = \submodfun(\submodgroundset \setminus
\submodsetX | \submodsetX ) + \submodfun( \submodsetX )$.  Thus, when
we maximize $\submodfun( \submodsetX )$, we are minimizing the
residual information
$\submodfun(\submodgroundset \setminus \submodsetX | \submodsetX )$
about $\submodgroundset$ not contained in $\submodsetX$.
}

\subsubsection{Submodular Maximization}
\label{sec:submodularMaximization}

While the cardinality constrained submodular maximization problem is
NP complete~\cite{feige1998threshold}, it was shown in
\cite{nemhauser1978,fisher1978analysis} that the very simple and
efficient greedy algorithm finds an approximate solution guaranteed to
be within $1-1/e \approx 0.63$ of the optimal solution.
Moreover, the approximation ratio achieved by
the simple greedy algorithm is provably the best achievable in
polynomial time, assuming $P \neq NP$~\cite{feige1998threshold}.
The greedy algorithm proceeds as follows: Starting with
$\submodsetX_0 = \emptyset$, we repeat the following greedy step for
$i = 0 \dots (\submodcardcnstr-1)$:
\begin{equation}
  \submodsetX_{i+1} = \submodsetX_i \cup
  (\argmax_{\submodel \in \submodgroundset \setminus \submodsetX_i }
  \submodfun(\submodsetX_i \cup \{ \submodel \} ) )
    \label{eqn:greedy_algorithm}
\end{equation}
\submodlongonly{If there are any ties in the $\argmax$, we break them arbitrarily.}
What the above approximation result means is that if
$\submodsetX^* \in \argmax \{ \submodfun(\submodsetX) : |\submodsetX|
\leq \submodcardcnstr \}$, and if $\tilde \submodsetX$ is the result
of the greedy procedure, then
$\submodfun(\tilde \submodsetX) \geq
(1-1/e)\submodfun(\submodsetX^*)$.

\submodlongonly{
It is very easy to show where the famous $1-1/e$ approximation ratio
comes from, and we show this in the following in as simple and direct
a way as possible.  From series analysis, we can define
$e^x = \lim_{\submodcardcnstr \to \infty}
(1+x/\submodcardcnstr)^\submodcardcnstr$ and since
$(1-1/\submodcardcnstr)^\submodcardcnstr$ is increasing with
$\submodcardcnstr$ having a limit of $e^{-1}$, $1 - (1-1/\submodcardcnstr)^\submodcardcnstr$ is
decreasing with $\submodcardcnstr$ and has limit of $1-1/e$.  Then,
note that the greedy procedure creates a sequence, or a chain, of
elements at step $i$, and we define
$\submodsetX_i = \{ \submodel_1, \submodel_2, \dots, \submodel_i\}$
where $\submodel_i$ is the $i^\text{th}$ element added by greedy and
$\submodsetX_\submodcardcnstr$ is the final solution.  This
forms a chain since
$\emptyset = \submodsetX_0 \subset \submodsetX_1 \subset \submodsetX_2
\subset \dots \subset \submodsetX_i$. Also, with $\submodsetX^*$ being
an optimal set of size $\submodcardcnstr$ we have
\begin{align}
  \submodfun(\submodsetX^*)
  &\leq \submodfun(\submodsetX^* \cup \submodsetX_i) =
\submodfun(\submodsetX^* | \submodsetX_i) + \submodfun(\submodsetX_i)  \\
  &\leq \sum_{\submodel \in \submodsetX^*}
    \submodfun(\submodel | \submodsetX_i) + \submodfun(\submodsetX_i) \label{eq:kgreed_from_submod} \\
  &\leq \sum_{\submodel \in \submodsetX^*}
    \submodfun(\submodel_{i+1} | \submodsetX_i) + \submodfun(\submodsetX_i) \label{eq:kgreed_from_greedy} \\
  &= \submodcardcnstr \submodfun(\submodel_{i+1} | \submodsetX_i) + \submodfun(\submodsetX_i) \label{eq:kgreed_from_size},
\end{align}
where Equation~\eqref{eq:kgreed_from_submod} follows
from submodularity, Equation~\eqref{eq:kgreed_from_greedy} follows since greedy chooses the optimal
single element at each step, and Equation~\eqref{eq:kgreed_from_size} follows
since the given optimal solution is one of size $\submodcardcnstr$.
Then
\begin{align}
  \submodfun(\submodsetX^*) - \submodfun(\submodsetX_i)
  &\leq \submodcardcnstr \submodfun(\submodel_{i+1} | \submodsetX_i) \\
  &= \submodcardcnstr \left[ \submodfun(\submodsetX_{i+1}) - \submodfun(\submodsetX_i) \right] \\
  &= \submodcardcnstr \left[ \submodfun(\submodsetX^*) - \submodfun(\submodsetX_i) - \left[ \submodfun(\submodsetX^*) - \submodfun(\submodsetX_{i+1}) \right] \right]
\end{align}
which immediately yields
\begin{align}
  \submodfun(\submodsetX^*) - \submodfun(\submodsetX_{i+1})
  \leq (1-1/\submodcardcnstr)\left[ \submodfun(\submodsetX^*) - \submodfun(\submodsetX_i) \right]
\end{align}
When we note that $\submodfun(\submodsetX_{0}) = 0$,
apply this
$\submodcardcnstr$ times,
and use the above bound on $1-1/e$,
we get
\begin{align}
  \submodfun(\submodsetX^*)(1-1/e) \leq
  \submodfun(\submodsetX^*)(1 - (1-1/\submodcardcnstr)^\submodcardcnstr) \leq \submodfun(\submodsetX_k).
\end{align}
This means that when $k$ is one, the greedy procedure has an
approximation ratio of 1, which is clear since greedy selects the best
single element. As $k$ increases, the approximation ratio gets worse,
but it is never worse than the lower bound of $1-1/e$, but it does
mean that the smaller the $k$ the better the approximation ratio.
}

The $1-1/e$ guarantee is a powerful constant factor approximation
result since it holds regardless of the size of the initial set
$\submodgroundset$ and regardless of which polymatroid function
$\submodfun$ is being optimized.  It is possible to make this
algorithm run extremely fast using various acceleration
tricks~\cite{fisher1978analysis,
  nemhauser1978,minoux1978_accelerated_greedy}.

\submodlongonly{
A dual variant for which the greedy algorithm immediately applies is
the \emph{submodular set cover} problem, where one wishes for the
smallest set of a given valuation rather than the highest
valuation set of a given size. Here
we wish to find
\begin{align}
\submodsetX^* \in \argmin_{\submodsetX \subseteq \submodgroundset}
  |\submodsetX| \text{ such that } \submodfun(\submodsetX) \geq \alpha
\end{align}
where $\alpha$ is a minimum valuation threshold. To approximately
solve this problem, the greedy algorithm above runs until
$\submodfun(\submodsetX_{i}) \geq \alpha$ and $\submodsetX_{i}$
becomes the solution. Observe that we then get a solution that
satisfies the given constraint (since that is what determines the
stopping point). The approximation is in terms of how much bigger the
resulting set is than $|\submodsetX^*|$. For a general real-valued
polymatroid function, the size of the greedy solution
$\tilde \submodsetX = \submodsetX_{i}$ that first exceeds the
threshold satisfies the following upper bound:
$|\tilde \submodsetX| \leq |\submodsetX^*|\left(1 + \ln \frac{
    \submodfun(\submodgroundset) }{ \submodfun(\submodgroundset) -
    \submodfun(\submodsetX_{i-1})}\right)$ \citep{wolsey1982analysis}.
It is interesting to note that these approximations are the best
approximations we can do for these problems in polynomial time,
assuming $\text{P} \neq \text{NP}$.
}

A minor bit of additional information about a polymatroid function,
however, can improve the approximation guarantee. Define the total
curvature if the polymatroid function $\submodfun$ as
$\submodcurvature = 1 - \min_{\submodel \in \submodgroundset}
\submodfun(\submodel | \submodgroundset - \submodel)/
\submodfun(\submodel)$ where we assume $\submodfun(\submodel) > 0$ for
all $\submodel$ (if not, we may prune them from the ground set since
such elements can never improve a polymatroid function valuation).  We
thus have $0 \leq \submodcurvature \leq 1$,
and~\cite{conforti1984submodular} showed that the greedy algorithm
gives a guarantee of
$\frac{1}{\submodcurvature}(1 - e^{-\submodcurvature}) \geq 1-1/e$. In
fact, this is an equality (and we get the same bound) when
$\submodcurvature = 1$, which is the fully curved case. As
$\submodcurvature$ gets smaller, the bound improves, until we reach
the $\submodcurvature = 0$ case and the bound becomes unity. Observe
that $\submodcurvature = 0$ if and only if the function is modular, in
which case the greedy algorithm is optimal for the cardinality
constrained maximization problem. \submodlongonly{Extending this even further, in many
cases, the function need not even be submodular for the greedy
algorithm to have a guarantee. Given a non-negative but otherwise
arbitrary set function $\submodaltaltfun$, the submodularity
ratio~\cite{das2018approximate} can be defined as a measure of
deviation from submodularity as follows:
\begin{align}
  \gamma(\submodaltaltfun) =
  \min_{\submodsetX \subseteq \submodgroundset, \submodsetY \subseteq \submodgroundset : \submodsetY\cap \submodsetX=\emptyset}
  \frac
  {\sum_{\submodel\in \submodsetY}\submodaltaltfun(\submodel|\submodsetX)}
  {\submodaltaltfun(\submodsetY|\submodsetX)},
\end{align}
where $|\submodgroundset| = \submodgroundsetsize$.  $\submodaltaltfun$
is submodular if and only $\gamma(\submodaltaltfun) = 1$, and
otherwise $\gamma(\submodaltaltfun) < 1$ measures how far, in a sense,
$\submodaltaltfun$ is from being submodular. Moreover, a variety of
non-submodular functions can be optimized approximately with
guarantees that are a function of
$\gamma(\submodaltaltfun)$. Unfortunately, $\gamma(\submodaltaltfun)$
is in general information theoretically hard to compute.}
In some
cases, non-submodular functions can be decomposed into components that
each might be more amenable to approximation. We see below that any
set function can be written as a difference of
submodular~\cite{narasimhan2005-subsup,iyer2012-min-diff-sub-funcs}
functions, and sometimes (but not always) a given $\submodaltaltfun$
can be composed into a monotone submodular plus a monotone
supermodular function, or a BP function~\cite{bai-bp-functions-2018},
i.e., $\submodaltaltfun = \submodfun + \submodaltfun$ where
$\submodfun$ is submodular and $\submodaltfun$ is
supermodular. $\submodaltfun$ has an easily computed quantity called
the supermodular curvature
$\submodcurvature^\submodaltfun = 1 - \min_{\submodel \in
  \submodgroundset} \submodaltfun(\submodel)/\submodaltfun(\submodel |
\submodgroundset - \submodel)$ that, together with the submodular
curvature, can be used to produce an approximation ratio having the
form
$\frac{1}{\submodcurvature}(1 -
e^{-\submodcurvature(1-\submodcurvature^\submodaltfun)})$ for greedily
maximization of $\submodaltaltfun$.

\submodlongonly{
In sum, while the greedy algorithm for cardinality constrained
maximization might seem overly simple, it is actually quite powerful
and computationally unimprovable in the case of polymatroid functions.
}

\subsubsection{Discrete Constraints}

There are many other types of constraints one might desire besides a
cardinality limitation. The next simplest constraint allows each
element $\submodel$ to have a non-negative cost, say
$\submodmodfun(v) \in \real_+$. In fact, this means that the costs are
modular, i.e., the cost of any set $\submodsetX$ is
$\submodmodfun(\submodsetX) = \sum_{\submodel \in \submodsetX}
\submodmodfun(\submodel)$. A submodular maximization problem subject
to a \emph{knapsack constraint} then takes the form
$\max_{\submodsetX \subseteq \submodgroundset :
  \submodmodfun(\submodsetX) \leq \submodknapnstr }
\submodfun(\submodsetX)$ where $\submodknapnstr$ is a non-negative
budget. While the greedy algorithm does not solve this problem
directly, a slightly modified cost-scaled version of the greedy
algorithm \citep{sviridenko2004note} does solve this problem for any
set of knapsack costs. This has been used for various multi-document
summarization tasks
\citep{lin2011-class-submod-sum,hui2012-submodular-shells-summarization}.

\submodlongonly{When performing convex minimization, a natural constraint to use is a
convex subset of the reals.} There is no single direct analogy for a
convex set when one is optimizing over subsets of the set
$\submodgroundset$, but there are a few forms of discrete constraints that are
both mathematically interesting and that often occur repeatedly in
applications.

The first form is the independent subsets of a matroid. \submodlongonly{Firstly,
matroids is an area in discrete mathematics that, to do them justice,
would require an entire book~\citep{oxley2011matroid}.  Secondly, to
really understand submodularity, one must also master matroids, and
this can take some time. For the purposes of this section, we give
only the briefest of introductions to matroids, enough to see how they
can be useful for constraints in submodular maximization problems. A
matroid is an algebraic set system that consists of a set
$\submodgroundset$ and a set of subsets
$\submodmatroidsetofsets =
(\submodmatroidindsetI_1,\submodmatroidindsetI_2,\dots)$ known as the
independent sets of a matroid, where $\submodmatroidindsetI_i \subseteq \submodgroundset$
for all $i$.
A matroid can fully defined this way
and is often indicated by the pair
$(\submodgroundset,\submodmatroidsetofsets)$. To be a matroid, the set
of sets must have certain properties, namely: (1)
$\emptyset \in \submodmatroidsetofsets$; (2) If
$\submodsetY \in \submodmatroidsetofsets$ and
$\submodsetX \subset \submodsetY$, then
$\submodsetX \in \submodmatroidsetofsets$; and (3) If
$\submodsetX,\submodsetY \in \submodmatroidsetofsets$ with
$|\submodsetX| < |\submodsetY|$, then
$\exists \submodel \in \submodsetY \setminus \submodsetX$ such that
$\submodsetY + \submodel \in \submodmatroidsetofsets$.  The first
property says that the set of sets is not empty and must contain the
empty set at the very least. The second property says that the set of
sets are closed under subsets. The third property is a form of
exchangeability, meaning that if we have two independent sets, one
larger than the other, there must be an element in the larger
independent set that can be added to the smaller independent set
without loss of independence.  You would not be wrong if the set of
independent sets reminds you of sets of independent vectors in a
vector space --- in fact, matroids are a purely algebraic generalization of
linear independence properties of sets of vectors.  Importantly, it is
a strict generalization, meaning that there are matroids that are not
representable as independent vectors in a vector space over any field
\citep{oxley2011matroid}.

}The independent sets of a matroid are useful to represent a constraint
set for submodular maximization
\cite{calinescu2007maximizing,lee2009submodular,lee2010maximizing},
$\max_{\submodsetX \in \submodmatroidsetofsets }
\submodfun(\submodsetX)$, and this can be useful in many ways. We can
see this by showing a simple example of what is known as a
\emph{partition matroid}. Consider a partition
$\submodgroundset = \{ \submodgroundset_1, \submodgroundset_2, \dots,
\submodgroundset_\submodpartmatroidnumblocks\}$ of $\submodgroundset$
into $\submodpartmatroidnumblocks$ mutually disjoint subsets that we
call blocks. Suppose also that for each of the
$\submodpartmatroidnumblocks$ blocks, there is a positive integer
limit $\submodpartmatroidlimit_i$ for
$i \in [\submodpartmatroidnumblocks]$.  Consider next the set of sets
formed by taking all subsets of $\submodgroundset$ such that each
subset has intersection with $\submodgroundset_i$ no more than
$\submodpartmatroidlimit_i$ for each $i$. I.e., consider
\begin{align}
 \submodmatroidsetofsets_\text{p} = \{ \submodsetX : \forall i \in [\submodpartmatroidnumblocks], |\submodgroundset_i \cap \submodsetX | \leq \submodpartmatroidlimit_i \}.
\end{align}
Then $(\submodgroundset,\submodmatroidsetofsets_\text{p})$ is a
matroid. The corresponding submodular maximization problem is a
natural generalization of the cardinality constraint in that, rather
than having a fixed number of elements beyond which we are
uninterested, the set of elements $\submodgroundset$ is organized
into groups, and here we have a fixed per-group limit beyond which we
are uninterested. This is useful for fairness applications since
the solution must be distributed over the blocks of the matroid.
Still, there are many much more powerful types of
matroids that one can use~\citep{oxley2011matroid,gordon2012matroids}.

Regardless of the matroid, the problem
$\max_{\submodsetX \in \submodmatroidsetofsets }
\submodfun(\submodsetX)$ can be solved, with a 1/2 approximation
factor, using the same greedy algorithm as above \citep{nemhauser1978,
  fisher1978analysis}. Indeed, the greedy algorithm has an intimate
relationship with submodularity, a fact that is well studied in some
of the seminal works on submodularity
\citep{edmonds1970matroids,lovasz1983submodular, schrijver2004}.  It
is also possible to define constraints consisting of an
\emph{intersection of matroids}, meaning that the solution must be
simultaneously independent in multiple distinct matroids.  Adding on
to this, we might wish a set to be independent in multiple matroids
and also satisfy a knapsack constraint. 
Knapsack
constraints are not matroid constraints, since there can be multiple
maximal cost solutions that are not the same size (as must be the
case in a matroid).  It is also possible to define discrete
constraints using level sets of another completely different
submodular function \citep{rishabh2013-submodular-constraints} ---
given two submodular functions $\submodfun$ and $\submodaltfun$, this
leads to optimization problems of the form
$\max_{\submodsetX \subseteq \submodgroundset :
  \submodaltfun(\submodsetX) \leq \submodscscscskthreshold }
\submodfun(\submodsetX)$ (the submodular cost submodular knapsack, or
SCSK, problem) and
$\min_{\submodsetX \subseteq \submodgroundset :
  \submodaltfun(\submodsetX) \geq \submodscscscskthreshold }
\submodfun(\submodsetX)$ (the submodular cost submodular cover, or
SCSC, problem).  Other examples include covering constraints
\cite{iwata2009submodular}, and cut constraints
\cite{jegelka-coop-cut-journal-2016}.  Indeed, the type of constraints on
submodular maximization for which good and scalable algorithms exist
is quite vast, and still growing.

One last note on submodular maximization. In the above, the function
$\submodfun$ has been assumed to be a polymatroid function. There are
many submodular functions that are not
monotone~\citep{feldman2012optimal}.  One example we saw before,
namely the graph cut function.  Another example is the log of the
determinant (log-determinant) of a submatrix of a positive-definite
matrix (which is the Gaussian entropy plus a constant). Suppose that
$\submodDPPmatrix$ is an
$\submodgroundsetsize \times \submodgroundsetsize$ symmetric
positive-definite (SPD) matrix, and that
$\submodDPPmatrix_\submodsetX$ is a row-column submatrix (i.e., it is
an $|\submodsetX| \times |\submodsetX|$ matrix consisting of the rows
and columns of $\submodDPPmatrix$ consisting of the elements in
$\submodsetX$). Then the function defined as
$\submodfun(\submodsetX) = \log \det ( \submodDPPmatrix_\submodsetX )$
is submodular but not necessarily monotone non-decreasing.  In fact,
the submodularity of the log-determinant function is one of the
reasons that \emph{determinantal point processes} (DPPs), which
instantiate probability distributions over sets in such a way that
high probability is given to those subsets that are diverse according
to $\submodDPPmatrix$, are useful for certain tasks where we wish to
probabilistically model diversity~\citep{kulesza2011k}.  Diversity of
a set $\submodsetX$ here is measured by the volume of the
parallelepiped which is known to be computed as the determinant of the
submatrix $\submodDPPmatrix_\submodsetX$ and taking the log of this
volume makes the function submodular in $\submodsetX$.  A DPP in fact
is an example of a log-submodular probabilistic model (more in
Section~\ref{sec:prob-model}). \submodlongonly{Also, a relatively
  simple randomized greedy algorithm~\citep{feldman2012optimal} can
  yield a 1/2 approximation algorithm for maximizing any non-monotone
  submodular function, a class that includes not only DPPs but also
  graph cuts and differences between a submodular and a modular
  function. Interestingly, the log-determinant can be seen as a
  function of spectral decompositions of principle submatrices of a
  given SPD matrix, but it is not the only one that is
  submodular~\cite{friedland2011submodular}.
}

\subsubsection{Submodular Function Minimization}
\label{sec:subm-minim}

\submodlongonly{
So far, we have spoken mostly about submodular maximization, but
submodular minimization is also a fascinating study, perhaps even more
so than submodular maximization. We discussed, for example, the MAP
inference problem (amongst others) above.
}

In the case of a polymatroid function, unconstrained minimization is
again trivial.  However, even in the unconstrained case, the
minimization of an arbitrary (i.e., not necessarily monotone)
submodular function
$\min_{\submodsetX \subseteq \submodgroundset}
\submodfun(\submodsetX)$ might seem hopelessly intractable.
Unconstrained submodular maximization is NP-hard (albeit approximable),
and this is not surprising given that there are an exponential number
of sets needing to be considered. Remarkably, submodular minimization
does not require exponential computation, is not NP-hard, and in fact,
there are polynomial time algorithms for doing so, something that is
not at all obvious. This is one of the important characteristics that
submodular functions share with convex functions, their common
amenability to minimization. Starting in the very late 1960s and
spearheaded by individuals such as Jack
Edmonds~\cite{edmonds1970matroids}, there was a concerted effort in
the discrete mathematics community in search of either an algorithm
that could minimize a submodular function in polynomial time or a
proof that such a problem was NP-hard.  The nut was finally cracked in
a classic paper~\citep{grotschel1981ellipsoid} on the ellipsoid
algorithm that gave a polynomial time algorithm for submodular
function minimization (SFM). While the algorithm was polynomial, it
was a continuous algorithm, and it was not practical, so the search
continued for a purely combinatorial strongly polynomial time
algorithm.  Queyranne~\cite{queyranne98} then proved that an
algorithm~\cite{nagamochi1992computing} worked for this problem when the set
function also satisfies a symmetry condition (i.e.,
$\forall \submodsetX \subseteq \submodgroundset,
\submodfun(\submodsetX) = \submodfun(\submodgroundset \setminus
\submodsetX)$), which only requires $O(n^3)$ time. The result
finally came around the year
2000 using two mostly independent
methods~\cite{iwata00,schrijver2000combinatorial}. These algorithms,
however, also were impractical in that while they are polynomial time,
they have unrealistically high polynomial degree (i.e.,
$\tilde O(\submodgroundset^7*\gamma + \submodgroundset^8)$ for
\cite{schrijver2000combinatorial} and
$\tilde O(\submodgroundset^7*\gamma)$ for~\cite{iwata00}).  This led
to additional work on combinatorial algorithms for SFM leading to
algorithms that could perform SFM in time
$\tilde O(\submodgroundset^5\gamma + \submodgroundset^6)$
in~\cite{iwata2009simple}.  Two practical algorithms for SFM include
the Fujishige-Wolfe
procedure~\citep{fujishige2005submodular,wo76}\footnote{This is the
  same Wolfe as the Wolfe in Frank-Wolfe but not the same algorithm.}
as well as the Frank-Wolfe procedure, each of which minimize the
2-norm on a polyhedron $\submodbasepoly_\submodfun$ associated with
the submodular function $\submodfun$ and which is defined below (it
should also be noted that the Frank-Wolfe algorithm can also be used
to minimize the convex extension of the function, something that is
relatively easy to compute via the \lovasz{} extension~\cite{lovasz1983submodular}\submodlongonly{, see below}).
More recent work on SFM are also based continuous relaxations of the
problem in some form or another, leading algorithms with strongly
polynomial running time~\cite{lee2015faster} of
$O(\submodgroundsetsize^3 \log^2 n)$ for which it was possible to drop
the log factors leading to a complexity of $O(\submodgroundsetsize^3)$
in~\cite{jiang2021minimizing}, weakly-polynomial running
time~\cite{lee2015faster} of $\tilde O(\submodgroundsetsize^2 \log M)$
(where $M >= \max_{S \subseteq V} |\submodfun(S)|$), pseudo-polynomial
running time~\cite{axelrod2020near,chakrabarty2017subquadratic} of
$\tilde O(\submodgroundsetsize M^2)$, and a $\epsilon$-approximate
minimization with a linear running time~\cite{axelrod2020near} of
$\tilde O (\submodgroundsetsize/\epsilon^2)$.  There have been other
efforts to utilize parallelism to further improve SFM
\cite{balkanski2020lower}.

\submodlongonly{
What is remarkable about all of the above algorithms is that while
submodular functions live in a cone in
$2^\submodgroundsetsize$-dimensional space, this cone has a non-empty
(measurable) interior. This means that it is not possible to map this
cone of points down to a lower-dimensional space, or equivalently,
that submodular functions in general have $2^\submodgroundsetsize$
independent degrees of freedom (at least locally). This can be seen by
considering strictly submodular functions, ones for which the
inequality is never an equality, even such simple functions like
$\sqrt{|\submodsetX|}$.  Even in such rich a space, submodular
functions can be minimized with so few oracle queries to
$\submodfun$.  The above polytime
(and in particular the $\tilde O (\submodgroundsetsize/\epsilon^2)$)
results are particular interesting. As
$\submodgroundsetsize \to \infty$ the ratio of points that are queried
to minimize a submodular function, and independent points the
submodular function contains, goes to zero.  This is quite analogous
to convex functions.

Practically speaking, with so many algorithms capable of minimizing a
submodular function, it may not be obvious which one to choose. In
general, the practical utility of the various algorithms, even the
ones with lower asymptotic complexity, have not been compared to the
Fujishige-Wolfe and Frank-Wolfe procedures (which are the ones used
most often in practice for general purpose submodular function
minimization). If the submodular function is a special case, however
(e.g., the graph cut function) then it is in general better to use an
algorithm specifically geared towards exploiting the additional extant
structure. Also, sometimes the minimum norm can be slow, and even
though SFM is polytime, it is sometimes useful to approximately
minimize a submodular function~\cite{jegelka2011-fast-approx-sfm}.
}

\submodlongonly{
\subsubsection{Other Submodular Optimization Problems}
\label{sec:other-subm-optim}

Besides maximization and minimization, there are other discrete
optimization problems involving submodular that are interesting and
useful. In this section, we consider several of them.

Firstly, the above submodular greedy maximization procedure has a
guarantee only if the function is submodular and monotone
non-decreasing. A simple trick can be used to find a solution when the
function is monotone non-increasing. Suppose that
$\submodfun$ is submodular and monotone non-increasing, so that
$\submodfun(\submodsetX) \geq \submodfun(\submodsetY)$ whenever
$\submodsetX \supseteq \submodsetY$.  We define a function
$\submodaltfun$ so that
$\submodaltfun(\submodsetX) = \submodfun(\submodgroundset \setminus
\submodsetX)$, so now, whenever $\submodsetX \supseteq \submodsetY$,
$\submodaltfun(\submodsetX) = \submodfun(\submodgroundset \setminus
\submodsetX)) \leq \submodfun(\submodgroundset \setminus \submodsetY)
= \submodaltfun(\submodsetY)$ and hence $\submodaltfun$ is monotone
non-decreasing and is also submodular. Using the greedy algorithm, or whatever
procedure we wish, let $\tilde \submodsetX$
be a set of size $\submodcardcnstr$ having
$\submodaltfun(\tilde \submodsetX) \geq \alpha
\submodaltfun(\submodsetX^*)$ where
$\submodsetX^* \in \argmax_{\submodsetX \subseteq \submodgroundset : |\submodsetX| \leq \submodcardcnstr } \submodaltfun(\submodsetX)$. Thus,
$\submodfun(\submodgroundset \setminus \tilde \submodsetX) \geq \alpha
\submodfun(\submodgroundset \setminus \submodsetX)$
for all $\submodsetX \subseteq \submodgroundset$
with $|\submodsetX| = \submodcardcnstr$. Thus,
we've found a set
$\tilde \submodsetZ = \submodgroundset \setminus \tilde \submodsetX$
of size $\submodcardcnstr' = \submodgroundsetsize - \submodcardcnstr$
such that 
$\submodfun(\tilde \submodsetZ) \geq \alpha
\submodfun(\submodsetZ)$ for any set
$\submodsetZ$ of size $\submodcardcnstr'$.

Often, however, the submodular function is itself neither monotone
non-decreasing nor monotone non-increasing, e.g., the graph cut or the
log-determinant functions. In such cases, there are both randomized and
deterministic
algorithms~\cite{feldman2012optimal,buchbinder2018deterministic}
some of which are quite simple. For example, the
bidirectional greedy algorithm of~\cite{feldman2012optimal} takes
elements in a random order, and randomly inserts the element into
either a growing set or removes the element from a shrinking set depending on
probabilities based on the gain change. These algorithms give a
$O(1/2)$ guarantee for unconstrained non-monotone submodular
maximization, which is the best one can do in polynomial time. 

In some cases, submodular maximization alone does not suffice, and we
wish for the solution to be robust to deletions.  For example, in
summarization problems (see Section~\ref{sec:core-sets-summ}), once a
summary has been constructed, a deletion request might be made on a
certain number of elements in the summary. We may therefore wish
to ensure that the remaining summary after deletion is still highly
valued by the submodular function. This idea leads to an optimization
problem~\cite{orlin2018robust} of the form:
\begin{align}
  \max_{\submodsetX \subseteq \submodgroundset : |\submodsetX| \leq \submodcardcnstr }
  \min_{\submodsetY \subseteq \submodsetX : |\submodsetY| \leq \submodcardcnstr' }
  \submodfun(\submodsetX \setminus \submodsetY)
\end{align}
Here we wish to maximize the submodular function under a cardinality
constraint of size $\submodcardcnstr$ but in such a way that the
solution is robust even to the worst-case deletion of a set of size
$\submodcardcnstr'$. 
For any set
$\submodsetY \subseteq \submodsetX$, the function
$\submodfun_\submodsetY(\submodsetX) = \submodfun(\submodsetX
\setminus \submodsetY)$ is submodular in $\submodsetX$ and thus the
above optimization problem can
be expressed as the maximization
of the minimum over of a set of submodular
functions~\cite{krause2008robust}.
\begin{align}
  \max_{\submodsetX \subseteq \submodgroundset : |\submodsetX| \leq \submodcardcnstr }
  \min_{i} \submodfun_i(\submodsetX)
\end{align}
Seeing the problem this way is not without a loss of useful structure
that is exploited in~\cite{orlin2018robust}, but if the inner
minimization is over only a relatively small number of submodular
functions, then a bi-criterion guarantee can be obtained using a
binary search
procedure~\cite{krause2008robust,powers-constrained-submod-fusion2016,cotter2018constrained}.

As mentioned elsewhere, clustering is the problem of grouping items
together that are homogeneous in some way. A submodular function, when
maximized, prefers heterogeneous items and when minimized prefers
homogeneous items. Using these notions, there are ways to use
submodular objectives to determine such properties in a way that is
very flexible, as flexible as the space of submodular functions. The
general approach is to partition the set $\submodgroundset$ into
blocks. Here, the submodular function itself is used to construct the
partition, unlike with a partition matroid (above) where the partition
is pre-constructed and used as a constraint. The most general form of
problems takes the following forms:
\begin{align}
\text{Diverse Blocks Partitioning Problem: } &
  \max_{\pi \in \Pi_\submodpartmatroidnumblocks} \Bigl[\bar\lambda\min_{i} \submodfun_i(\submodsetX_i^{\pi}) + \frac{\lambda}{\submodpartmatroidnumblocks} \sum_{j = 1}^\submodpartmatroidnumblocks \submodfun_j(\submodsetX^{\pi}_j)\Bigr], \label{prob:maxmin-genmix}  \\
\text{Homogeneous Block Partitioning Problem: } &
  \min_{\pi \in \Pi_\submodpartmatroidnumblocks} \Bigl[\bar\lambda\max_{i} \submodfun_i(\submodsetX^\pi_i) + \frac{\lambda}{\submodpartmatroidnumblocks} \sum_{j = 1}^\submodpartmatroidnumblocks \submodfun_j(\submodsetX^\pi_j)\Bigr].
  \label{prob:minmax-genmix}
\end{align}
where $0 \leq \lambda \leq 1$, is a tradeoff parameter,
$\bar \lambda \triangleq 1-\lambda$ is the complement of the tradeoff
parameter, the set of sets
$\pi = (\submodsetX^\pi_1, \submodsetX^\pi_2, \cdots,
\submodsetX^\pi_\submodpartmatroidnumblocks)$ is a partition of the
ground set $\submodgroundset$ (i.e, recall from above,
$\cup_i \submodsetX^\pi_i = \submodgroundset$ and
$\forall i\neq j, \submodsetX^\pi_i \cap \submodsetX^\pi_j =
\emptyset$), and $\Pi_\submodpartmatroidnumblocks$ refers to the set
of all possible partitions of $\submodgroundset$ into
$\submodpartmatroidnumblocks$ blocks.  The parameter $\lambda$
controls the objective: $\lambda=1$ is the average case, $\lambda=0$
is the robust case, and $0 < \lambda < 1$ is a mixed case. In general,
these problems are intractable even to approximate. Assuming
$\submodfun_1, \submodfun_2, \cdots,
\submodfun_\submodpartmatroidnumblocks$ are all polymatroid functions,
however, enables relatively simple algorithms to have approximation
guarantees. Polymatroid functions also mean that the partitioning
problems are natural to a variety of practical problems.  Often it is
assumed that the $f_i$'s are identical to each other (the
\emph{homogeneous} case) as the problem becomes similar, while the
more general \emph{heterogeneous} case is stated above.  Taken
together, these problems are called \emph{Submodular Partitioning}.

Problem~\ref{prob:maxmin-genmix} asks for a partition whose blocks
each (and that collectively) have a high valuation according to the
polymatroid functions. Problem~\ref{prob:maxmin-genmix} with
$\lambda=0$ is called \emph{submodular fair allocation} (SFA), where
the goal is to choose the partition so that the worse valued block is
as large as possible. If one thinks of $\submodgroundset$ as
consumable goods that can be allocated to people, and $\submodfun_i$
as a valuation function for person $i$, this problem ensures that the
poorest person gets well taken care of.
Problem~\ref{prob:maxmin-genmix} for $\lambda=1$, on the other hand,
is called \emph{submodular welfare}. Again, when $\submodgroundset$ is
a set of goods, this procedure asks only that the average allocation
of goods to consumers is high, even if the partition allocates little
to the poorest and much to the richest individual.  For various
additional assumptions on the submodular functions, there is much to
exploit. For the $\lambda=0$ case,
see~\cite{asadpour2010approximation,golovin2005max,khot2007approximation},
while for the $\lambda=1$ case,
see~\cite{vondrak2008optimal,fisher1978analysis,vondrak2008optimal}.

Problem~\ref{prob:minmax-genmix} asks for a partition whose blocks
each (and that collectively) are internally homogeneous. This is
typically the goal in clustering algorithms so we can think of
problem~\ref{prob:minmax-genmix} as a submodular generalization of
clustering.  When $\lambda=0$, Problem~\ref{prob:minmax-genmix} is
called the \emph{submodular load balancing} (SLB) problem and when
$f_i$'s are all modular, it is called {\em minimum makespan
  scheduling}. Like above, are solutions for such special
cases~\cite{hochbaum1988polynomial,lenstra1990approximation}. Even in
the homogeneous submodular case, however, the problem is shows that
the problem~\cite{svitkina2008submodular} is information theoretically
hard to approximate within $o(\sqrt{n/\log n})$.  When $\lambda=1$,
problem~\ref{prob:minmax-genmix} becomes the \emph{submodular multiway
  partition} (SMP) problem for which one can obtain
$2$-approximations~\cite{chekuri2011approximation,zhao2004generalized,
  narasimhan2005-q} in the homogeneous case and
$O(\log n)$~\cite{chekuri2011submodularSCA} in the heterogeneous case.

The general $0 \leq \lambda \leq 1$ case for both
Problem~\ref{prob:maxmin-genmix} and~\ref{prob:minmax-genmix} was
addressed in~\cite{kaiwei2015nips_submod_partitioning} and in the
Problem~\ref{prob:maxmin-genmix}, this was extended
in~\cite{cotter2018constrained} to include a submodular
cross-block diversity term in the objective, i.e., one that
prefers not only intra-block but also inter-block diversity.

Even when the function of interest is not submodular, but is
\emph{approximately} submodular, or decomposable into submodular
components, algorithms for submodular optimization can be relied on to
behave reasonably well.  For example, another form of submodular
optimization strives to minimize a non-submodular function that has
been decomposed into a difference of submodular functions.  Let
$\submodaltaltfun: 2^\submodgroundset \to \real$ be {\bf any}
arbitrary set function, meaning neither submodular nor supermodular
nor possessing any other useful structural property. Then
there are two
polymatroid functions
$\submodfun$ and $\submodaltfun$
such that
we can
decompose $\submodaltaltfun$ into a difference ---
this means that for all $\submodsetX$,
$\submodaltaltfun(\submodsetX) = \submodfun(\submodsetX) -
\submodaltfun(\submodsetX)$.
Very much like the difference of concave
(DC) optimization, this is known as a difference of submodular (or DS)
decomposition~\cite{narasimhan2005-subsup,iyer2012-min-diff-sub-funcs}
of $\submodaltaltfun$ and we may, exploiting the properties of this
decomposition, either maximize or minimize
$\submodaltaltfun$.  This can be particularly useful when the function
$\submodaltaltfun$ has a natural decomposition, for example the mutual
information function when seen as a set function where
$\submodaltaltfun(\submodsetX) = I(\entropyrv_\submodsetX;
\entropyrv_\submodsetY) = H(\entropyrv_\submodsetX) -
H(\entropyrv_\submodsetX | \entropyrv_\submodsetY) =
\submodfun(\submodsetX) - \submodaltfun(\submodsetX)$.  This is a DS
decomposition since both
$\submodfun(\submodsetX) = H(\entropyrv_\submodsetX)$ and
$\submodaltfun(\submodsetX) = H(\entropyrv_\submodsetX |
\entropyrv_\submodsetY)$ are polymatroid functions in $\submodsetX$.
In fact, submodular functions can be used to define generalized
combinatorial forms of mutual and conditional mutual
information~\cite{iyer-cmi-alt-2021}, ones that are much more
practical for a variety of problems since even computing one entropic
query can involve a complex probabilistic inference problem.  A
special case of DS optimization asks to maximize the sum of a monotone
submodular and a monotone supermodular (or a BP) function and this was
studied in~\cite{bai-bp-functions-2018} (more on this below).
}

\submodlongonly{

Algorithms for DS and BP optimization
use a form of discrete semi-gradients, analogous to semi-gradients
that exists for concave and convex functions, except for here there
are sub-gradient and super-gradients that are available both for
submodular and supermodular functions. This is another reason that
submodular functions are both like and unlike both convex and concave
functions.  Related to DS problems are those that involve maximizing
one submodular objective while using another submodular objective as a
constraint, which use the second function to define submodular
sublevel sets in problems such as
$\max_{\submodsetX} \submodfun(\submodsetX)$ subject to
$\submodaltfun(\submodsetX) \leq \alpha$ or alternatively,
$\min_{\submodsetX} \submodfun(\submodsetX)$ subject to
$\submodaltfun(\submodsetX) \geq \alpha$, as was introduced
in~\cite{rishabh2013-submodular-constraints}. Instances of using
multiple submodular functions in this way, and comparing DS and
constrained optimization formulations, was explored for speech
recognition applications
in~\cite{liu-svitchboard-fisver-submodular-2017}.
}

\submodlongonly{
\subsection{Continuity, Polyhedra, and Extensions}
\label{sec:cont-polyh-extens}

One of the strategies for optimizing submodular functions is via
continuous relaxations. On this topic, another way that submodular
functions are like convex functions, and that is related to their
continuous extension. Recall from above that a submodular function is
defined by offering a value to every vertex of the
$\submodgroundsetsize$-dimensional hypercube.

Recall, any continuous valued function can be transformed into a
unique convex function via its convex envelope. That is, given any
function $h : \mathcal D_h \to \mathbb R$, where
$\mathcal D_h \subseteq \mathbb R^n$, we may define a new function
$\submodccl h: \mathbf R^n \to \mathbb R$ via the following
\begin{align}
  \submodccl h(x) = \sup \{ g(x) : g \text{ is convex \& }
    g(y) \leq h(y), \forall y \in \mathcal D_h  \}.
\end{align}
It can be shown that (1) $\submodccl h(x)$ is convex, (2)
$\submodccl h(x) \leq h(x), \forall x$, and (3) if $g(x)$ is any
convex function having the property that $g(x) \leq h(x), \forall x$,
then $g(x) \leq \submodccl h(x)$.  The function $\submodccl h(x)$ is
the convex envelope of $h$, and is in some sense the
everywhere-largest-valued function lower than $h$ that is still convex.

Submodular functions are discrete functions, so in order to extend
them to a continuous function, and then compute the convex envelope,
we must start by placing the value $\submodfun(\submodsetX)$ for each
$\submodsetX$ at some coordinate in continuous space, and for this
purpose we use the vertices of the hypercube. A \emph{continuous
  extension} $\bar \submodfun : \real^\submodgroundset \to \real$ of a
set function $\submodfun : 2^\submodgroundset \to \real$ is a
continuous function that agrees with the set function on all vertices
of the hypercube.  This means that
$\bar \submodfun(\submodcharv_\submodsetX) = \submodfun(\submodsetX)$
for all $\submodsetX \subseteq \submodgroundset$.  Any set function
has an infinite number of continuous extensions, and any set function
has an infinite number of extensions that may be either convex or
concave.  The tightest convex extension of an arbitrary set function
is known as the \emph{convex closure}, which is a function
$\submodccl \submodfun(\submodvecx) : [0,1]^\submodgroundset \to
\real$ defined in the following way:
\begin{align}
  \submodccl \submodfun(\submodvecx) = \min_{p \in \triangle ^\submodgroundsetsize(\submodvecx)} \sum_{\submodsetX \subseteq \submodgroundset} p_\submodsetX \submodfun(\submodsetX)
\label{eq:convexclosure}
\end{align}
where
$\triangle ^\submodgroundsetsize(\submodvecx) = \{ p \in
\real^{2^\submodgroundsetsize} : \sum_{\submodsetX \subseteq
  \submodgroundset} p_\submodsetX = 1; \; p_\submodsetX \geq 0, \;
\forall \submodsetX \subseteq \submodgroundset ; \text{ \& }
\sum_{\submodsetX \subseteq \submodgroundset} p_\submodsetX
\submodcharv_\submodsetX = x \}$ is the probability simplex in
$2^\submodgroundsetsize$ dimensions. It can be shown that
$\submodccl \submodfun(\submodvecx)$ tight (i.e.,
$\forall \submodsetX \subseteq \submodgroundset$, we have
$\submodccl \submodfun(\submodcharv_\submodsetX) =
\submodfun(\submodsetX)$), that $\submodccl \submodfun$ is convex (and
thus, that any arbitrary set function has a tight convex extension),
that the convex closure $\submodccl \submodfun$ is the convex envelope
of a function taking values only on the hypercube vertices, and which
takes value $\submodfun (\submodsetX)$ on hypercube vertex
$\submodcharv_\submodsetX$ for all
$\submodsetX \subseteq \submodgroundset$.

We say it is tightest such
convex continuous extension of $\submodfun$ in the sense that it is
the convex envelope.  It is also apparent that the function is defined
as a minimization over an exponentially large (in
$\submodgroundsetsize$) space of a sum over an exponential number of
terms. Some immediate and interesting questions, then, to ask are: (1)
when is $\submodccl \submodfun(\submodvecx)$ computationally feasible
to obtain or estimate?  (2) When does
$\submodccl \submodfun(\submodvecx)$ have interesting mathematical
properties?  and (3) When is $\submodccl \submodfun(\submodvecx)$
useful as a surrogate for a given $\submodfun$? It turns out that when
$\submodfun$ is submodular, there will be good answers to all three
questions as we shall now see.

The first quantity we will need to define is the submodular
polyhedron, i.e., the polyhedron $\mathcal P_\submodfun$ associated
with a submodular function $\submodfun$. 
\begin{equation}
  \mathcal P_\submodfun = \{ \submodmodaltfun \in \real^\submodgroundset :
  \submodmodaltfun(\submodaltsetX) \leq \submodfun(\submodaltsetX), \forall
  \submodaltsetX \subseteq \submodgroundset \}
\end{equation}
This polyhedron lies within an $|\submodgroundset|$-dimensional space
and is defined via an exponential set of inequalities. It can be shown
that this polyhedron contains the same information as $\submodfun$
(i.e., given $\mathcal P_\submodfun$ we can reconstruct $\submodfun$).
When $\submodfun$ is a {\em polymatroid function}, then
$\mathcal P_\submodfun \cap \real_+$ is called a polymatroid since,
for a variety of reasons, it is a natural polyhedral generalization of
a matroid. It is for this reason that non-negative monotone
non-decreasing submodular functions are called polymatroid functions
as we mentioned above. We also define a polytope associated with a
submodular function called the {\em base polytope} defined as
\begin{align}
  \mathcal B_\submodfun =
  \mathcal P_\submodfun \cap
  \left\{ x \in \mathbb \real^\submodgroundset :  x(\submodgroundset) = \submodfun(\submodgroundset) \right\}
  \label{eq:base_polytope}
\end{align}
Since $\mathcal B_\submodfun$ can also have an exponential (in
$\submodgroundsetsize$) number of facets, we would expect that using
$\mathcal B_\submodfun$ as a constraint in a linear programming
problem would lead to an intractability. Indeed, consider the
following parameterized linear programming problem which defines a
function $\submodlex \submodfun : \real^V \to \real$:
\begin{align}
  \submodlex \submodfun(\submodvecx) = \max( \transpose{\submodvecx} \submodvecy : \submodvecy \in \mathcal B_\submodfun ).
\label{eq:lexlinearprogramming}
\end{align}
This function is convex (since it maximizes over a set of linear
functions) but interestingly, when $\submodfun$ is submodular, we have
that
$\submodlex \submodfun(\submodvecx) = \submodccl
\submodfun(\submodvecx)$.  Even more interestingly, it is possible to
compute this function easily using what is known as the Lov\'asz
extension, defined as follows. Given a particular $\submodvecx$,
find an ordering $\submodorder = (\submodorder_1,\submodorder_2, \dots, \submodorder_\submodgroundsetsize)$ of the elements in $\submodgroundset$
where $\submodorder_i \in [\submodgroundset]$ that
sorts the elements of $\submodvecx$ descending, meaning
that
$\submodvecx_{\submodorder_1} \geq \submodvecx_{\submodorder_2} \geq
\dots \geq \submodvecx_{\submodorder_\submodgroundsetsize}$.
We also define a series of sets
$\submodsetchain_i = \{ \submodorder_1, \submodorder_2, \dots, \submodorder_i \}$
for $i \in [\submodgroundsetsize]$; thus
$\emptyset = \submodsetchain_0 \subset
\submodsetchain_1 \subset \submodsetchain_2 \subset \dots \subset
\submodsetchain_\submodgroundsetsize = \submodgroundset$. Such
a sequence of sets is called a \emph{chain}.
Then
Lov\'asz~\cite{lovasz1983submodular} showed that
\begin{align}
  \submodlex \submodfun(\submodvecx) &=
  \submodvecx_{\submodorder_\submodgroundsetsize}\submodfun(\submodgroundset)
  + \sum_{i=1}^{\submodgroundsetsize-1}
                                       (\submodvecx_{\submodorder_i} - \submodvecx_{\submodorder_{i+1}})\submodfun(\submodsetchain_i) 
  = \sum_{i=1}^{\submodgroundsetsize} \submodlexweight_i \submodfun(\submodsetchain_i)
\label{eq:lovaszextension}
\end{align}
where $\submodlexweight_i$ is defined accordingly. We see that
computing this expression is relatively easy --- it requires only a
sort of the elements of $\submodvecx$ and then
$O(\submodgroundsetsize)$ evaluations of $\submodfun$, much better
than the exponential costs seen above.  The right-hand side of
Equation~\eqref{eq:lovaszextension}, in and of itself, by no means is
guaranteed to be convex.  Lov\'asz~\cite{lovasz1983submodular} showed,
however, that not only is this expression equal to the convex closure
of the submodular function, but that the expression on the right of
Equation~\eqref{eq:lovaszextension} is convex if and only if
$\submodfun$ is submodular. This is a strong relationship between
submodular and convex functions and is one of the reasons it was
possible to show that submodular function minimization was possible in
polynomial time~\citep{grotschel1981ellipsoid} as mentioned earlier.
One should also be aware that the Lov\'asz extension is equivalent to
the Choquet integral~\cite{choquet1953theory} which was defined much
earlier.

On the other hand, fondly recall the simple example of submodular
function $\submodfun(\submodsetX) = \sqrt{|\submodsetX|}$ or the
feature-based submodular functions
$\submodfun(\submodaltsetX) = \sum_{\submodfeatureel \in
  \submodfeatureset} \submodconcavefun_\submodfeatureel
(\sum_{\submodaltelx \in \submodaltsetX}
\submodmodfun_{\submodfeatureel,\submodaltelx})$ where
$\submodconcavefun_\submodfeatureel$ are concave.  So, while the convex
closure of a submodular function is easy to compute, and the form of
the extension is convex if and only if the function $\submodfun$ is submodular,
it is a concave (not a convex) function composed with a modular
function that is submodular. Thus, asking if a submodular function is
more concave-like or more convex-like, or neither, is really a
meaningless question. Indeed, it is important to realize that submodularity defines
inequalities between variables of different dimensions, while
convexity also requires a certain relationship even amongst values in
the same dimension and is convex in all directions (e.g., if
$\submodconvexfun : \real^\submodgroundsetsize \to \real$ is a convex
function and we define
$\bar \submodconvexfun_{x,y}(\alpha) = \submodconvexfun(x + \alpha
y)$, then $\bar \submodconvexfun : \real \to \real$ is also
convex). There is no such corresponding operation possible with
submodular functions, although there are analogous argument
restrictions for set functions that preserve submodularity (these
include subset restrictions, as we saw above).

Speaking of concave composed with modular functions, recall that
$\submodconcavefun(\submodmodfun(\submodsetX))$ is submodular whenever
$\submodconcavefun: \real \to \real$ is concave.  Supposing instead
that $\submodconcavefun : \real^\submodgroundset \to \real$ is a
multivariate real-valued function, and then we construct set function
as via
$\submodfun(\submodsetX) =
\submodconcavefun(\submodcharv_\submodsetX)$.  We can easily see, by
considering multivariate quadratic functions, that concavity of
$\submodconcavefun$ is insufficient to ensure the submodularity of
$\submodfun(\submodsetX)$.
Consider a concave $\submodconcavefun : \real^2 \to \real$
defined as $\submodconcavefun(\submodrealvectorx) =
\transpose{\submodrealvectorx}\submodhessianmatrix \submodrealvectorx$
where $\submodhessianmatrix = \begin{psmallmatrix} -2 & 1\\ 1 & -3\end{psmallmatrix}$.
This matrix has negative eigenvalues, and thus $\submodconcavefun$
is concave, but we see that 
$\submodconcavefun(\submodcharv_\submodsetX)
+ \submodconcavefun(\submodcharv_\submodsetY) = -3 - 2
<
-3 + 0
= \submodconcavefun(\submodcharv_{\submodsetX \cup \submodsetY})
+ \submodconcavefun(\submodcharv_{\submodsetX \cap \submodsetY})$
when
$\submodcharv_\submodsetX =
\begin{psmallmatrix} 0 \\ 1 \end{psmallmatrix}$
and
$\submodcharv_\submodsetY =
\begin{psmallmatrix} 1 \\ 0 \end{psmallmatrix}$.

There is, however, another condition on a multivariate real-valued
function that is sufficient to ensure submodularity and that also is
useful in its own right, and that is based on a continuous form of
submodularity. We can define a form of continuous submodularity on
real-valued vectors as follows. Given two real-valued vectors
$\submodrealvectorx,\submodrealvectory$ the submodularity of a
function $\submodconcavefun : \real^\submodrelaxrealvectordim \to \real$ is
defined as:
\begin{align}
\submodconcavefun(\submodrealvectorx) +
\submodconcavefun(\submodrealvectory) \geq
\submodconcavefun(\submodrealvectorx \submodjoin \submodrealvectory) +
  \submodconcavefun(\submodrealvectorx \submodmeet \submodrealvectory)
\label{eqn:submodular_inequality_on_lattice}
\end{align}
where the join $\submodrealvectorx \submodjoin \submodrealvectory$
is
defined as the vector with the element-wise maximums of the two
(i.e.,
$(\submodrealvectorx \submodjoin \submodrealvectory)(i) =
\max(\submodrealvectorx(i),\submodrealvectory(i))$ and
$\submodrealvectorx \submodmeet \submodrealvectory$ is the vector of
element-wise minimums.
We see this definition coincides with the
standard discrete definition of submodularity
(Definition~\ref{defn:main_submodular_def}) whenever
$\submodrealvectorx, \submodrealvectory \in
\{0,1\}^\submodrelaxrealvectordim$ are characteristic vectors of sets.  If
$\submodconcavefun$ is also differentiable, this definition is
identical to the Hessian matrix of $\submodconcavefun$ having
non-positive off-diagonal entries. That is,
for all $\submodrealvectorx$,
we have that
\begin{align}
\forall i \neq j, \qquad  \frac{\partial^2 \submodconcavefun(\submodrealvectorx)}{\partial \submodrealvectorx_i, \partial \submodrealvectorx_j} \leq 0
\end{align}
This is analogous to the definition given for supermodularity in the
economics literature~\cite{topkis1998supermodularity,samuelson1947}. It is important
to realize that the above does not place constraints on the diagonal
entries of the Hessian. Submodularity in the differentiable case, in
general, does not require negatives (or positives) along the diagonal.
Therefore, a real valued submodular function might neither be convex
nor concave. A simple example of a two-dimensional quadratic would put
$\submodhessianmatrix = \begin{psmallmatrix} 1 & -2\\ -2 &
  1\end{psmallmatrix}$.  Here, all the off-diagonal entries are
non-positive, and hence the quadratic is submodular, but the
quadratic is neither convex nor concave since it has
both a positive (3) and a negative (-1) eigenvalue.

For classic discrete submodular functions, there is an equivalent
definition of submodularity based on diminishing returns
(Definition~\ref{defn:diminishing_returns_submodular_def}). There is
also a diminishing returns property in the continuous world as well.
This is as follows: for all
$\submodrealvectorx \leq \submodrealvectory$ (taken element-wise,
meaning that
$\forall i, \submodrealvectorx(i) \leq \submodrealvectory(i)$), for
all $\submodel \in \submodgroundset$, for all $\alpha \geq 0$, we have
$\submodconcavefun(\submodrealvectorx + \alpha
\submodcharv_{\{\submodel\}}) - \submodconcavefun(\submodrealvectorx)
\geq \submodconcavefun(\submodrealvectory + \alpha
\submodcharv_{\{\submodel\}}) -
\submodconcavefun(\submodrealvectory)$.  Any real-valued function that
satisfies the above is called {\emph DR Submodular} (for diminishing
returns submodular), and it means that any time we have an increment
in the positive direction, the increment becomes less valuable if we
start from more in the positive direction along any dimension.

For discrete set functions (i.e., functions defined only on vertices
of the hypercube), the two definitions, continuous submodularity and
DR submodularity, are mathematically identical. In the general
continuous case, however, we only have that DR submodularity implies
submodularity, but not vice versa (meaning there are submodular
functions that are not DR submodular). As an example, consider the
quadratic function above with
$\submodhessianmatrix = \begin{psmallmatrix} 1 & -2\\ -2 &
  1\end{psmallmatrix}$. We saw above that such a quadratic is
submodular, but it is not DR submodular. If we take a submodular
function, however, and also insist that it is coordinate wise
concave, then it is DR submodular. Indeed, when differentiable, the DR
submodular condition is identical to $\forall i,j$,
$\frac{\partial^2 \submodconcavefun(\submodrealvectorx)}{\partial
  \submodrealvectorx_i, \partial \submodrealvectorx_j} \leq 0$, note
that this says for {\emph all} $i,j$, not just $i \neq j$ as in the
case of a submodular function. This therefore means that the Hessian
matrix is all non-positive, a condition that does not require
concavity which requires a non-positive definite Hessian.

When are submodular and DR submodular functions useful? Like the
more typical discrete submodular functions, such continuous submodular
functions can often be optimized even though they are neither concave
nor convex~\cite{bach2019submodular,mokhtari2020stochastic}. For
example, DR submodular functions, despite not being convex, can be
exactly minimized~\cite{bian2017continuous} and approximately
maximized \cite{pmlr-v97-bian19a}, very much like discrete submodular
functions. Also, the deep submodular functions
(DSFs)~\citep{bilmes-dsf-arxiv-2017} mentioned above are often
instantiated using certain classes of nested DR-submodular
functions. There are also interesting strategies for approximate
continuous DR-Submodular maximization~\cite{sadeghi2020online}.
Therefore, it seems reasonable that this is a natural generalization
of the discrete to the continuous.

\submodlongonly{
Another extension of discrete submodular functions is that of
bisubmodularity \cite{chandrasekaran1988pseudomatroids,
  qi1989bisubmodular, bouchet1995delta, ando1996characterization,
  thomas2007strongly,mccormick2010strongly,singh2012-bisubmod} and multivariate
submodularity~\cite{santiago2019multivariate}.  A bisubmodular
function is of the form:
$\submodfun: 2^\submodgroundset \times 2^\submodgroundset \rightarrow
\real$.  There are several types of bisubmodular functions
\cite{qi1989bisubmodular}, the common type having the property that
$\submodfun(\submodaltsetX,\submodaltsetY) +
\submodfun(\submodaltsetZ,\submodaltsetZZ) \geq
\submodfun(\submodaltsetX \cap \submodaltsetZ, \submodaltsetY \cap
\submodaltsetZZ) + \submodfun((\submodaltsetX \cup \submodaltsetZ)
\setminus (\submodaltsetY \cup \submodaltsetZZ), (\submodaltsetY \cup
\submodaltsetZZ) \setminus (\submodaltsetX \cap \submodaltsetZ))$
whenever both $\submodaltsetX,\submodaltsetY$ and
$\submodaltsetZ,\submodaltsetZZ$ are nonintersecting. Bisubmodular
functions can be seen as a form of signed (or typed) set membership,
where an item is a member of a set if it is a member of either
$\submodaltsetX$ or $\submodaltsetY$, but the sign is determined by
which member. Bisubmodular functions also have polyhedra for which the
greedy algorithm can be applied~\cite{dunstan1973greedy}. A simpler
form of bisubmodularity was defined in~\cite{singh2012-bisubmod}.
More general still, a submodular function can be defined over any
lattice \cite{birkhoff1948lattice,fan1968inequality}. In fact, the
submodular and DR submodular functions above defined on real vectors
can be considered submodular functions defined on a real lattice, so
this is quite general.
}

}

\submodlongonly{
\subsection{Historical Applications of Submodularity}
\label{sec:backgr-subm}

Submodular functions~\cite{fujishige2005submodular} have a long history in economics
\cite{vives2001oligopoly,carter2001foundations,samuelson1947}, game
theory \cite{topkis1998supermodularity,shapley1971cores},
combinatorial optimization
\cite{edmonds1970matroids,lovasz1983submodular,schrijver2004},
electrical networks \cite{narayanan1997submodular}, and operations
research \cite{cornuejols1990facilitylocation}. Entire issues of
journals of discrete mathematics have been devoted to their exposition
\cite{submodularity2003}. The classic applications of submodularity
lie in their use to describe a variety of properties of problems
arising in graph theory and combinatorial
optimization~\cite{frank1993applications}.  The graph cut function
mentioned above is one typical example, but there are many others (see
\cite{frank1993applications,frank1998applicationsrelaxedsubmodularity}).
Similar to how convergent evolution indicates the utility of a genetic
trait, submodularity, under many different names, has been
independently discovered in a number of different domains.

For example, in statistical physics, it is sometimes known as the
ferromagnetic assumption, standard in the Ising model
\cite{ising1925beitrag,peierls1936ising} and its generalization to the
Potts model~\cite{potts1952some}. In the Ising case,
$\submodrealvectorx \in \{0,1\}^\submodrelaxrealvectordim$ represents a
configuration of $\submodrelaxrealvectordim$ particles and each particle
$\submodrealvectorx_i$ can be in one of two states
$\submodrealvectorx_i \in \{-1,+1\}$.  Each particle, for example,
might be an atom of magnetic material that can have a magnetic moment
oriented in either the ``up'' or the ``down'' directions.  The overall
energy of these particles can be described as the Hamiltonian
$\energy(\submodrealvectorx) = - \sum_{(i,j) \in
  \submodedgesgroundset} \submodneighborenergyconsant
\submodrealvectorx_i \submodrealvectorx_j - \sum_{i \in
  [\submodrelaxrealvectordim]} \submodexternalenergyconsant
\submodrealvectorx_i$ where $\submodedgesgroundset$ are the edges of a
graph.  In the ferromagnetic case, $\submodneighborenergyconsant > 0$
so that the lowest energy configuration (i.e.,
$\sum_{(i,j) \in \submodedgesgroundset} \submodneighborenergyconsant
\submodrealvectorx_i \submodrealvectorx_j$ is maximum) occurs when all
neighboring particles point in the same direction
($\submodrealvectorx_i = \submodrealvectorx_j$). When used as an
energy in a probability model
$\submodprob{\submodrealvectorx} \propto \exp(-\energy(\submodrealvectorx))$
this lowest energy configuration corresponds to one with the highest
probability. How does this correspond to submodularity?  Consider any
pair $(i,j) \in \submodedgesgroundset$.  Since $\beta > 0$, we have
that
$-\beta\times(-1)\times(+1) - \beta\times(+1)\times(-1) \geq
-\beta\times(+1)\times(+1) -\beta\times(-1)\times(-1)$, but we see
that this is precisely the same condition as
Equation~\eqref{eqn:submodular_inequality_on_lattice} applied to the
Boolean lattice (i.e., 0/1-valued vectors). Moreover, we can see
$\sum_{i \in [\submodrelaxrealvectordim]} \submodexternalenergyconsant
\submodrealvectorx_i$ as a modular function. Hence, minimizing
$\energy(\submodrealvectorx)$ is an instance of submodular function
minimization, and the Ising model's energy function is a special case
of a submodular function in the ferromagnetic case.

Another example where submodularity plays a role is in efficiently
computing the earth mover's distance (or Wasserstein metric). We start
with a non-negative supply vector
$\submodrealvectorx = (\submodrealvectorx_1, \submodrealvectorx_2,
\dots, \submodrealvectorx_n)$ (representing the amount or mass of some
substance at position $i \in [n]$) and demand vector
$\submodrealvectory = (\submodrealvectory_1, \submodrealvectory_2,
\dots, \submodrealvectory_m)$ (representing the mass of some substance
needed at position $j \in [m]$). There is an $n \times m$ matrix $C$
where $C_{ij}$ represents the cost, or distance, to transport a unit
of mass between position $i$ and $j$.  Let $Z$ be an $n \times m$
matrix.  The goal is to solve the linear programming problem
$\min_{Z} \sum_{i=1}^n\sum_{j=1}^m C_{ij} Z_{ij}$ subject to
$\sum_{j=1}^m Z_{ij} = \submodrealvectorx_i$,
$\sum_{i=1}^n Z_{ij} = \submodrealvectory_j$, and $Z_{ij} \geq 0$.
The normalized solution to this problem is the earth mover's distance.
We can view this problem as either a transportation problem (the total
mass at locations represented by $\submodrealvectorx$ needs to be
moved to locations represented by $\submodrealvectory$ in the best
possible way) or as a way of computing a best-case distance between
two probably mass functions (say if the vectors sum to
one). Ordinarily, solving this problem requires general linear
programming.  Suppose, however, that the costs satisfy the Monge
property, i.e.: $C_{ij} + C_{rs} \leq C_{is} + C_{rj}$ for all
$1 \leq i < r \leq n$ and $1\leq j < s \leq m$. Then the above
transportation problem can then be solved using a simple greedy
algorithm (known as the north-west corner rule) in only $O(n+m)$
time~\cite{burkard1996perspectives}.  Now suppose we define a function
$\submodconcavefun : [n] \times [m] \to \real$ and set
$\submodconcavefun(i,j) = C_{ij}$.  Then $\submodconcavefun$ is
submodular in the sense
of~\eqref{eqn:submodular_inequality_on_lattice} if and only if the
matrix $C$ is Monge. Hence, the submodularity property allows rapid
computation of the earth mover's distance in some cases. This can be
generalized to higher dimensional tensors as well~\cite{bein1995monge}
using a form of Monge property that is the same as the multivariate
submodularity mentioned in Section~\ref{sec:cont-polyh-extens}.

We've already seen in Section~\ref{sec:entropy} how the Shannon
entropy function is submodular, and how the submodular inequalities
are just a restatement of the classic Shannon
inequalities~\cite{Yeung91,zhang1997non,zhang1998characterization}. In
block coding applications, we might wish to partition a set of random
variables into two disjoint sets so as to minimize the mutual
information between them. If we find such a partition, each set of
random variables can be block coded with a minimum coding length over
all such partitions. If $S \subset \submodgroundset$ is a set of
random variables, then minimizing
$\submodfun(S) = I(S; \submodgroundset\setminus S)$ (where
$I(\cdot;\cdot)$ is the mutual information function) corresponds
precisely to this task. Here, $\submodfun(S)$ is a symmetric
submodular function that can be minimized in $O(n^3)$
time~\cite{queyranne98,nagamochi1992computing} as mentioned in
Section~\ref{sec:subm-minim}.  Incidentally, regarding information
theory, the von Neumann, or quantum, information~\cite{vonneumann1932}
also satisfies a property, called ``strongly sub-additivity'' and
considered an essential and fundamental property of quantum
information, that is essentially the same as submodularity
\cite{delbrurk1936,lieb1973proof,robinson1967mean}.

Cooperative game theory is also an area that makes good use of
submodularity, or rather supermodularity which of course is quite
related. Here, we think of $\submodgroundset$ as a set of ``players'',
any subset $\submodsetX \subseteq \submodgroundset$ of players is
considered a ``coalition'', and a function
$\submodfun: 2^\submodgroundset \to \real$ is considered a ``game'',
in that it assigns a value or worth to each coalition. Interacting
players, if they agree, can strategically join a coalition
$\submodsetX$ and the resulting worth $\submodfun(\submodsetX)$ must
somehow be distributed to the players. As a side note, it is
interesting that game theoretic settings were significantly developed
and advanced~\cite{vonneumann1944} by the same von Neumann who
developed quantum information~\cite{vonneumann1932} --- besides so
many other amazing achievements, he also seems to have touched
submodularity in many ways.  It is usually the case that games are
normalized $\submodfun(\emptyset) = 0$, superadditive
($ \submodfun(\submodsetX) + \submodfun(\submodsetY) \leq
\submodfun(\submodsetX \cup \submodsetY)$ so that the joint coalition
$\submodsetX \cup \submodsetY$ is worth no less than the two coalitions
acting independently, thereby incentivizing cooperation) and
supermodular
(Definition~\ref{eqn:increasing_returns_submodular_inequality}),
although in the game theory literature, supermodularity is sometimes
called convex~\cite{shapley1971cores} due to the analogy that
Definition~\ref{eqn:increasing_returns_submodular_inequality} has with
the non-negativity of second derivatives. The way the worth
$\submodfun(\submodsetX)$ is distributed to the players is given by a
payoff vector, a normalized modular function $\submodmodfun$.  Those
payoff vectors that satisfy both
$\submodmodfun(\submodgroundset) \leq \submodfun(\submodgroundset)$
and $\forall \submodsetX \subseteq \submodgroundset$,
$\submodmodfun(\submodsetX) \geq \submodfun(\submodsetX)$ are those
modular functions that aren't worth more than the entire game, but
that would be satisfactory to any coalition. Such payoff vectors are
called the core of the game and, collectively, they can be seen as a
dual of the base polytope of a submodular function
(Equation~\eqref{eq:base_polytope}).
The question regarding such a
game is what the outcome, or resulting coalition, should be; and also what
the right payoff vector is, considering that the core itself can be an
infinitely large set.
This was addressed in fact by what is known as
the Shapley value, which says that the payoff vector should be based
on the average contribution of each player to all possible
contexts. I.e., the Shapley
value~\cite{shapley1953valuefornpersongames} is a modular function
$\submodmodfun_\text{shapley}$ that can be defined, for all
$\submodel \in \submodgroundset$, as
$\submodmodfun_\text{shapley}(\submodel) =
\frac{1}{\submodgroundsetsize !}  \sum_{\sigma \in \Sigma} \submodfun(
\submodel | \sigma_{< \submodel} )$, where $\Sigma$ is the set of all
total orderings of $\submodgroundset$, $\sigma$ is a particular
ordering, and $\sigma_{< \submodel} \subseteq \submodgroundset$ is the
set of items that precede $\submodel$ in the order $\sigma$.  While
there are other definitions of the Shapley value, this one has an
intuitive definition as the average improvement in value that
$\submodel$ provides over all possible coalitions into which
$\submodel$ is added. Hence, this seems to be a reasonable payoff
amount for $\submodel$. Interestingly, the Shapley value is a member
of the core of the game.  The Shapley value in fact has, of late,
become useful as an attribution method~\cite{lundberg2017unified},
i.e., to evaluate the contribution of parts of a feature vector for
the resulting final output in, for example, a deep neural network.

It is also interesting to note that fuzzy set theory uses fuzzy
measures~\cite{grabisch2010fuzzy} (or capacities) and are a form of
non-additive measures~\cite{pap2006non} on sets in a way that
generalizes standard additive (e.g., Lebesque) measures on sets. These
measures are often submodular or supermodular functions.  Integrals
with respect to these non-additive measures can be defined, and one
extremely useful one is the Choquet integral~\cite{choquet1953theory}
which as we have seen earlier is identical to the \lovasz{} extension
(Equation~\eqref{eq:lovaszextension} and~\cite{lovasz1983submodular}).

In natural language processing, submodularity has been uncovered for a
variety of purposes. For example, the popular Rouge-recall
metric~\cite{lin:04} for judging the quality of a candidate document
summary was shown to be submodular~\cite{lin2011-class-submod-sum}.
In statistical machine translation, submodularity also spontaneously
arose~\cite{biccici2013feature,kirchhoff2014-submodmt} for the
purposes of selecting a subset of training data that was specifically
relative for a given test translation task.  In computer vision,
Markov random fields were at one time all the rage, and since solving
them in general was known to be intractable, there was a concerted
effort to identify conditions that enabled their practical and
tractable solution.  Submodularity arose as that condition, although
it was at that time called, or was strongly related to, potential
functions which are called ``regular'', ``Potts'', or
``attractive''~\cite{boykov2001fast,kolmogorov2002energy,kolmogorov2004energy,boykov2004experimental,kumar2008improved,kumar2009map,jegelka-coop-cut-journal-2016}. 
Algorithms related to graph cut can often solve such models
either exactly or approximately well.  This research led to
algorithms that overcame pure binary image segmentation and led to
move-making algorithms such as alpha expansions and alpha-beta swaps
\cite{boykov2001fast,kolmogorov2004energy,boykov2004experimental}.
This had a big impact on problems such as semantic image segmentation.
Since these are based on Markov random fields but with restricted
potential functions, submodularity also plays a role in inducing fast
algorithms for graphical model inference (see
Section~\ref{sec:prob-model}).

There continues to be much active work on submodularity in the
discrete optimization literature. There has been a surge of interest
in optimizing submodular functions under various constraints, such as
matroid constraints
\cite{calinescu2007maximizing,lee2009submodular,lee2010maximizing},
covering constraints \cite{iwata2009submodular}, cut constraints
\cite{jegelka-coop-cut-journal-2016}, and other combinatorial
constraints~\cite{jegelka2011-online-submodular-min,iyer2013-curvature}
such as paths, st-cuts, spanning trees, and perfect graph matchings.

Most of the algorithms that optimize submodular functions are
inherently sequential, meaning the algorithm steps through a sequence
of decisions in series based on the outcomes of past decisions it has
made. Submodular functions are inherently global and, unlike
graphical models, do not necessarily allow any one set of
elements to be selected without influencing the value of selecting
other sets of elements. This can be seen even when considering the
simple square-root submodular function
$\submodfun(\submodaltsetX) = \sqrt{|\submodaltsetX|}$ seen above,
where every element interacts with every other element, although
trivially. More complex submodular functions still involve global
interaction but much less trivially. This property is anathema to
parallel computation where we wish, in a structure, to find a
sub-structure that can be computed relatively independently of other
sub-structures. Submodular functions, in general, offer no such
independence-between-sub-structures property. Nevertheless, there have
been efforts to develop parallel submodular optimization algorithms.
In~\cite{mirzasoleiman2013distributed}, for example, the ground set is
partitioned, each block is optimized over separately, and then the
resulting solutions are merged and optimized over again. More
recently, researchers began studying the tradeoffs between
approximation quality of an algorithm that proceeds in ``rounds'',
where each round one is allowed to perform polynomially many steps of
computation. One in particular wishes to find the minimum number of
sequential rounds (i.e., the minimum time) to achieve a given
approximation quality where a polynomial number of computational steps
(e.g., set function queries) may be performed embarrassingly in
parallel, meaning each function query may be derived using results only from
previous rounds but not from any results of the concurrent round thus
allowing each query to occur simultaneously in parallel.
The minimum number of rounds, each having polynomial cost, to achieve a
given (often constant) approximation ratio is called the {\em adaptive
  complexity} of the problem. One wishes for algorithms that have the
best (lowest) adaptive complexity for a given approximation
ratio~\cite{li2020polynomial,balkanski2018adaptive,breuer2020fast}.

It is an exciting time for combinatorial optimization involving
submodular and supermodular objectives, as many combinatorial problems
that in the past have used only modular cost functions can be extended
to submodular functions while still retaining polynomial time
approximability; this gives combinatorial problems significantly more
power and expressivity.  These algorithms, especially when they are
shown to be polynomial time approximable in one form or another, open
the doors up to an enormous number of useful applications. A few of
these applications, in the case of machine learning, are listed next.

}

\submodlongshortalt{
  \section{Applications of Submodularity in Machine Learning and AI}
}{
  \subsection{Applications of Submodularity in Machine Learning and AI}
 }
  
\label{sec:appl-subm-mach}

Submodularity arises naturally in applications in machine learning and
artificial intelligence, but its utility has still not yet been as
widely recognized and exploited as other techniques. For example,
while information theoretic concepts like entropy and mutual
information are extremely widely used in machine learning (e.g., the
cross-entropy loss for classification is ubiquitous), the
submodularity property of entropy is not nearly as widely explored.

Still, in the last several decades, submodularity has been increasingly
studied and utilized in the context of machine learning. \submodlongonly{It has now
become an essential ingredient in many machine learning and data
scientists' book of recipes.  Submodular maximization has been
successfully used for selecting summaries (or sketches) of text
documents~\cite{lin2010-submod-sum-nlp, lin2011-class-submod-sum,
  hui2012-submodular-shells-summarization} recorded
speech~\cite{lin2009-submod-active-seq,liu:submodular,
  wei-duc-sum-for-speech-2013, wei:submodular} image
compendia~\cite{sebastian2014-submod-image-sum} sets of protein
sequences~\cite{max2018-protein-submod}, sets of genomics
assays~\cite{wei2016genomicselection} and sets of genomic
loci~\cite{gasperini:genome-wide}.  It has been used for
variable/feature selection \cite{krause2005near}, modeling and
selecting influential nodes in a social network~\cite{kempe2003maximizing,kempe2015maximizing},
Gaussian processes~\cite{guestrin2005}, and active learning
\cite{guillory2010-icml-interactive-sub,guillory2011-noisy-interactive-submod-cover,golovin2010adaptive}.
Submodular minimization has found use in graphical model inference
\cite{kolmogorov2004energy} and PAC
learning~\cite{narasimhan2004-paclearn}, clustering
\cite{narasimhan2005-q,narasimhan2007-bcut}, and sensor placement for
graphical models~\cite{krause2005near}.

Central to these papers
is that some form of entropy, mutual information, and social network
influence of a set of nodes in a social network are all submodular.
One would also expect (intuitively) that the information capacity of
subsets of features or queries could be formulated as a submodular
function.  Indeed, submodularity has also become a key component in
active learning
\cite{hoi2006batch,krause2008robust,guillory2009-label-selection,Settles2010,Guillory2010,golovin2010adaptive,guillory2011-online-submod-set-cover,guillory2011-active-semisupervised-submodular}.
It is natural, therefore, to expect that results from the theory of
submodular functions can be used to resolve,
both, some of the
theoretical and some of the practical problems arising in machine learning.

}In the below we begin to provide only a brief survey of some of the major
sub-areas within machine learning that have been touched by
submodularity. The list is not meant to be exhaustive, or even
extensive. It is hoped that the below should, at least, offer a
reasonable introduction into how submodularity has been and can
continue to be useful in machine learning and artificial intelligence.

\submodlongonly{
  \subsection{A Summary of the History of Summarization: Sketching, CoreSets, Distillation, and Data Subset \& Feature Selection}
}
\submodshortonly{
  \subsection{Sketching, CoreSets, Distillation, and Data Subset \& Feature Selection}
}
\label{sec:core-sets-summ}

\submodlongonly{
The idea of a summary goes back at least to ancient Rome, where the
Latin word ``Summ\=arium'' means the highest, or top most, number ---
this is because numbers were added from bottom to top with the
mathematical ``sum'' on the top row. The summary is enough
of an important concept for humanity that it has been granted many 
different terms --- just in English, this includes
abridgment,
abstract,
bottom line,
cliff note,
compendium,
conspectus,
digest,
synopsis,
outline,
overview,
recapitulation,
rundown,
and so on.
In this era of big data that the fields of machine learning and
artificial intelligence are currently in, a correspondingly large
number of notions of summary have been developed as well. In this
section, we briefly offer them a survey and taxonomy, and explain how
submodularity plays a role in producing a variety of useful
summarization algorithms.
}

A summary is a concise representation of a body of data that can be
used as an effective and efficient substitute for that data.  There
are many types of summaries, some being extremely simple.  For example,
the mean or median of a list of numbers summarizes some property (the
central tendency) of that list. A random subset is also a form of
summary. \submodlongonly{Most of the time, however, summaries are more complex and
have a specific purpose.  For example, the result of a web search is
a summary of the web that is relevant to the given query.  All
clustering algorithms are summarization procedures in that the
resulting set of cluster centers (however they may be defined)
summarize the data.  Any machine learning procedure, in fact, can be
viewed as a summarization. Machine learning is the art of telling a
computer what one wants the computer to tell a second computer about a
lot of data, and this takes the form of model parameters, summarizing
the training data.  For example, when we train an
SVM~\cite{Scholkopf2001book}, the samples corresponding to the support
vectors constitute a weighted sum over a subset of samples.  When we
train a neural network, the parameters of the resulting model
summarize the aspects of the data that are useful for mapping from
input to output.}

Any given summary, however, is not guaranteed to do a good job serving
all purposes. Moreover, a summary usually involves at least some
degree of approximation and fidelity loss relative to the original,
and different summaries are faithful to the original in different ways
and for different tasks.  For these and other reasons, the field of
summarization is rich and diverse, and summarization procedures are
often very specialized.

\submodlongonly{
Purposefully designed algorithms that summarize data began outside the
field of machine learning. For example, in statistics, the notion of a
sufficient statistic~\cite{bickel2015mathematical} is a computation on
a sample of data that, once known, renders the sample irrelevant to
estimating the parameters of a distribution from which the sample was
drawn. In other words, the data and model parameters are rendered
independent given a sufficient statistic and, in this sense, the
sufficient statistic summarizes everything about the data necessary to
estimate the parameters.  In statistics and computer science, random
subsets are used frequently for analysis and even such agnostic
computations comprise statistically unbiased summaries of the body of
data being sampled.  This is also true, for example, in the field of
public opinion polling where the sample of people selected to answer
poll questions constitutes a (hopefully) representative summary of the
entire population.
}

Several distinct names for summarization have been used over the past
few decades, including ``sketches'', ``coresets'', (in the field of
natural language processing) ``summaries'', and ``distillation.''

Sketches~\cite{cormode2017data,cormode2020small,cormode2012synopses},
arose in the field of computer science and was based on the
acknowledgment that data is often too large to fit in memory and too
large for an algorithm to run on a given machine, something enabled by
a much smaller but still representative, and provably approximate,
representation of the data.  \submodlongonly{Viewable as a summary data structure that
efficiently represents the data in some way, a sketch is computed
under a given computational model.  For example, in the streaming
case, the data is never held in memory at the same time and instead is
sequentially processed where a machine can simultaneously hold only
one (or a small number) of samples simultaneously in memory as well as
the sketch being incrementally updated.  This allows the data to be
unboundedly, or at least unknowingly, large, but the challenge is that
there is one and only one chance to view a given data sample since the
algorithm is unable to go back and revisit a sample later.  This
includes one of the earliest sketching
algorithms~\cite{flajolet1985probabilistic} that computes the number
of different types of elements in a large dataset --- the data,
however, is available only serially (e.g., imagine a sequence of balls
arriving in some order and the goal is to count the number of distinct
ball colors).  Other models include the ability to merge two sketches
such that the result of the merged sketch is the same as if the sketch
was produced from the merged data. Linear sketches are particularly
useful here since they have the property that the sketch of the sum or
concatenation of two datasets is the same as the sum of the sketches.
There are many sketching algorithms and sketch data structures, and
usually each is developed to solve a very specific problem.  For
example, sketching can be as simple as counting a number of
items. Normally one needs $O(\log n)$ bits to count $n$ items but when
$n$ is very big, even $O(\log n)$ bits is too costly, especially when
there are many different item types each needing to be counted. A
sketching algorithm can be developed that counts using only
$O(\log \log n)$ bits. Other examples of sketching include Bloom
filters (a data structure that approximates set membership queries in
that detecting non-membership is exact but detecting membership is
only likely but not assured), detecting ``Heavy hitters'' (i.e.,
elements in a large array that are above a certain threshold),
dimensionality reduction (often via the Johnson-Lindenstrauss family
of results which guarantee that random projections of vectors
approximately and with high probability preserve distance in a lower
dimensional space as long as the lower dimensional space is not too
low), and various graph based sketches.  Even uniformly-at-random
sampling procedures have sketching-like algorithms. Rather than
load all items in just to randomly sample, one can stream them in and
occasionally choose a random update in an online fashion as done by
the famous reservoir sampling algorithm~\cite{knuth2014art}. Sketching
is also useful for a variety of problems in numerical linear
algebra~\cite{liberty2013simple,woodruff2014sketching} including
sketches for multiplying together very large matrices.  It is almost
always the case that sketching algorithms offer guarantees, often in
the form of having a high $1-\delta$ probability of having an error no
more than $\epsilon$, and the computational complexity of the
algorithms grows inversely with $\epsilon$ and $\delta$ (ideally
poly-logarithmically).  A very good sketch of sketching algorithms is
given in~\cite{cormode2020small}.
}

Coresets are similar to sketches and there are some properties that are
more often associated with coresets than with sketches, but sometimes
the distinction is a bit vague.  The notion of a
coreset~\cite{badoiu2002approximate,agarwal2005geometric,buadoiu2008optimal}
comes from the field of computational geometry where one is interested
in solving certain geometric problems based on a set of points in
$\real^\submodrealvectordim$.  For any geometric problem and a set of
points, a coreset problem typically involves finding the smallest
weighted subset of points so that when an algorithm is run on the
weighted subset, it produces approximately the same answer as when it
is run on the original large dataset.  For example, given a set of
points, one might wish to find the diameter of set, or the radius of
the smallest enclosing sphere, or finding the narrowest annulus (ring)
containing the points, or a subset of points whose $k$-center
clustering is approximately the same as the $k$-center clustering of
the whole~\cite{badoiu2002approximate}.

\submodlongonly{
Common to the coreset paradigm is that each point in the summary is
assigned a weight that indicates the importance in each point.  Also,
in common with sketching algorithms, most coreset algorithms offer
approximation guarantees where the size of the coreset depends
inversely on an error $\epsilon$ and possibly (but ideally not) on the
dimensionality of the points.  Soon after their introduction, coresets
became useful in machine learning, e.g., the core-SVMs
of~\cite{tsang2005core} that facilitate the use of SVMs on larger
datasets by identifying a coreset.  More recently, the notion of
coresets and summarization have been used for many specific
applications in machine learning, good surveys and examples are
in~\cite{DBLP:journals/corr/abs-1910-08707,DBLP:journals/corr/abs-2011-09384,wei2015-submodular-data-active,munteanu2018coresets},
the last of which categorizes four types of coreset constructions for
machine learning: geometric, gradient descent, random sampling, and
sketching. In the sketching case, the coreset is viewed as a form of
dimensionality reduction in data space which is like a transpose of
dimensionality reduction in feature space (more on this when we
discuss distillation below).
}

Document summarization became one of the most important problems in
natural language processing (NLP) in the 1990s although the idea of
computing a summary of a text goes back much further to the
1950s~\cite{luhn1958automatic,edmundson1969new}, also and
coincidentally around the same time that the
CliffsNotes~\cite{enwiki:1051798921} organization began. There are two
main forms of document summarization~\cite{yao2017recent}.  With {\em extractive
  summarization}~\cite{nenkova2012survey}, a set of sentences (or
phrases) are extracted from the documents needing to be summarized,
and the resulting subset of sentences, perhaps appropriately ordered,
comprises the summary. \submodlongonly{Extractive summarization therefore is a subset
selection problem, where if $\submodgroundset$ represents the set of
all sentences, we wish to find a subset
$\submodsetX \subseteq \submodgroundset$ that represents
$\submodgroundset$ and in particular $\submodsetX$ must represent
everything not selected $\submodgroundset \setminus \submodsetX$.  In
an extractive summary, each sample is a genuine sentence from the
original and so there is no need to ensure that the summary sentences
are grammatical or factually accurate.}

With abstractive summarization~\cite{lin2019abstractive}, on the other
hand, the goal is to produce an ``abstract'' of the documents, where
one is not constrained to have any of the sentences in the abstract
correspond to any of the sentences in the original documents. With
abstractive summarization, therefore, the goal is to synthesize a
small set of new pseudo sentences that represent the original
documents. CliffsNotes, for example, are abstractive
summaries of the literature being represented.
\submodlongonly{Conceptually, abstractive summarization is more natural and
flexible, but it is a computationally more challenging task since one
must ensure that the synthesized sentences are grammatical, natural,
relevant, and factual.  Indeed, there is research that compares
extractive with abstractive summarization and that report ``factual
hallucinations'' in the abstractive
summaries~\cite{maynez2020faithfulness, kryscinski2020evaluating}, and
that finds that extractive summarization is often better than
abstractive due to its inherent faithfulness and
factual-consistency~\cite{huang-etal-2020-achieved,durmus-etal-2020-feqa}.
Extractive summarization is considerably easier computationally since
the process involves only picking and choosing sentences. With
abstractive summarization, one must synthesize and regenerate new
sentences, similar to language translation.  With extractive
summarization, however, it can be more difficult to piece together the
extracted sentences into a coherent paragraph or whole article, since
the selected sentences might just not flow well together as a
text. With abstractive summarization, the sentences can be composed for
the purposes of being a coherent and fluid whole, with clarity,
organization, voice, grammar, and even style. It is also possible
for an abstractive summarization to be more concise since in any
accurate extractive summarization there might be redundancies that are
impossible to remove without either changing the sentences or removing
relevant information.  Therefore, there are tradeoffs between
abstractive and extractive summarization.  In fact, extractive
summarization is often used as the first step of abstractive
summarization in NLP
including~\cite{pilault2020extractive,gehrmann-etal-2018-bottom,
  liu2018generating,hsu2018unified,li2021ease} --- interestingly,
there is evidence that this might even be a mechanism used by
human summarizers~\cite{jing1999decomposition} where extractions are done first
before concept merging and resynthesis into an abstractive summary
later.

Another useful variant of summarization commonly encountered in the
text domain is called {\em update summarization}~\cite{tac-summarization-task-overview-2008}. Here,
the assumption is that a person has already read a set of articles and
is interested only in a summary of any updates, or any new
information, that has happened that is not in the articles that have
already been read. Update summarization is natural in document
summarization since at any given time there are too many articles to
read, but a human does not start from zero knowledge, rather the human
starts from having read articles up to a previous time point.  Update
summarization can be either extractive or abstractive.

A big part of document summarization research is in developing
automatic ways to judge the quality of a summary, the reason being
that obtaining human judgments of summaries is even more challenging
than the human labeling tasks. This is because the human now needs to
judge not just what objects or concepts lie within each sample, but
instead needs to examine a large sample collection and tell if
it is representative of the whole.  Due to this difficulty, accurate
automated summary assessment algorithms are extremely valuable.  One
strategy used for document summarization requires a set of ground truth
reference summaries of a collection of documents, and all new
candidate summaries are judged against statistics within these
references.  Since there can be many different summaries all of which
are equally good, one strategy to do this is via n-gram statistics. A
classic example of this approach is the ROUGE metric~\cite{lin:04}.
The ROUGE-N recall metric can be written as the following set function:
\begin{align}
\submodfun_{\text{ROUGE-N-Recall}}(S) 
\triangleq \frac{ \sum_{i=1}^K \sum_{e \in R_i } \min (c_e(S),r_{e,i})}{\sum_{i=1}^K\sum_{e \in R_i} r_{e,i}},
\end{align}
where $K$ is the number of reference summaries, $R_i$ is the set of
sentences in the $i^\text{th}$ reference summary, $e$ is a particular
n-gram, $c_e(S)$ is the number of times n-gram $e$ occurs in a
hypothesized set of sentences $S$, and $r_{e,i}$ is the number of times
n-gram $e$ occurs in reference summary $i$.  Since $c_e(S)$ is
non-negative monotone and $\min(x,a)$ is a concave non-decreasing
function of $x$, $\min(c_e(S),r_{e,i})$ is monotone submodular since
it is a monotone non-decreasing concave function composed with a
modular function.  Since summation preserves submodularity, and the
denominator is constant, we see that
$\submodfun_{\text{ROUGE-N-recall}}$ is monotone submodular, as was
shown in~\cite{lin2011-class-submod-sum}. That ROUGE-N is submodular
does not mean that it can be used to produce a summary, since it is
instantiated from human summaries. Rather it is used only to judge a
hypothesized summarization that has not had the opportunity to benefit
from the human reference summaries (if it did, it would be like
training on the test set).  On the other hand, that such a widely used
metric is submodular provides evidence of the naturalness of
submodularity for judging summary quality, further discussed below.
}

Another form of summarization that has more recently become popular in
the machine learning community is {\em data distillation}
~\cite{NIPS2005_4491777b, wang2020dataset,such2020generative,
  bohdal2020flexible, nguyen2020dataset, SucholutskyDistill2021,
  nguyen2021dataset} or equivalently {\em dataset
  condensation}~\cite{zhao2021dataset,ZhaoB21-0}.  With data distillation\footnote{Data distillation is distinct
  from the notion of {\em knowledge
    distillation}~\cite{hinton2015distilling,ba2014deep,buciluǎ2006model}
  or {\em model distillation}, where the ``knowledge'' contained in a
  large model is distilled or reduced down into a different smaller
  model.}, the goal is to produce a small set of synthetic pseudo-samples that
can be used, for example, to train a model.  The key here is that in
the reduced dataset, the samples are not compelled to be the same as,
or a subset of, the original dataset. \submodlongonly{This is akin to the difference
between the $k$-means algorithm (which is a form of data distillation)
and the $k$-medoids algorithm (which is a form of extractive
summarization or unweighted coreset). Data distillation therefore is
the same as abstractive summarization but on arbitrary data modalities
(e.g., images). Some references even consider sketchings to be data
projections~\cite{munteanu2018coresets,tremblay2019determinantal} and
thus data distillation --- for example, if we pre-multiply an
$n \times m$ design matrix $X$ by a $k \times n$ projection matrix
$P$, we get a resulting ``sketched'' $k \times n$ matrix $PX$ that can
be thought of as having $k$ linear distilled samples in the same
dimension-$m$ feature space as the original design matrix, and hence
we can view this as dimensionality reduction in data space.  The above
cited methods correspond to non-linear distillation procedures
produced by taking gradients, with respect to the data samples being
learnt, of an objective that involves a deep neural network, and
usually involve implicit
gradients~\cite{wang2020dataset,zhao2021dataset}. Sometimes the
objective is based on kernel ridge regression using a kernel obtained
from a neural network via the neural tangent
kernel~\cite{nguyen2020dataset}.  With any form of data distillation,
a large dataset is distilled down to a smaller synthetic or
pseudo-sample dataset.
}

All of the above should be contrasted with data {\em compression},
which in some sense is the most extreme data reduction method.
With compression, either lossless or lossy, one is no longer under any
obligation that the reduced form of the data
must be usable, or even recognizable,
by any algorithm or entity other than the decoder, or
uncompression, algorithm.
\submodlongonly{
The compressed data is thus unrecognizable,
looks like uniformly-at-random bits, and is unusable unless
decompressed back to its original form.  So, while the compressed form of
the data might be very small, the data is not useful in this form
until one expands it back to its original large size.  This is
different from summarization procedures which require that the data is
useful immediately in its summarized form without re-expansion. Unlike
compression, summarization procedures also have no requirement that
the original data, either exactly or approximately, must be
reconstructable back from the summary.  Compression can be slow for
larger datasets but might be faster than data distillation since
compression is not obliged to keep the data in a form that can be used
like the original.  Compression therefore is meant for other tasks,
specifically storage and communication.

\begin{figure}
\centering
\includegraphics[page=1,width=0.99\textwidth]{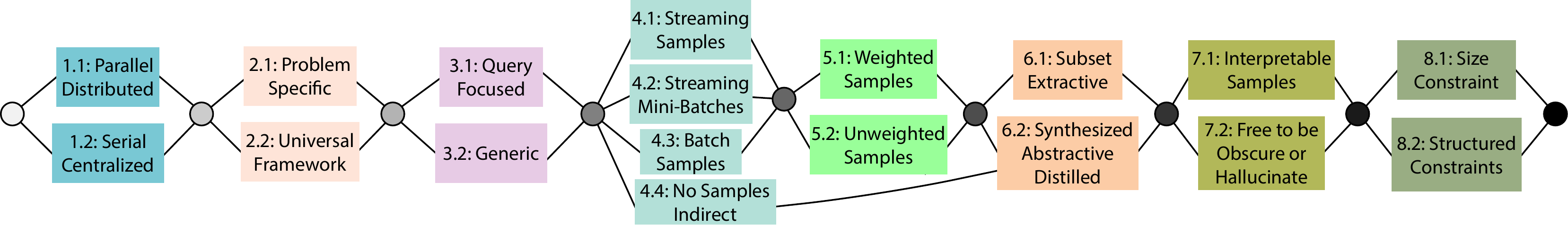}
\caption{A summary of the flavors of summarization. Any path from the
  left-most dot to the right-most dot selects the choices one
  may make when choosing or designing a summarization
  procedure. Regardless of the path, the purpose of the summary is to
  act as a concise purposeful representation of the original whole
  data.}
\label{fig:summarization_taxonomy}
\end{figure}

What then precisely is the difference between summarization, sketching,
coresets, and distillation? It is challenging to come up with a clear
distinction between the methods. This gets even more convoluted
because, for example, there are core set algorithms for the $k$-means
procedure ~\cite{bachem2018scalable} (i.e., how to reduce a dataset
so that running $k$-means on the reduced dataset produces the same
answer as when $k$-means is run on the entire dataset), but then the
resulting $k$-means set itself can be seen as a coreset --- running
$k$-means on the result of $k$-means will produce the same
answer. This is also true for $k$-center~\cite{badoiu2002approximate}
and other clustering procedures.
}

\subsubsection{Summarization Algorithm Design Choices}
\label{sec:summ-design-choic}

It is the author's contention that the notions of summarization,
coresets, sketching, and distillation are certainly analogous and
quite possibly synonymous, and they are all different from
compression.  The different names for summarization are simply
different nomenclatures for the same language game.  What matters is
not what you call it but the choices one makes when designing a
procedure for summarization.  And indeed, there are many choices.

\submodlongonly{
We
next summarize (or distill) the choices available to make when
considering or designing a summarization procedure and discuss their
uses, advantages, and disadvantages. These choices for summarization
are summarized in Figure~\ref{fig:summarization_taxonomy}.

\begin{itemize}

\item Parallelizable and/or distributed construction or centralized.
  This first choice has little to do with the type of summarization
  and instead what a given summarization algorithm offers in terms of
  its natural ability to be parallelized. Hence, we do not discuss
  this further other than to say that it is of course desirable to
  have algorithms that are easy to be efficiently parallelized.

\item Problem specific or a universal framework.  As mentioned
  above, many of the corset and sketching algorithms are developed for
  a specific geometric or algorithmic problem in mind and have
  guarantees associated only with that specific problem.  The
  algorithmic approaches can be very different from each other, and
  they also have very different approximations.  We might, however,
  conceive of a framework that is more universal, where the same
  algorithmic framework can be used to produce a coreset for a variety
  of different problems and where the algorithmic framework is
  parameterized in some way. Indeed, this is the strategy for
  submodular summarization as discussed further below, where the
  submodular function itself acts as a parameter to the summarization
  procedure which is done using the same algorithm.

\item Query specific or generic. Sometimes it is desirable for a
  summary to be generic, in an attempt to capture as much about a
  dataset as possible and be ready to answer any query that might be
  asked of the original. This can be useful when one does not know the
  purpose of a summary beforehand. On the other hand, a smaller, more
  direct, specific, and focused summary can be produced if a query is
  available at summarization time.  The classic example, as mentioned
  above, is web search, where we extractively summarize the web in
  response to a web search query; generic web summarization would not
  be very useful (e.g., a diverse top ten web page list is unlikely to
  solve many problems) but query-specific web summarization is the
  basis for billion-dollar businesses.

\item Streaming, Batch, or Mini-batch samples or indirect (no samples)
  input. If it is possible to load all of the data into memory
  simultaneously and be given random access to all samples, this
  affords an algorithm the greatest opportunity to produce a better
  summary. However, as mentioned above in the context of sketching,
  this is often infeasible for large datasets or when all the data is
  not available at a given time. A solution to this problem is to
  stream data sequentially. Often, however, the data need not be
  streamed one sample at a time, but rather one batch at a time, so an
  algorithm holds, at any point in time, both the summary and one (or
  perhaps two) batches of data.  While a batch is in memory, random
  access over that batch is granted, but once a batch is removed from
  memory is never again accessible, although samples within the batch
  might have been placed in the summary.

  A summary can also be produced not directly from the data, but
  instead via an intermediate representation that was computed from
  the data beforehand. That is, an indirect summarization means that
  the summary is produced by some structure other than the original
  samples.  For example, the summary might be produced from a model
  that has been trained on the data, but when the data is no longer
  available. One example of this achieves knowledge distillation via
  indirect data distillation using a pre-trained fixed glass-box
  teacher model~\cite{yin2020dreaming}, hence, this is a way of
  achieving knowledge distillation by first recovering samples
  directly from a model that had previously been trained on data, but
  the data is not available (most data distillation methods require the
  training data to still be available). This process is related to the
  privacy worry that exists about neural networks, i.e., that it is
  sometimes possible reconstruct the training data (or something that
  mimics and ``summarizes'' the training data) using access only to
  the model, or some property of the model such as its gradients
  (see~\cite{NEURIPS2019_60a6c400} and~\cite{huang2021evaluating} in a
  federated learning context). Another imaginary example would be to
  ask GPT-3 or any other modern foundation models to ``Please
  summarize the book War and Peace'' and it will produce some form of
  summary where the quality will depend on the extent the model has
  memorized the information from that book.

\item Weighted or unweighted samples. As mentioned above, a core
  feature of coresets is that the resulting summary consists of
  weighted samples indicating how important each sample is, but a
  summary might also take the form of unweighted samples. In the
  indirect summarization strategy, there are no samples, so this
  distinction does not apply.

\item Extractive subset or synthesized distilled abstractive.  This is
  one of the most interesting of choices for a summarization
  procedure and corresponds to the difference between abstractive
  vs.\ extractive document summarization and also the difference
  between data distillation vs.\ subset selection that we saw above.

  Extractive summarization and coresets have a number of important
  properties. For example, one is guaranteed that the data samples have
  the same form and style as the original data, and that all samples
  are natural, since the summary is a subset of the original. Hence,
  there can be no aberrational samples, in the same sense that a median
  is an original sample unlike a mean.  Since one is constrained to be
  a subset of the original, a summary of a given size has less
  opportunity to retain the information in the original as a
  distillation strategy. The reason is that it is possible that any
  pair of samples are partially redundant with respect to each other
  (this can be made precise using submodular functions where for any
  $v,v'$ we have $\max(f(v),f(v')) < f(v,v') < f(v) + f(v')$). On the
  other hand, extractive summarization has the greatest opportunity to
  be fast, since subset selection is computationally relatively easy
  (compared to abstraction) since the summary is always a subset of
  the original. We say relatively since, in the most general case,
  there are still an exponential, in $k$, number of size-$k$ subsets
  of a set of size $n$. This is again something submodularity helps with
  as seen below.

  In contrast, abstractive summarization (or data distillation, or
  sketching) have different properties than the above. Here, it is
  still guaranteed that the data has the same form as the original
  and, unlike compression, the summary can be used and viewed as if it
  was a subset of the original. There is some chance, however, that the
  synthesized samples are unlike any of the original, in the same
  sense that a mean might be unlike any of the original points and
  could live in a region of very low probability when the distribution
  is multimodal.  Hence, a synthetic sample has the chance of being
  unnatural and uninterpretable (this is discussed further
  below). Abstractive summarization offers the greatest chance to
  reduce the size of the summary while preserving information (since
  it is not constrained to have the summary be a subset). Thus,
  abstractive summarization is like a hybrid between
  extractive summarization (where the samples of the summary
  are all real) and compression (where the compressed form
  is not constrained to use a representation resembling
  samples at all). Like compression,
  abstract summarization is also
  computationally challenging. Even compression, however, can be computationally
  easier than abstractive summarization since compression is not
  constrained to have the compressed representation resemble samples
  in the original form. Therefore, abstract summarization (data
  distillation) is the most difficult, since we want (and have the
  opportunity) to make the smallest set possible while ensuring the
  summary is immediately useful in its summarized form.

  Lastly, we mention that subset selection and synthesized summaries
  are not required to be exclusive. It is conceivable to get the best
  of both worlds by producing a summary that consists of some samples
  from the original augmented by some additional samples that are
  synthesized. Indeed, in human-produced document summarizations,
  sometimes there is nothing better than to quote the original,
  other times it is best to paraphrase, and in many cases, it is best
  to do both.

\item {\bf Interpretable samples vs.\ anything goes.}
  Interpretability has become an important research topic in machine
  learning, where it is desirable to be able to interpret and
  understand the decisions made by a large complex AI system.  Another
  design choice for a summarization procedure is whether or not the
  summary should be interpretable, and this can take the form of the
  individual samples being interpretable and natural or the entire
  sample being collectively coherent. We first discuss the former and
  then the latter.

  As mentioned above, extractive summarization ensures that the
  summary consists of natural samples since the samples are all from
  the original. Abstractive summarization, however, might produce a
  summary where the synthetic pseudo-samples are unnatural,
  hallucinatory, or factually inconsistent. Regardless, the choice
  between interpretable vs.\ ``free'' is application dependent. For
  example, in the NLP domain, emphasis is placed on the synthesized
  samples being factual and natural since the typical end user of the
  summary is a human reader.  Effort must therefore be placed on
  ensuring that the distilled samples have this property, leading to
  the aforementioned increase in computation.

  For other applications, however, interpretability might or might not
  be less important. The data distillation examples shown
  in~\cite{wang2020dataset} look distorted and unnatural even though
  they train a model well. If a corresponding degree of distortion
  existed in an abstractive document summarization example, the
  results would be unacceptable to a human reader.  In computer
  vision, hallucinatory concepts in distilled data might mean a bundle
  of pixels are set to values that one would never find in a real
  image, but if it helps the training of a classifier, the distortion
  might not matter.  In abstractive document summarization, if a
  random grammatically correct but factually nonsensical phrase ``Dogs
  swim on the roof of Seattle'' appears, it would not be useful for a
  human-indented summary, but perhaps such hallucinatory phrases could
  be useful to train a text classification system or for a large language
  model. There have not, as of yet, been many studies on abstractive
  data distillation for the purposes of training a text classifier
  model save for~\cite{SucholutskyDistill2021} and a few others. The
  results of~\cite{zhao2021dataset,nguyen2020dataset} show fairly
  realistic looking synthetic samples, but even here we see some
  distortion that might not be acceptable in a human-consumed document
  summarization.  Therefore, the choice between interpretable vs.\
  ``free'' depends on the reason for doing summarization.

  Now, one might think that factual inconsistency and naturalness do
  not matter when producing a distilled dataset for the purpose of
  training another model. The distilled samples might not need to
  resemble any of the original samples since, when the goal is
  classification, it matters only that the dataset retains the key
  distinctive properties of each class relative to other classes. This
  is the same idea as ``structural discriminability'' as shown in
  Figures~5 and~6 of~\cite{bilmes2001discriminatively} which is
  analogous to Figure~3 in~\cite{wang2020dataset}, and Table T12
  in~\cite{zhao2021dataset}.  In~\cite{zhao2021dataset}, for example,
  a generated synthetic summary is compared between a GAN and their
  method and while it was found that their method did a better job in
  terms of training a new model, it was speculated that the reason was
  that the GAN's purpose was to generate realistic-looking images, but
  not to train an accurate model.

  One might reasonably say that, for classification purposes, only the
  structurally discriminative aspects of a dataset need to be captured
  in a distilled dataset if one wants the smallest possible
  summary. In other words, the smallest summary retains only minimum
  information necessary, and here that information means the distilled
  data retains only that which distinguishes different classes from
  each other. Anything irrelevant to achieving high classification
  accuracy is stripped away.

  On the other hand, as with document summarization, we still may wish
  to retain the factual consistency and naturalness that one
  automatically gets with extractive summarization. There are several
  reasons for this. Firstly, we may wish for the summary itself to be
  interpretable by humans, for the same reason we would wish any AI
  model to be interpretable. Having an interpretable summary can even
  improve interpretability of any subsequently trained model on that
  data. This may be particularly important when debugging the ML
  training process --- it becomes harder to debug depending on the
  degree to which the images in the synthesized dataset do not look
  natural. One of the properties of computer vision that makes it an
  attractive domain within which to perform machine learning research
  relative to other domains, such as biology, sensors (e.g., radar,
  lidar, or sonar), financial records, and so on, is that humans have
  cognitive capabilities that help them debug the results.  It is
  estimated that between 20\% and 30\% of the cortex in humans is
  devoted to visual processing~\cite{10.3389/fnint.2016.00037} and
  this can only help when debugging computer-vision based machine
  learning systems, especially with the synergy possible when the
  final training data looks realistic and natural.

  Secondly, and more importantly, there is a danger that the distilled
  summary contains only the non-robust features of the training data
  that are sufficient to obtain good accuracy on a test set not
  because the summary captures important information about the objects
  being classified but only because it captures aspects in the data
  set that are associated with those objects. An extreme example of
  this is shown in~\cite{Ilyas2019,Gilmer2019} where it is implied that the
  hallucinatory results in synthetic samples matter very little when
  wishing to train a model that performs well. This work shows that by
  producing a synthetic training dataset that uses only the
  ``non-robust'' features of the data, they can train other models
  that work well even though the samples look entirely wrong to a
  human. For example, synthesized images that appear as dogs may be
  labeled as cats, and synthesized images that look like cats might be
  labeled as dogs. These images have been modified in adversarial ways
  to look, to a model, like their real label. When those images become
  the training data, the resulting model works reasonably well. These
  images take advantage of the opportunistic tendency when training
  large neural networks to associate the easiest signal in the input
  as possible with the correct label. There is much research into the
  lack of robustness that some models have since, during training,
  they associate wrong cues in the input with output
  labels~\cite{Gilmer2019}.  While such work does not have abstractive
  summarization or distilled data as its goal, it is very
  relevant. The reason is that it shows it is possible to produce
  synthetic samples that look very different from natural samples but
  that still can be used to train a model.

  This poses a danger for data distillation as well since there is no
  reason, moreover, that distilled data is immune to these effects
  unless special provisions are made.  In some sense, adversarial
  samples that look identical to real samples but that fool a model
  into making a wrong decision are the dual of synthetic abstractive
  factually-inconsistent samples that fool a model into learning a
  correct decision.  Of course, such factually-inconsistent images do
  not represent the original samples since they are adversarially
  produced to take advantage of a model training’s tendency to focus
  on the wrong part of an image when learning to associate with a
  label.  Hence, the goal in NLP's abstractive summarization to avoid
  such hallucinatory artifacts is a good one as well for any machine
  learning task when one wishes to summarize, where the summary is
  supposed to represent the original.

  It was mentioned above that for NLP's abstractive summary, the set
  of sentences ideally can be ordered to make a coherent text with
  long-range interactions. Another form of interpretability of a
  summary is to consider the gestalt of the entire summary itself.
  One can imagine that either extractive or abstractive summarization
  has the chance to produce a coherent summary on the whole, but
  abstractive summarization being less constrained than extractive
  summarization has more opportunity to do so.

  To summarize the past two choices (i.e.,
  Figure~\ref{fig:summarization_taxonomy}-6.x and -7.x), if one wants
  to guarantee that the samples are natural, extractive summarization
  (subset selection) is easiest. If computation is at a premium or if
  the data set sizes are massive, then extractive summarization is
  best. If one wants the absolute smallest possible summary for a
  given amount of information, and one is not worried about
  interpretability of the samples and allows for structurally
  discriminative samples, then data distillation is best since there
  is more opportunity for compression.  If one wants to ensure the
  samples produce a summary that constitutes a coherent whole, and if
  computation is not the primary concern, then abstractive
  summarization is best.

\item Constraint type (size or other structure). We almost always want
  a summary to be small, and this can be seen as a constraint on the
  cardinality of the summary. In other cases, we may wish the summary
  to satisfy other constraints as well.  For example, the summary
  might comprise a tree or a sequence in some space, or the sum of the
  costs of each of the summary items might need to be below some
  budget (e.g., as in a knapsack constraint), or we might wish the
  summary to collectively be coherent so that summary elements can be
  organized into a good presentation (as in abstractive document
  summarization, where the summary sentences together must be able to
  produce a good text). In all of these cases, we can view this as an
  additional constraint on or quality required of the resulting
  summary.

\end{itemize}

A distinction not included in Figure~\ref{fig:summarization_taxonomy}
is that of supervised vs.\ unsupervised. This means, when each element
in a set to be summarized is a data sample, if the summarization
procedure uses a feature-label $(x,y)$ pair, or if it uses just the
features $(x)$.  This distinction is specific to the problem of
summarizing a training set for the purposes of classification or
regression, and not a more general summarization goal. Another
distinction not given in the figure is that of update summarization,
but if it were added it could be located at a hypothetical
Fig.~\ref{fig:summarization_taxonomy}-3.3, but this is not ideal. The
reason is that update summarization can be seen as a form of
``anti-query'' focused summarization since the summarization in this
case should, if anything, neglect information relevant to what is
already known, and the summary should be focused only on what is
new. It is also the case that one might want both to focus on a
particular query and neglect a particular anti-query, so therefore it
would be better to consider update-summarization as part of
Fig.~\ref{fig:summarization_taxonomy}-3.3.

How is submodularity situated within this sea of summarization
strategies shown in Figure~\ref{fig:summarization_taxonomy}? In fact,
submodular summarization procedures touch all options except for
Fig.~\ref{fig:summarization_taxonomy}-4.4 (indirect summarization),
Fig.~\ref{fig:summarization_taxonomy}-6.2 (abstractive summarization),
and Fig.~\ref{fig:summarization_taxonomy}-7.2 (obscure samples). The
reason is that a submodular function is a set function, and producing
a summary (or coreset, or sketch) of a ground set $\submodgroundset$
corresponds to producing a subset of that ground set. Hence, there is
no immediate way for a submodular function to bypass the data (4.4),
produce an abstractive summary (6.2), or to suffer from obscure and/or
hallucinatory samples (7.2) (although there is a chance to select
outliers).

It is however odd that in much of the sketching and coreset
literature, the connections to a submodular summary does not more
frequently arise. In some cases, when it comes up, it is not for the
purposes of using submodularity to create a sketch, but rather to
sketch a particular submodular function
itself~\cite{goemans2009approximating,balcan2011learning,balcan2012learning,badanidiyuru2012sketching,10.1145/3039871,yaroslavtsev_et_al:LIPIcs:2019:11284}.
This is essentially the problem of how to approximate a given 
submodular function using a relatively small number of oracle function
queries.
}

\submodlongonly{On the other hand, s}\submodshortonly{S}ubmodularity offers essentially
an infinite number
of ways to perform data sketching and coresets.  When we view the
submodular function as an information function (as we discussed in
Section~\ref{sec:comb-inform-funct}), where $\submodfun(\submodsetX)$
is the information contained in set $\submodsetX$ and
$\submodfun(\submodgroundset)$ is the maximum available information,
finding the small $\submodsetX$ that maximizes
$\submodfun(\submodsetX)$ (i.e.,
$\submodsetX^* \in \argmax \{ \submodfun(\submodsetX) : |\submodsetX|
\leq \submodcardcnstr \}$), is a form of coreset computation that is
parameterized by the function $\submodfun$ which has
$2^\submodgroundsetsize$ parameters since $\submodfun$ lives in a
$2^\submodgroundsetsize$-dimensional cone.  Performing this
maximization will then minimize the residual information
$\submodfun(\submodgroundset \setminus \submodsetX | \submodsetX)$
about anything not present the summary
$\submodgroundset \setminus \submodsetX$ since
$\submodfun(\submodgroundset) = \submodfun(\submodsetX \cup
\submodgroundset \setminus \submodsetX) = \submodfun(\submodgroundset
\setminus \submodsetX | \submodsetX) + \submodfun(\submodsetX)$ so
maximizing $\submodfun(\submodsetX)$ will minimize
$\submodfun(\submodgroundset \setminus \submodsetX | \submodsetX)$.
For every $\submodfun$, moreover, the same algorithm (e.g., the greedy
algorithm) can be used to produce the summarization, and in every case,
there is an approximation guarantee relative to the current
$\submodfun$, as mentioned in earlier sections, as long as
$\submodfun$ stays submodular. Hence, submodularity provides a
universal framework for summarization, coresets, and sketches
\submodlongonly{(Fig.~\ref{fig:summarization_taxonomy}-2.2)} to the extent that the
space of submodular functions itself is sufficiently diverse and spans
over different coreset problems. \submodlongonly{This is promising since even within
one class of submodular function, e.g., with just the facility
location function, there is enormous flexibility and expressivity in
terms of what similarity measure to use and how to parameterize that
similarity measure.
But the facility location, as mentioned in
Section~\ref{sec:example-subm-funct} is only one of many submodular
functions.}

\submodlongonly{
Ordinarily, one would expect a summary to necessarily lose fidelity
relative to an original larger dataset. On the other hand, with a
submodular perspective it is possible for the summary to lose no
information relative to the original.  For example, assuming the
submodular objective is the correct summary objective, then given two
sets $A \subset B$ with $f(A) = f(B)$, there is no reason to prefer
$B$ over $A$ since $B$ is partially redundant.  That is,
$f(B|A) = f(A+B) - f(A) = f(B) - f(A) = 0$, the additional samples in
$B$ not in $A$ are redundant relative to those already in $A$. Also, a
summary can be preferable to the whole due to bias removal. If the
whole $V$ is biased or imbalanced (i.e., if there are
over-represented majorities and under-represented minorities
within the whole), a summary can re-balance the data so that concepts
have a more equal and balanced representation.  For example, one might
find the value per item $f(V)/|V|$ of the whole to be smaller than the
value per item in the summary $f(A)/|A|$. This suggests algorithms to
optimize the ratio of submodular functions, as has been addressed
in~\cite{bai-ratio-submod-icml-2016}.  For these reasons, we sometimes
see machine learning performance increase when trained on a subset
than when trained on the whole
albeit it is usually the case that the subset in such cases is about
80\% to 90\% of, i.e., not that much smaller than, the whole.
}

Overall, the corset or sketching problem, when using submodular functions,
therefore becomes a problem of ``submodular design.'' That is, how do
we construct a submodular function that, for a particular problem, acts
as a good coreset producer when the function is maximized. There are
three general approaches to produce an $\submodfun$ that works well as
a summarization objective: (1) a pragmatic approach where the function
is constructed by hand and heuristics, (2) a learning approach where
all or part of the submodular function is inferred from an optimization
procedure, and (3) a mathematical approach where a given submodular
function when optimized offers of a coreset property.

\submodlongonly{
\subsubsection{Crafting a Submodular Function By Hand}
\label{sec:craft-subm-funct}
}

When the primary goal is a practical and scalable algorithm that can
produce an extractive summary that works well on a variety of different
data types, and if one is comfortable with heuristics that work well
in practice, a good option is to specify a submodular function by
hand. For example, given a similarity matrix, it is easy to
instantiate a facility location function and maximize it to produce a
summary. If there are multiple similarity matrices, one can construct
multiple facility location functions and maximize their convex
combination. \submodlongonly{One can utilize feature-based functions to ensure that
the summary is diverse with respect to a set of qualities represented
by the features or combine such a feature-based function with the
above mixture of facility location functions to produce a richer
summarization objective. It is also easy to mix in a modular function
when one wishes to ensure that the summary is relevant to a particular
query or that has a particular quality.  That is, one might take a
combination of a diversity component (which is a mixture of submodular
functions to appropriately capture diversity) along with a relevance
modular component (to ensure the summary is relevant to a query).  As
an example, we might be interested in image summarization, but
preference should be placed on those images with human faces.  A
modular function $\submodmodfun$ can be used to express the
probability that an image contains a face, and
$\submodmodfun(\submodsetX)$ corresponds to the overall face quality
in set $\submodsetX$. Maximizing this alone, however, would lead to a
redundancy since $\submodmodfun$ is only a modular function --- e.g.,
consider an image with a high face score that is duplicated and thus
selected multiple times.  But maximizing
$\lambda\submodmodfun(\submodsetX) +
(1-\lambda)\submodfun(\submodsetX)$, where $\submodfun(\submodsetX)$
is a mixture of facility location functions and thus indicates the
diversity in set $\submodsetX$ will produce a summary consisting of a
diverse set of faces, assuming $\lambda$ is set appropriately.

Such a strategy for producing a summary might involve a number of
hyperparameters that need to be tuned, $\lambda$ being only one of
them. Hyperparameter tuning would need to occur in the same way that
hyperparameter are always tuned, i.e., via a search strategy that
evaluates performance on a validation set, and there are many search
strategies available (e.g., genetic algorithms, Bayesian methods, grid
search with early pruning, and so on). It should be clear that even
after tuning, and while the resulting function is submodular,
performing submodular maximization on the summarization objective
offers a mathematical guarantee regarding the quality of the resulting
set relative to the objective, and not necessarily relative to the
original problem (i.e., the true quality of a summary as judged, say,
by a human or a machine learning model). Hence, the hyperparameter
tuning must be done appropriately.

Still, this}\submodshortonly{Such an approach}
is viable and practical and has been used
successfully many times in the past for producing good summaries.  One
of the earliest examples of this is the algorithm presented
in~\cite{kempe2003maximizing} that shows how a submodular model can be used to
select the most influential nodes in a social network.  Perhaps the
earliest example of this approach used for data subset selection for
machine learning is~\cite{lin2009-submod-active-seq} which utilizes a
submodular facility location function based on Fisher kernels
(gradients w.r.t.\ parameters of log probabilities) and applies it to
unsupervised speech selection to reduce transcription costs.  Other
examples of this approach includes: \cite{lin2010-submod-sum-nlp,
  lin2011-class-submod-sum} which developed submodular functions for
query-focused document summarization; \cite{kirchhoff2014-submodmt}
which computes a subset of training data in the context of
transductive learning in a statistical machine translation system;
\cite{lin2010-pp-corpus-creation,wei-duc-sum-for-speech-2013,wei2014-unsupervised-icassp}
which develops submodular functions for speech data subset selection
(the former, incidentally, is the first use of a deep submodular
function and the latter does this in an unsupervised label-free
fashion); \cite{sener2018active} which is a form of robust
submodularity for producing coresets for training CNNs;
\cite{kaushal2019learning} which uses a facility location to facilitate
diversity selection in active learning;
\cite{bairi2015summarization,chali2017towards} which develops a
mixture of submodular functions for document summarization where the
mixture coefficients are also included in the hyperparameter set;
\cite{xu2015gaze} uses a symmetrized submodular function for the
purposes of video summarization\submodlongonly{ (see below)}.

\submodlongonly{

We note that some of the above methods are generic
(Fig.~\ref{fig:summarization_taxonomy}-3.2) while others are query
focused (Fig.~\ref{fig:summarization_taxonomy}-3.1) and some are
supervised while others are unsupervised.

One last thought on the by-hand submodular crafting approach.  Like
any empirical endeavor, such as most deep learning research, when
developing a submodular function for the purpose of summarization, it
is important to establish careful, competitive, and state-of-the-art
baselines to ensure that the submodular function being utilized is
working as well as possible. Metrics to compare include both the
summarization quality but also the underlying computational effort
needed to perform the summarization.  }

\submodlongonly{
\subsubsection{Learning a Submodular Function}
\label{sec:learn-subm-funct}
}

The learnability and identifiability of submodular functions has
received a good amount of study from a theoretical perspective. \submodlongonly{The
fundamental question in this setting is how the quality of the
resulting learnt function is judged. Moving from most challenging to most
accommodating, this includes how should the function being learnt
approximate the true submodular function, should the learning take
place in certain subfamilies of submodular functions, and if all
aspects of a target submodular function itself should be learnt or if
it is acceptable to learn only some aspects (e.g., sets with large
value).

}Starting with the strictest learning settings, the problem looks
pretty dire. For
example,~\cite{svitkina2008submodular,goemans2009approximating} shows
that if one is restricted to making a polynomial number of queries
(i.e., training pairs of the form $(S,f(S))$) of a monotone submodular
function, then it is not possible to approximate $f$ with a multiplicative
approximation factor better than $\tilde \Omega(\sqrt{n})$.
In~\cite{balcan2011learning}, goodness is judged multiplicatively,
meaning for a set $A \subseteq V$ we wish that
$\tilde f (A) \leq f(A) \leq g(n)f(A)$ for some function $g(n)$, and
this is typically a probabilistic condition (i.e., measured by
distribution, or
$\tilde f (A) \leq f(A) \leq g(n)f(A)$, should happen on a fraction at
least $1-\beta$ of the points). Alternatively, goodness may also be
measured by an additive approximation error, say by a norm. I.e.,
defining
$\text{err}_p(f,\tilde f) =\| f - \tilde f \|_p = ( E_{A \sim
  \textbf{Pr}} [ {| f(A) - \tilde f(A) |}^p ] )^{1/p}$, we may wish
$\text{err}_p(f,\tilde f) < \epsilon$ for $p =1$ or $p=2$.  In the PAC
(probably approximately correct) model, we probably ($\delta > 0$)
approximately ($\epsilon > 0$ or $g(n) > 1$) learn ($\beta = 0$) with
a sample or algorithmic complexity that depends on $\delta$ and
$g(n)$.  In the PMAC (probably mostly approximately correct)
model~\cite{balcan2011learning}, we also ``mostly'' $\beta > 0$ learn.
In some cases, we wish to learn the best submodular approximation to a
non-submodular function. In other cases, we are allowed to deviate
from submodularity as long as the error is small.  Learning special
cases includes coverage
functions~\cite{pmlr-v35-feldman14a,DBLP:journals/corr/abs-1304-2079},
and low-degree polynomials~\cite{7354435}, curvature limited
functions~\cite{iyer2013-curvature}, functions with a limited
``goal''~\cite{deshpande2014approximation,BACH20181}, functions that
are Fourier sparse~\cite{wendler2020learning}, or that are of a family
called ``juntas''~\cite{feldman2016optimal}, or that come from
families other than submodular~\cite{dozinski2021gross}, and still
others~\cite{10.1145/3055399.3055406,feldman2014nearly,feldman2017tight,feldman2020tight,pmlr-v30-Feldman13,yaroslavtsev_et_al:LIPIcs:2019:11284}.
Other results include that one cannot minimize a submodular function
by learning it first from samples~\cite{balkanski2017minimizing}.  The
essential strategy of learning is to attempt to construct a submodular
function approximation $\hat f$ from an underlying submodular function
$f$ querying the latter only a small number of times. The overall gist
of these results is that it is hard to learn everywhere and
accurately.

In the machine learning community, learning can be performed extremely
efficiently in practice, although there are not the types of
guarantees as one finds above.  For example, given a mixture of
submodular components of the form $f(A) = \sum_i \alpha_i f_i(A)$, if
each $f_i$ is considered fixed, then the learning occurs only over the
mixture coefficients $\alpha_i$. This can be solved as a linear
regression problem where the optimal coefficients can be computed in a
linear regression setting. Alternatively, such functions can be learnt
in a max-margin setting where the goal is primarily to adjust
$\alpha_i$ to ensure that $f(A)$ is large on certain
subsets~\cite{sipos2012large,hui2012-submodular-shells-summarization,sebastian2014-submod-image-sum}.
Even here there are practical challenges, however, since it is in
general hard in practice to obtain a training set of pairs
$\{(S_i,F(S_i))\}_i$. Alternatively, one also ``learn'' a submodular
function in a reinforcement learning
setting~\cite{chen2017interactive} by optimizing the implicit function
directly from gain vectors queried from an environment. In general,
such practical learning algorithms have been used for image
summarization~\cite{sebastian2014-submod-image-sum}, document
summarization~\cite{hui2012-submodular-shells-summarization}, and
video
summarization~\cite{Gygli2015VideoSB,vasudevan2017query,gong2014diverse,sharghi2016query,sharghi2017query}.
While none of these learning approaches claim to approximate some true
underlying submodular function, in practice, they do perform better
than the by-hand crafting of a submodular function mentioned above.

\submodlongonly{
\subsubsection{Submodularity Based CoreSets}
\label{sec:subm-based-cores}
}

By a submodularity based coreset, we mean one where the direct
optimization of a submodular function offers a theoretical guarantee
for some specific problem. This is distinct from
\submodlongonly{those cases in
  Sections~\ref{sec:craft-subm-funct} and~\ref{sec:learn-subm-funct}}
\submodshortonly{above}
where the submodular function is used as a surrogate heuristic
objective function and for which, even if the submodular function is
learnt, optimizing it is only a heuristic for the original problem.
\submodlongonly{Here, the goal is to find a submodular function so that, under its
maximization using, say, the greedy algorithm, we can show an
approximation for some underlying problem.

}In some limited cases, it can be shown that the function we wish to
approximate is already submodular, e.g., in the case of certain naive
Bayes and k-NN classifiers~\cite{wei2015-submodular-data-active} where
the training accuracy, as a function of the training data subset, can
be shown to be submodular. Hence, maximizing this function offers the
same guarantee on the training accuracy as it does on the submodular
function. Unfortunately, the accuracy function for many models is not
submodular, although they do have a difference of
submodular~\cite{narasimhan2005-subsup,iyer2012-min-diff-sub-funcs}
decomposition.

In other cases, it can be shown that certain desirable coreset
objectives are inherently submodular. For example,
in~\cite{mirzasoleiman2020coresets}, it is shown that the normed
difference between the overall gradient (from summing over all samples
in the training data) and an approximate gradient (from summing over
only samples in a summary) can be upper bounded with a supermodular
function that, when converted to a submodular facility location
function and maximized, will select a set that reduces this
difference, and will lead to similar convergence rates to an
approximate optimum solution in the convex case. A similar example of
this in a DPP context is shown in~\cite{tremblay2019determinantal}.
In other cases, subsets of the training data and training occur
simultaneously using a continuous-discrete optimization framework,
where the goal is to minimize the loss on diverse and challenging
samples measured by a submodular objective~\cite{zhou2018minimax}.
In still other cases, bi-level objectives related to but not
guaranteed to be submodular can be formed where a set is selected from
a training set with the deliberate purpose of doing well on a
validation
set~\cite{killamsetty2020glister,borsos2020coresets}. \submodlongonly{Overall, this
area of research is quite nascent as of this writing and it is
believed that there will be many more coreset-producing submodular
showcases in the future.}

\submodlongonly{
\subsubsection{Advantages to Submodularity as a CoreSet Paradigm}

Regardless of how a submodular is produced as an objective to produce
a coreset, there are a number of potential advantages and a few
disadvantages.

First, this approach allows us to focus on the parameter of the
coreset producer, namely the submodular objective $\submodfun$, rather
than the algorithm stays fixed throughout the search. Regardless of
the objective, therefore, one maintains essentially the same
computational cost for a given coreset problem (assuming that
evaluating the candidate objectives also stays the same). Moreover,
once we have found the right $\submodfun$, we immediately have a
guarantee associated with the guarantee provided by a submodular
optimization algorithm.

Secondly, as seen above, there are many ways to partially parameterize
a submodular function in fewer dimensions than the $2^n$ dimensions
associated with the submodular cone.  Submodularity is closed under
convex mixtures, so a coreset objective can be formed as
$f(A) = \sum_i \alpha_i f_i(A)$ a weighted sum of specific objectives.
Alternatively, robust versions can be constructed of the form
$f(A) = \min_i f_i(A)$ --- this later case is not submodular but there
are fast algorithms for maximizing the minimum of a set of submodular
functions available~\cite{krause2008robust,tzoumas2017resilient,orlin2018robust}.
In either case, it also is easy to add modular functions to the mix
to generate individual item costs as well, all while staying
within the same paradigm.

Since a primary algorithm for maximizing a submodular function is
greedy, then prefixes of summaries are still summaries.  This
contrasts with, for example, data distillation methods, when one
wishes a summary of size $k$, one must optimize a summary of size $k$
--- as mentioned in~\cite{nguyen2020dataset,nguyen2021dataset},
selecting a summary of size smaller than $k$ from a distilled data set
of size $k$ is suboptimal. In the submodular approach, however, the
greedy algorithm gives an ordering $s_1, s_2, \dots$ so that a summary
if size $k$ is $S_k$ and is a prefix of the ordering. Hence if you
already have a summary of size $k$ and kept track of the order while
running the algorithm, a summary of size $k’$ for all $k’ < k$ is
immediately available.

Another advantage of the submodular approach is that there are many
types of constraints besides a cardinality constraint that the
resulting summary might need to satisfy and there are simple
modifications to the objective that result in update summarization.
Regarding other constraints, this includes matroid (where the summary
must be an independent of a matroid), intersection of matroids (where
the summary must be independent in multiple matroids), knapsack (where
the summary must not have a cost that exceeds a given budget), and
combinatorial constraints (where the summary might need to be a tree,
a cycle, a cut, or some other combinatorial structure), or
combinations thereof. For all of these types of constraints there are
algorithms that usually have theoretical guarantees and often work
quite well in practice. Regarding objective modifications, defining a
function $g(A) = f(A|B)$ for a given fixed set $B$ preserves
submodularity. If $g(A)$ is then maximized, the result are those items
that are both diverse but not already explained by $B$, which is
precisely the goal of update summarization.  For example, suppose $B$
is a summary of newspaper articles that have already been read, and we
utilize $g(A)$ as the objective over a set of new articles. The result
will be a subset of articles that are both diverse and new relative
those already read articles $B$.  Such an objective modification is
not unrelated to the combinatorial mutual information described
further below.

Several other advantages include that it is relatively easy to produce
random diverse samples from any submodular function, very much like a
DPP. For example, there are methods such as submodular point
processes~\cite{iyer2015-spps} and log-submodular probability
distributions~\cite{djolonga2014map,djolonga2015scalable,djolonga2018provable},
which include DPPs~\cite{kulesza2012determinantal}. Sampling from such
distributions produces subsets that are more likely diverse than
otherwise.  This is often known as negative dependence where similar
items have a lower probability of being seen within the same sample
than dissimilar items.

Of course, another important use for a summary is to reduce the costs
and time associated with data labeling endeavors. For example, given a
large unlabeled and redundant dataset, it is inefficient to label all
of it. Labeling a representative summary and training on that subset,
or using it in a semi-supervised learning procedure, reduces the cost
and time of data labeling, not to mention that the act of labeling a
diverse summary can be less susceptible to the act of labeling
redundant data which can be tiresome. It is just as easy to
instantiate a submodular function over unlabeled data as it is to do
so over labeled data, as has been done
in~\cite{wei2014-unsupervised-icassp} and which can be seen as a form
of single seed-stage or batch active
learning~\cite{wei2015-submodular-data-active}. More formal
relationships between submodular optimization and batch active
learning can be found
in~\cite{guillory2011-active-semisupervised-submodular} where it is
shown that a form of submodular optimization selects a set that
directly reduces an upper bound on the training error in a
semi-supervised context.

}

\submodlongonly{
\subsubsection{Feature Selection}
\label{sec:feature-selection}
}

The methods above have focused on reducing the number of
samples in a training dataset. Considering the transpose of a design
matrix, however, all of the above methods can be used for reducing the
features of a machine learning procedure as well. Specifically, any of
the extractive summarization, subset selection, or coreset methods can
be seen as feature selection while any of the abstract summarization,
sketching, or distillation approaches can be seen as dimensionality
reduction. \submodlongonly{Conversely, any strategy used for feature selection could,
on principle, be used for data subset selection and any strategy used
for dimensionality reduction (such as PCA, LDA, factor analysis, or
even visualization methods such as tSNE or UMAP, or modified versions
thereof due to the different relative sizes of $m$ vs $n$ in the
design matrix) could be used for distillation.}

\subsection{Combinatorial Information Functions}
\label{sec:comb-inform-funct}

The entropy function over a set of random variables
$X_1, X_2, \dots, X_n$ is defined as
$H(X_1, X_2, \dots, X_n) = - \sum_{x_1, x_2, \dots, x_n} p(x_1, \dots,
x_n) \log p(x_1, \dots, x_n)$.  From this we can define three
set-argument conditional mutual information functions as
$I_H(A;B|C) = I(X_A; X_B | X_C)$ where the latter is the mutual
information between variables indexed by $A$ and $B$ given variables
indexed by $C$. This mutual information expresses the residual
information between $X_A$ and $X_B$ that is not explained by their
common information with $X_C$.

As mentioned above, we may view any polymatroid function as a type of
information function over subsets of $V$.  That is, $f(A)$ is the
information in set $A$ --- to the extent that this is true, this
property justifies $f$'s use as a summarization objective as mentioned
above. The reason $f$ may be viewed as an information function stems
from $f$ being normalized, $f$'s non-negativity, $f$'s monotonicity,
and the property that further conditioning reduces valuation (i.e.,
$f(A|B) \geq f(A|B,C)$ which is identical to the submodularity
property). These properties were deemed as essential to the
entropy function in Shannon's original work~\cite{shannon1948} but are true of
any polymatroid function as well.  Hence, given any polymatroid
function $f$, is it possible to define a combinatorial mutual information
function~\cite{iyer2021generalized} in a similar way. Specifically, we
can define the combinatorial (submodular) conditional mutual
information (CCMI) as
$I_f(A;B|C) = f(A + C) + f(B + C) - f(C) - f(A + B + C)$, which has
been known as the connectivity
function~\cite{cunningham1983decomposition} amongst other names.  If
$f$ is the entropy function, then this yields the standard entropic
mutual information but here the mutual information can be defined for
any submodular information measure $f$. For an arbitrary polymatroid
$f$, therefore, $I_f(A;B|C)$ can be seen as an $A,B$ set-pair
similarity score that ignores, neglects, or discounts any common
similarity between the $A,B$ pair that is due to $C$.

Historical use of a special case of CCMI, i.e., $I_f(A;B)$ where
$C = \emptyset$, occurred in a number of circumstances. For example,
in~\cite{guestrin2005near} the function $g(A) = I_f(A;V\setminus A)$ (which,
incidentally, is both symmetric ($g(A) = g(V\setminus A)$ for all $A$) and submodular
was optimized using the
greedy procedure which has a guarantee as long as $g(A)$ is monotone
up $2k$ elements whenever one wishes for a summary of size $k$. This
was done for $f$ being the entropy function, but it can be used for
any polymatroid function. In similar work where $f$ is Shannon
entropy,~\cite{krause2005near} demonstrated that $g_C(A) = I_f(A;C)$
for a fixed set $C$ is not submodular in $A$ but if it is the case
that the elements of $V$ are independent given $C$ then submodularity
is preserved. This can be seen quickly easily by the consequence of
the assumption which states that
$I_f(A;C) = f(A) - f(A|C) = f(A) - \sum_{a \in A} f(a|C)$ where the
second equality is due to the conditional independence property. In
this case, $I_f$ is the difference between a submodular and a modular
function which preserves submodularity for any polymatroid $f$.

On the other hand, it would be useful for $g_{B,C}(A) = I_f(A;B|C)$,
where $B$ and $C$ are fixed, to be possible to optimize in terms of
$A$. One can view this function as one that, when it is maximized,
chooses $A$ to be similar to $B$ in a way that neglects or discounts
any common similarity that $A$ and $B$ have with $C$.  One option to
optimize this function to utilize difference of
submodular~\cite{narasimhan2005-subsup,iyer2012-min-diff-sub-funcs}
optimization as mentioned earlier. A more recent result shows that in
some cases $g_{B,C}(A)$ is still submodular in $A$. Define the second
order partial derivative of a submodular function $f$ as follows
$f(i,j|S) \triangleq f(j|S+i) - f(j|S)$. Then if it is the case that
$f(i,j|S)$ is monotone non-decreasing in $S$ for
$S \subseteq V \setminus \{ i,j \}$ then $I_f(A;B|C)$ is submodular in
$A$ for fixed $B$ and $C$. It may be thought that only esoteric
functions have this property but in fact~\cite{iyer2021generalized}
shows that this is true for a number of widely used submodular
functions in practice, including the facility location function which
results in the form
$I_\submodfun(A;B|C) = \sum_{v \in V} \max\Bigl( \min\bigl( \sum_{a
  \in A} \text{sim}(v,a), \max_{b \in B} \text{sim}(v,b) \bigr) -
\max_{c \in C} \text{sim}(v,c), 0\Bigr)$. This function was
used~\cite{kothawade2021prismaaai} to produce summaries $A$ that were
particularly relevant to a query given by $B$ but that should neglect
information in $C$ that can be considered ``private'' information
to avoid.

\submodlongonly{
Entropic quantities from the field of information
theory such as entropy mutual information, conditional mutual
information, and the Kullback-Leibler divergence have had an enormous
impact on the field of machine learning.  Each one of these measures
involves at least one joint probability distribution over the set of
random variables that are being measured. The submodular combinatorial
measures, however, are not based on an underlying probability
distribution and hence, unlike the entropic queries, are
computationally more tractable. That is, even a single entropic query
requires an exponential cost to exactly compute since the entropy is
sum of an exponential, in the number of random variables, number of
terms. While Gaussian entropy is a special case and can be computed
relatively easily since it is only a log-determinant, there are many
other submodular functions as we've seen that are easier to evaluate
than even log-determinants.
}

\subsection{Clustering, Data Partitioning, and Parallel Machine Learning}
\label{sec:clust-data-part}

There are an almost limited number of clustering algorithms and a
plethora of reviews on their variants. Any given submodular function
can also instantiate a clustering procedure as well, and there are
several ways to do this. Here we offer only a brief outline of the
approach. In the last section, we defined $I_f(A;V\setminus A)$ as the
CCMI between $A$ and everything but $A$. When we view this as a
function of $A$, then $g(A) = I_f(A;V\setminus A)$ and $g(A)$ is a
symmetric submodular function that can be minimized using Queyranne's
algorithm~\cite{queyranne98,nagamochi1992computing}. Once this is
done, the resulting $A$ is such that it is least similar to
$V\setminus A$, according to $I_f(A;V\setminus A)$ and hence forms a
2-clustering. This process can then be recursively applied where we
form two new functions $g_A(B) = I_f(B;A\setminus B)$ for
$B \subseteq A$ and
$g_{V \setminus A}(B) = I_f(B; (V \setminus A) \setminus B)$ for
$B \subseteq V \setminus A$. These are two symmetric submodular
functions on different ground sets that also can be minimized using
Queyranne's algorithm. This recursive bisection algorithm then repeats
until the desired number of clusters is formed. Hence, the CCMI
function can be used as a top-down recursive bisection clustering
procedure and has been called
Q-clustering~\cite{narasimhan2005q,bilmes2006-dbntri-tr}. It should be
noted that such forms of clustering often generalize forming a
multiway cut in an undirected graph in which case the objective
becomes the graph-cut function that, as we saw above, is also
submodular. In some cases, the number of clusters need not be
specified in advance~\cite{nagano2010minimum}.
Another submodular approach to clustering \submodshortonly{can be
  found in~\cite{kaiwei2015nips_submod_partitioning}}
\submodlongonly{has already been seen in
  Equation~\ref{prob:minmax-genmix}} where the goal is to minimize the
maximum valued block in a partitioning which can lead to submodular
load balancing or minimum makespan
scheduling~\cite{hochbaum1988polynomial,lenstra1990approximation}\submodlongonly{ as
we have seen}.  \submodlongonly{Also, in this light submodular fair allocation or
submodular welfare can themselves be seen as a form of
anti-clustering, where each block of the partition should be diverse
rather than homogeneous.}

Yet another form of clustering can be seen via the simple cardinality
constrained submodular maximization process itself which can be
compared to a $k$-medoids process whenever the objective $f$ is the
facility location function. Hence, any such submodular function can be
seen as a submodular-function-parameterized form of finding the $k$
``centers'' among a set of data items. There have been numerous
applications of submodular clustering. For example, using these
techniques it is possible to identify parcellations of the human
brain~\cite{10.1007/978-3-319-66182-7_55}.  Other applications include
partitioning data for more effective and accurate and lower variance
distributed machine learning
training~\cite{kaiwei2015nipsparallelworkshop} and also for more ideal
mini-batch construction for training deep neural
networks~\cite{pmlr-v89-wang19e}.

\subsection{Active and Semi-Supervised Learning}
\label{sec:curr-active-semi}

\submodlongonly{
Active
learning~\cite{atlas1990training,cohn1994improving,Cohn96,settles2009active}
is a learning setting where it is acknowledged that acquiring the
labels for samples is a more challenging prospect than acquiring just
the unlabeled samples which can be found in abundance. Rather than
impulsively and passively labeling all of your data, active learning
algorithms involve iteratively and actively asking the model, as it is
being learnt, which samples are most likely to be usefully
labeled. This can be done in a theoretical context, where the query is
done to maximally reduce a version space of possible model
hypotheses~\cite{cohn1994improving,dasgupta2005analysis} that agree
with the labels queried so far. Active learning also has many
empirical and practical variants where the still unlabeled samples for
which the model possesses the most uncertainty are queried first in
order to reduce residual uncertainty about the remaining unlabeled
samples. The critical difference between active learning is that,
unlike passive learning where labels are acquired for samples
regardless of how useful they might be, with active learning labels
are acquired adaptively, where the next label acquired is based on the
labels so far acquired, and therefore active learning has the
potential to be more label efficient (i.e., more information per
label) when less than 100\% of the samples are labeled.
}

Suppose we are given data set $\{ x_i, y_i\}_{i \in V}$ consisting of $|V|=n$
samples of $x,y$ pairs but where the labels are unknown.  Samples are
labeled one at a time or one mini-batch at a time, and after each
labeling step $t$ each remaining unlabeled sample is given a score
$s_t(x_i)$ that indicates the potential benefit of acquiring a label
for that sample. Examples include the entropy of the model's output
distribution on $x_i$, or a margin-based score consisting of the
difference between the top and the second-from-the-top posterior
probability.
This produces a modular function
on the unlabeled samples, $m_t(A) = \sum_{a \in A} s(x_a)$ where
$A \subseteq V$. It is simple to use this modular function to produce
a mini-batch active learning procedure where at each stage we form
$A_t \in \argmax_{A \subseteq U_t: |A| = k}m_t(A)$ where $U_t$ is the
set of labeled samples at stage $t$. Then $A_t$ is a set of size $k$
that gets labeled, we form $U_t = U_t \setminus A_t$, update $s_t(a)$
for $a \in U_t$ and repeat.
This is called \keywordDef{active learning}.

The reason for using active learning with
mini-batches of size greater than one is that it is often inefficient
to ask for a single label at a time. The problem with such a minibatch
strategy, however, is that the set $A_t$ can be redundant. The reason
is that the uncertainty about every sample in $A_t$ could be owing to
the same underlying cause --- even though the model is most uncertain
about samples in $A_t$, once one sample in $A_t$ is labeled, it may
not be optimal to label the remaining samples in $A_t$ due to this
redundancy.  Utilizing submodularity, therefore, can help reduce this
redundancy.  Suppose $f_t(A)$ is a submodular diversity model over
samples at step $t$. At each stage, choosing the set of samples to
label becomes
$A_t \in \argmax_{A \subseteq U_t: |A| = k}m_t(A) + f_t(A)$ --- $A_t$
is selected based on a combination of both uncertainty (via $m_t(A)$)
and diversity (via $f_t(A)$). This is precisely the submodular active
learning approach taken
in~\cite{wei2015-submodular-data-active,kaushal2019learning}.  \submodlongonly{There
is other work that models uncertainty and diversity in different
ways~\cite{Sener2018,Kirsch2019,ash2020deep}.}

Another quite different approach to a form of submodular ``batch''
active learning setting where a batch $L$ of labeled samples are
selected all at once and then used to label the rest of the unlabeled
samples. This also allows the remaining unlabeled samples to be
utilized in a semi-supervised
framework~\cite{guillory2009-label-selection,guillory2011-active-semisupervised-submodular}.
In this setting, we start with a graph $G=(V,E)$ where the nodes $V$
need to be given a binary $\{0,1\}$-valued label, $y \in \{0,1\}^V$.
For any $A \subseteq V$ let $y_A \in \{0,1\}^A$ be the labels just for
node set $A$. We also define $V(y) \subseteq V$ as
$V(y) = \{ v \in V : y_v = 1 \}$. Hence $V(y)$ are the graph nodes
labeled 1 by $y$ and $V \setminus V(y)$ are the nodes labeled 0.
Given submodular objective $f$, we form its symmetric CCMI
variant $I_f(A) \triangleq I_f(A ; V \setminus A)$ --- note that
$I_f(A)$ is always submodular in $A$. This allows $I_f(V(y))$ to
determine the ``smoothness'' of a given candidate labeling $y$. For
example, if $I_f$ is the weighted graph cut function where each weight
corresponds to an affinity between the corresponding two nodes, then
$I_f(V(y))$ would be small if $V(y)$ (the 1-labeled nodes) do not have
strong affinity with $V \setminus V(y)$ (the 0-labeled nodes). In
general, however, $I_f$ can be any symmetric submodular function.  Let
$L \subseteq V$ be any candidate set of nodes to be labeled, and
define
$\Psi(L) \triangleq \min_{T \subseteq (V \setminus L): T \neq
  \emptyset} I_f(T)/|T|$.  Then $\Psi(L)$ measures the ``strength'' of
$L$ in that if $\Psi(L)$ is small, an adversary can label nodes other
than $L$ without being too unsmooth according to $I_f$, while if
$\Psi(L)$ is large, an adversary can do no such
thing. Then~\cite{guillory2011-active-semisupervised-submodular}
showed that given a node set $L$ to be queried, and the corresponding
correct labels $y_L$ that are completed (in a semi-supervised fashion)
according to the following
$y' = \argmin_{ \hat y \in \{0,1\}^V : \hat y_L= y_L} I_f(V(\hat y))$,
then this results in the following bound on the true labeling
$\| y - y'\|^2 \leq 2 I_f(V(y))/\Psi(L)$ suggesting that we can find a
good set to query by maximizing $L$ in $\Psi(L)$, and this holds for
any submodular function. Of course, it is necessary to find
an underlying submodular function $f$ that fits a given problem,
and this is discussed in Section~\ref{sec:core-sets-summ}.

\subsection{Probabilistic Modeling}
\label{sec:prob-model}

Graphical models are often used to describe factorization requirements on families
of probability distributions. \submodlongonly{For example, given a graph $G=(V,E)$
where the graph can be described by a set of cliques $\mathcal C(G)$
where for each $C \in \mathcal C(G)$ the nodes are mutually connected,
it is possible to write
$p(x) = \prod_{C \in \mathcal C(G)} \phi_C( x_C) = 1/Z \exp( \sum_{C
  \in \mathcal C(G)} -E_C(x_C))$ where
$E(x) = \sum_{C \in \mathcal C(G)} E_C(x_C)$ is a non-negative energy function
and $Z$ the partition function and where, for simplicity in this
section, we assume $x \in \{0,1\}^V$ is a binary vector.  Being able
to discuss families of distributions all of which have this
factorization property makes it possible to derive inference
algorithms that work well (or approximately well) for any distribution
that factors in this way according to the graph. A well-known
complexity parameter associated with a graphical model is the
tree-width. Performing exact probabilistic inference in a distribution
that factors with respect to a graphical model has a cost that is, in
the worst case, exponential in the treewidth of the corresponding
graph. This is even true for MAP inference, i.e., computing
$\argmax_{x} p(x)$ can be exponentially costly in the tree width of
the graph within whose family $p$ resides. There is a plethora of
algorithms for approximate inference, many of which make
additional factorization assumptions in order to achieve a given
approximation~\cite{Monster}.

}Factorization is not the only way, however, to describe restrictions
on such families. In a graphical model, graphs describe only which
random variable may directly interact with other random variables.  An
entirely different strategy for producing families of often-tractable
probabilistic models can be produced without requiring any
factorization property at all.  Considering an energy function $E(x)$
where $p(x) \propto \exp(E(x))$, factorizations correspond to there
being cliques in the graph such that the graph's tree-width often is
limited. On the other hand, finding $\max_x p(x)$ is the same as
finding $\min_x E(x)$, something that can be done if $E(x) = f(V(x))$
is a submodular function (using the earlier used notation $V(x)$ to
map from binary vectors to subsets of $V$). Even a submodular function
as simple as $f(A) = \sqrt{|A|} - m(A)$ where $m$ is modular has
tree-width of $n-1$, and this leads to an energy function $E(x)$ that
allows $\max_x p(x)$ to be solved in polynomial time using submodular
function minimization (see Section~\ref{sec:subm-minim}). Such
restrictions to $E(x)$ therefore are not of the form {\em amongst the
  random variables, who is allowed to directly interact with whom},
but rather {\em amongst the random variables, what is the manner that
  they interact}. Such potential function restrictions can also
combine with direct interaction restrictions as well and this has been
widely used in computer vision, leading to cases where graph-cut and
graph-cut like ``move making'' algorithms (such as $alpha$-$beta$ swap
and $alpha$-expansion algorithms) used in attractive
models~\cite{Boykov99,Boykov01a,Boykov01,sudderth2008loop}.  In fact,
the culmination of these efforts~\cite{kolmogorov2002energy} lead to a
rediscovery of the submodularity (or the ``regular'' property) as
being the essential ingredient for when Markov random fields can be
solved using graph cut minimization, which is a special case of
submodular function minimization.

The above model can be seen as log-supermodular since
$\log p(x) = -E(x) + \log 1/Z$ is a supermodular function. These are
all distributions that put high probability on configurations that
yield small valuation by a submodular function. Therefore, these
distributions have high probability when $x$ consists of a homogeneous
and for this reason they are useful for computer vision segmentation
problems (e.g., in a segment of an image, the nearby pixels should
roughly be homogeneous as that is often what defines an object).  The
DPPs we saw above, however, are an example of a log-submodular
probability distribution since
$\submodfun(\submodsetX) = \log \det ( \submodDPPmatrix_\submodsetX )$
is submodular. These models have high probability for diverse sets.

More generally, $E(x)$ being either a submodular or supermodular
function can produce log-submodular or log-supermodular distributions,
covering both cases above where the partition function takes the form
$Z = \sum_{A \subseteq V} \exp(f(A))$ for objective $f$.  Moreover, we
often wish to perform tasks much more than just finding the most
probable random variable assignments. This includes marginalization,
computing the partition function, constrained maximization, and so
on. Unfortunately, many of these more general probabilistic inference
problems do not have polynomial time solutions even though the
objectives are submodular or supermodular. On the other hand, such
structure has opened the doors to an assortment of new probabilistic
inference procedures that exploit this
structure~\cite{djolonga2014map,djolonga2015scalable,djolonga2016variational,zhang2015higher,djolonga2018provable}.
Most of these methods were of the variational sort and offered bounds
on the partition function $Z$, sometimes making use of the fact that
submodular functions have easily computable
semi-gradients~\cite{rishabhbilmes_semidifferentials_arxiv2015,fujishige2005submodular}
which are modular upper and lower bounds on a submodular or
supermodular function that are tight at one or more subsets. Given a
submodular (or supermodular) function $f$ and a set $A$, it is
possible to easily construct (in linear time) a modular function
upper bound $m^A : 2^V \to \real$ and a modular function lower bound
$m_A : 2^V \to \real$ having the properties that
$m_A(X) \leq f(X) \leq m^A(X)$ for all $X \subseteq V$ and that is
tight at $X=A$ meaning
$m_A(A) = f(A) =
m^A(A)$~\cite{rishabhbilmes_semidifferentials_arxiv2015}.  For any
modular function $m$, the probability function for a characteristic
vector $x = \mathbf 1_A$ becomes
$p(\submodcharv_A) = 1/Z \exp(E(\submodcharv_A)) = \prod_{a \in A}
\sigma(m(a)) \prod_{a \notin A} \sigma(-m(a))$ where $\sigma$ is the
logistic function. Thus, a modular approximation of a submodular
function is like a mean-field approximation of the distribution and
makes the assumption that all random variables are independent. Such
an approximation can then be used to compute quantities such as upper
and lower bounds on the partition function, and much else.

\subsection{Structured Norms and Loss Functions}
\label{sec:struct-norms-loss}

Convex norms are used ubiquitously in machine learning, often as
complexity penalizing regularizers (e.g., the ubiquitous $p$-norms for
$p \geq 1$) and also sometimes as losses (e.g., squared error).
Identifying new useful structured and possibly learnable sparse norms
is an interesting and useful endeavor, and submodularity can help here
as well. Firstly, recall the $\ell_0$ or counting norm $\| x \|_0$
simply counts the number of nonzero entries in $x$. When we wish for a
sparse solution, we may wish to regularize using $\| x\|_0$ but it
both leads to an intractable combinatorial optimization problem, and it
leads to an object that is not differentiable. The usual approach is
to find the closest convex relaxation of this norm and that is the one
norm or $\| x \|_1$. This is convex in $x$ and has a sub-gradient
structure and hence can be combined with a loss function to produce an
optimizable machine learning objective, for example the lasso. On the
other hand, $\| x \|_1$ has no structure, as each element of $x$ is
penalized based on its absolute value irrespective of the state of any
of the other elements. There have thus been efforts to develop group
norms that penalize groups or subsets of elements of $x$ together,
such as group lasso~\cite{HastieSparse}. 

It turns out that there is a way to utilize a submodular function as
the regularizer. Penalizing $x$ via $\| x\|_0$ is identical to
penalizing it via $|V(x)|$ and note that $m(A)=|A|$ is a modular
function. Instead, we could penalize $x$ via $f(V(x))$ for a
submodular function $f$. Here, any element of $x$ being non-zero would
allow for a diminishing penalty of other elements of $x$ being zero
all according to the submodular function, and such cooperative
penalties can be obtained via a submodular parameterization.  Like
when using the zero-norm $\| x \|_0$, this leads to the same
combinatorial problem due to continuous optimization of $x$ with a
penalty term of the form $f(V(x))$.  To address this, we can use the
\lovasz{} extension \submodlongonly{mentioned above }
$\submodlex \submodfun(\submodvecx)$ on a vector $\submodvecx$. This
function is convex, but it is not a norm, but if we consider the
construct defined as
$\|\submodvecx\|_\submodfun = \submodlex \submodfun(|\submodvecx|)$,
it can be shown that this satisfies all the properties of a norm for
all non-trivial submodular
functions~\cite{popescuimage1999,bach2013learning} (i.e., those
normalized submodular functions for which $f(v) > 0$ for all $v$).  In
fact, the group lasso mentioned above is a special case for a
particularly simple feature-based submodular function (a sum of
min-truncated cardinality functions). But in principle, the same
submodular design strategies mentioned in
Section~\ref{sec:core-sets-summ} can be used to produce a submodular
function to instantiate an appropriate convex structured norm for a
given machine learning problem.

\submodlongshortalt{
  \section{Conclusions}
}{
  \subsection{Conclusions}
}

\label{sec:conclusions}

\submodlongonly{

  We have only just touched the surface of submodularity and how it
  can benefit machine learning.  From data summarization to clustering
  and active learning to parameterized sparse norms to data
  partitioning, submodularity and the discrete optimization it entails
  opens up the machine learning scientist's toolbox of ideas and
  techniques in a broad way. While the above offers only a glimpse of
  its expressiveness, it is interesting to once again revisit the
  innocuous looking submodular inequality. Very much like the
  definition of convexity, therefore, the submodular inequality belies
  much of its complexity while opening the gates to wide and fruitful
  avenues for machine learning exploration.

  Acknowledgements: This work was supported in part by the CONIX
  Research Center, one of six centers in JUMP, a Semiconductor
  Research Corporation (SRC) program sponsored by DARPA.  This
  material is based upon work supported by the National Science
  Foundation under Grant No. 2106389.

}
\submodshortonly{

  We have only barely touched the surface of submodularity and how it
  applies to and can benefit machine learning. For more details,
  see~\cite{bilmes-submod-and-ml-2022} and the many references
  contained therein.  Considering once again the innocuous looking
  submodular inequality, then very much like the definition of convexity,
  we observe something that belies much of its complexity
  while opening the gates to wide and worthwhile avenues for machine
  learning exploration.

}

\newpage
\printbibliography

\end{document}